\documentclass[12pt]{article}
\usepackage[margin =1in]{geometry}
\usepackage{times}
\usepackage{microtype}
\usepackage{graphicx}
\usepackage{booktabs} 
\usepackage{subcaption}
\usepackage[numbers]{natbib}
\usepackage{amsthm}

\usepackage[utf8]{inputenc} 
\usepackage[T1]{fontenc}    

\newtheorem{lemma}{Lemma}[section]
\newtheorem{theorem}{Theorem}[section]

\newtheorem{proposition}{Proposition}[section]
\newtheorem{assumption}{Assumption}[section]
\newtheorem{remark}{Remark}
\newtheorem{definition}{Definition}
\newtheorem{example}{Example}

\newcommand{\trace}[1]{\text{tr}\left(#1\right)}

\def\Id{\operatorname{I}}

\usepackage{url}            
\usepackage{booktabs}       
\usepackage{amsfonts}       
\usepackage{nicefrac}       
\usepackage{color}
\usepackage{bm,multirow,xspace}
\usepackage{enumerate,mathtools}
\usepackage{graphicx}
\usepackage{epstopdf,amsmath}
\usepackage{bbm}
\usepackage{rotating}
\usepackage{amssymb}

\usepackage{mathrsfs}
\usepackage{makecell}

\allowdisplaybreaks

\usepackage{amsmath}    
\usepackage{algorithm}

\usepackage{color}

\usepackage{caption}

\usepackage{algpseudocode}
\usepackage{graphicx}
\usepackage{epstopdf}

\usepackage{enumitem}

\allowdisplaybreaks

\usepackage{tcolorbox}
\tcbuselibrary{theorems}

\newcommand*{\colorboxedAux}[3]{%
	\begingroup
	\colorlet{cb@saved}{.}%
	\color#1{#2}%
	\boxed{%
		\color{cb@saved}%
		#3%
	}%
	\endgroup
}

\DeclareMathOperator*{\argmax}{argmax}
\DeclareMathOperator*{\argmin}{argmin}

\newcommand{\floor}[1]{\lfloor #1 \rfloor}

\newcommand{\R}{\mathbb{R}}

\newcommand{\pro}{{\mathbb P}}
\newcommand{\expec}[2]{\mathbb{E}_{#2}[ #1 ] }
\newcommand{\norm}[1]{\lVert #1 \rVert}
\newcommand\numberthis{\addtocounter{equation}{1}\tag{\theequation}}  

\def\fn[#1]#2{{f_{#1}\left(x_{#2}\right)}}
\newcommand{\prox}{\text{prox}}

\newcommand{\M}{\mathcal{M}}

\newcommand{\cM}{\mathcal{M}}

\DeclareMathOperator{\dom}{dom}

\newcommand*{\colorboxed}{}
\def\colorboxed#1#{%
	\colorboxedAux{#1}%
}

\newcommand{\inner}[1]{ \left\langle {#1} \right\rangle }

\newcommand{\tangentM}[1]{\tangent{\cM}{#1}}
\newcommand{\tangent}[2]{T_{#1}({#2})}
\usepackage[]{mdframed}

\def\exp{{\rm exp}}

\def\cP{{\cal P}}

\def\cM{{\cal M}}

\def\cH{{\cal H}}
\def\cX{{\cal X}}

\def\cI{{\cal I}}
\def\cF{{\cal F}}

\def\cU{{\cal U}}

\def\dom{{\rm dom}}

\def\perturb{\nu}
\def\RR{\mathbb{R}}
\def\localconstant{C}
\def\dist{\text{dist}}

\newcommand{\indic}[1]{\mathbbm{1}_{\tau_{k_s,\delta}>#1}}
\newcommand{\EE}{\mathbb{E}}
\newcommand{\stepsize}{\eta}
\newcommand{\proj}{P}

\newcommand{\dotp}[1]{\left\langle #1\right\rangle}
\newcommand{\ks}{k_s}

\newcommand{\cub}{C_{\text{ub}}}
\newcommand{\lm}{L_\cM}
\newcommand{\strong}{\gamma}

\newenvironment{talign*}
{\csname align*\endcsname}
{\endalign}

\numberwithin{equation}{section}

\usepackage[utf8]{inputenc} 
\usepackage[T1]{fontenc}    
\usepackage{hyperref}       
\usepackage{url}            
\usepackage{booktabs}       
\usepackage{amsfonts}       
\usepackage{nicefrac}       
\usepackage{microtype}      

\usepackage{xargs}                     
\usepackage{xargs}

\mathtoolsset{showonlyrefs=true}

\title{Online Covariance Estimation in Nonsmooth Stochastic Approximation}

\author{Liwei Jiang\thanks{H. Milton Stewart School of Industrial and Systems Engineering, Georgia Institute of Technology, Atlanta, GA 30332, USA. Emails: \href{mailto:ljiang306@gatech.edu}{ljiang306@gatech.edu}, \href{mailto:senna@gatech.edu}{senna@gatech.edu}. }\and 
Abhishek Roy\thanks{Department of Statistics,  Texas A\&M University,  College Station, TX, 77843, USA. Email: \href{mailto:abhishekroy@tamu.edu}{abhishekroy@tamu.edu}.  Abhishek Roy contributed to this work equally with Liwei Jiang.}\and
Krishna  Balasubramanian\thanks{Department of Statistics, University of California, Davis, CA 95616, USA. Email: \href{mailto:kbala@ucdavis.edu}{kbala@ucdavis.edu}. } \;\;
\and
Damek Davis\thanks{Wharton Department of Statistics and Data Science, University of Pennsylvania, Philadelphia, PA 19104, USA. Email: \href{mailto:damek@wharton.upenn.edu}{damek@wharton.upenn.edu}. }\and 
Dmitriy Drusvyatskiy\thanks{Department of Mathematics, University of Washington, Seattle, WA 98195. Email: \href{mailto:ddrusv@uw.edu}{ddrusv@uw.edu}. }
\and Sen Na\footnotemark[1]}

\date{}

\begin{document}
	
\maketitle
\begin{abstract}

We consider applying stochastic approximation (SA) methods to solve nonsmooth variational inclusion problems. 
Existing studies have shown that the averaged iterates of SA methods exhibit asymptotic normality, with an optimal limiting covariance matrix in the local minimax sense of  H\'{a}jek and Le Cam. However, no methods have been proposed to estimate this covariance matrix in a nonsmooth and potentially non-monotone (nonconvex) setting. 
In this paper, we study an~online batch-means covariance matrix estimator introduced in \cite{zhu2023online}. The estimator groups the SA iterates appropriately and computes the sample covariance among batches as an estimate of the limiting covariance. Its construction does not require prior knowledge of the total sample size, and updates can be performed recursively as new data arrives. 
We establish that, as long as the batch size sequence is properly specified (depending on the stepsize sequence), the estimator achieves a convergence rate of order $O(\sqrt{d}n^{-1/8+\varepsilon})$ for any $\varepsilon>0$, where $d$ and $n$ denote the problem dimensionality and the number of iterations (or samples) used. Although the problem is nonsmooth and potentially non-monotone (nonconvex), our convergence rate matches the best-known rate for covariance estimation methods using only first-order information in smooth and strongly-convex settings. The consistency of this covariance estimator enables asymptotically valid statistical inference, including constructing confidence intervals and performing hypothesis testing.

\end{abstract}

\section{Introduction}\label{sec:1}

A landmark result by \cite{polyak1992acceleration} shows that for smooth and strongly convex~optimization, Stochastic Gradient Descent (SGD) exhibits a central limit theorem: the averaged SGD iterates with a proper scaling factor converge to a normal distribution; see~\cite{toulis2017asymptotic,duchi2021asymptotic} for extensions and~\cite{anastasiou2019normal,shao2022berry,samsonov2024gaussian} for quantitative non-asymptotic bounds. Recently, \cite{davis2024asymptotic} extended this result to nonsmooth problems, showing that when solutions vary smoothly with respect to perturbations, the averaged generic stochastic approximation (SA) iterates remain asymptotically normal. This limiting distribution paves the way for constructing confidence intervals and statistical tests, critical tools for uncertainty quantification in machine learning and optimization. However, to perform (asymptotically) valid statistical inference, we need to estimate the covariance matrix of the limiting distribution. While efficient online estimators are well understood in the smooth setting, estimation in the nonsmooth setting has remained completely open. In this paper, we develop an online estimator with computation and memory scaling quadratically in dimension, and establish~its~rate~of~convergence in expectation (matching the smooth setting).

The theory encompasses many important problems in machine learning and operations research. Consider a two-player zero-sum game. To find the Nash equilibrium, the two players aim to solve:
\begin{equation*}
\min_{x_1 \in \cX_1} \max_{x_2 \in \cX_2} \mathop\EE_{\nu \sim \cP}[f(x_1,x_2,\perturb)],
\end{equation*}
where $f(x_1,x_2,\perturb)$ is a random payoff function and $\cX_1, \cX_2$ are strategy sets. Players update their strategies based on noisy observations, projecting onto their respective strategy sets. Another example is stochastic nonlinear programming; we solve:
\begin{equation}\label{eqn:constrained_opt}
\min_{x} \mathop\EE_{\nu \sim \cP}[f(x,\nu)] \quad \text{subject to} \quad g_i(x) \leq 0, \quad i=1,\ldots,m,
\end{equation}
where the objective depends on random data. Both settings, along with many others, can be unified through stochastic variational inequalities of the form:
\begin{equation}\label{eqn:variation_inclu_intro}
0 \in F(x) := \mathop\EE_{\nu \sim \cP}[A(x, \nu)] + N_{\cX}(x),
\end{equation}
where $A(\cdot, \nu)$ is a smooth operator for each $\nu$, and $N_{\cX}$ denotes the normal cone to the constraint~set~$\cX$. Throughout, we fix a solution $x^\star$ of this inclusion.

To solve the above problems in an online fashion, we consider SA algorithms based on a~generalized gradient mapping, $G: \RR_{++} \times \RR^d \times \RR^d \mapsto \RR^d$, of $F$. Given $x_0$, the algorithm iterates~as
\begin{equation}\label{snequ:1}
x_{k+1} = x_k - \eta_{k+1}G_{\eta_{k+1}}(x_k, \nu_{k+1}),
\end{equation}
where $\eta_{k+1} > 0$ is a stepsize sequence and $\nu_k$ is stochastic noise. As we show in Section~\ref{sec:SA_algorithms},~this~framework unifies many online algorithms -- in games it captures simultaneous gradient play; in constrained optimization it yields projected gradient methods; and more generally, it encompasses stochastic forward-backward splitting.

\cite{davis2024asymptotic} showed that when solutions to the perturbed system vary smoothly -- that is, when the graph of the solution map $S(v) = \{x : v \in F(x)\}$ locally coincides with the graph~of~some smooth function $\sigma(\cdot)$ -- the averaged iterates of \eqref{snequ:1} are asymptotically normal:
\begin{equation} \label{eqn:asymptotic_normality}
\sqrt{k}(\bar x_k - x^{\star}) \xrightarrow{D} N(0, \Sigma),
\end{equation}
where $\bar x_k = \sum_{i=1}^kx_i/k$ and $\Sigma = \nabla \sigma(0)\cdot\text{Cov}(A(x^{\star}, z))\cdot\nabla \sigma(0)^{\top}$. For example, in stochastic nonlinear programming~\eqref{eqn:constrained_opt}, $A(x^\star,\nu) = \nabla f(x^\star, \nu)$ and $\nabla \sigma (0)$ takes a particularly elegant form
\begin{equation*}
\nabla \sigma (0) = (P_{T}\nabla^2_{xx}\mathcal{L}(x^\star, y^\star)P_{T})^{\dagger},
\end{equation*}
where $(x^\star, y^\star)$ is the primal-dual solution of \eqref{eqn:constrained_opt}, $\mathcal{L}(x,y) = f(x) + \sum_{i=1}^{n+m} y_i g_i(x)$ is the Lagrangian function, and $P_T$ projects onto the tangent space of active constraints at the solution $x^\star$.

In order to leverage the aforementioned result in practice to construct confidence sets, it is required to estimate the asymptotic covariance matrix $\Sigma$. The batch-means estimator~\citep{lahiri2003resampling, flegal2010batch} from the larger Markov chain literature has been recently adapted in the literature for developing \emph{online} estimators of $\Sigma$; see, for example,~\cite{zhu2023online} and~\cite{roy2023online}. The key idea is to divide the iterates into blocks of increasing size,~with each block providing an approximately independent estimate of the covariance matrix. The block sizes are carefully~chosen to balance the bias-variance tradeoff while maintaining the~desirable~convergence~rate. Specifically, let $\{a_m\}_m$ be a strictly increasing sequence of integers with $a_1=1$. For any $k=1,2,\ldots$, we construct a block $B_k$ consisting of the iterates $\{x_{t_k},x_{t_k+1},\ldots,x_k\}$ where $t_k=a_m$ for $k\in [a_m,a_{m+1})$. Let $l_k=|B_k|$ denote the size of the block $B_k$. After $n$ iterations, the batch-means covariance estimator is given by:
\begin{equation}\label{eqn:online_batch_means}
\hat{\Sigma}_n=\frac{\sum_{i=1}^n\left(\sum_{k=t_i}^ix_k-l_i\bar{x}_n\right)\left(\sum_{k=t_i}^ix_k-l_i\bar{x}_n\right)^
\top}{\sum_{i=1}^nl_i}.
\end{equation}
\cite{zhu2023online} showed that for SGD with i.i.d. data stream, $\hat{\Sigma}_n$ (asymptotically) consistently~estimates $\Sigma$ with a convergence rate of order $O(n^{-1/8})$. Subsequently, \cite{roy2023online} extended this result to Markovian data. However, these limited existing works on online covariance estimation for first-order methods apply only to smooth and strongly convex problems, and their analyses do not apply to generic iterations as in \eqref{snequ:1}.

\paragraph{Main Contribution.} 
Our main contribution is to show that, despite significant complexity~introduced by nonsmooth geometry, we can achieve the same convergence rate as in the smooth~case using the same covariance estimator \eqref{eqn:online_batch_means}.
In particular,  we establish that under reasonable conditions and with a properly chosen batch size control sequence $\{a_m\}_m$, the online batch-means estimator $\hat \Sigma_n$ in \eqref{eqn:online_batch_means} with generic SA iterates \eqref{snequ:1} satisfies 
\begin{equation*}
\mathbb{E}\|\hat{\Sigma}_n - \Sigma\|_2 = O(\sqrt{d}n^{-1/8+\varepsilon}) \quad \text{for any }\; \varepsilon>0.
\end{equation*}
We also emphasize that when applying our result to stochastic optimization problems, the objective does not need to be strongly convex or even convex. This is in contrast with all existing works that heavily rely on global strong convexity \citep{chen2020statistical,zhu2023online, roy2023online}. Our analysis addresses the following main challenges:
\begin{enumerate}
\item 
Due to the nonsmooth nature of problem~\eqref{eqn:variation_inclu_intro}, Taylor's theorem -- on which all existing~methods \citep{chen2020statistical,zhu2023online, roy2023online} are based -- is~no longer applicable. Our key insight is that, despite the problem being nonsmooth, typical~instances exhibit partial smoothness near the solution. In other words, there exists a distinctive manifold containing the solution 
and capturing the hidden smoothness of the map $F$.
In a~local neighborhood around the solution, we project all iterates onto this manifold, forming what we call the \textit{shadow sequence}. We then prove that the shadow sequence behaves almost as if~it were generated by a smooth dynamic.

\item Our analysis of the shadow sequence builds on prior work on nonsmooth asymptotic normality \citep{davis2024asymptotic}; however, their asymptotic guarantees are insufficient for our~non-asymptotic study. In this work, we provide a more refined analysis and establish a tighter~bound on the distance between the original iterates and their shadows. Our results show that the~hypothetical batch-means estimator constructed from the shadow sequence converges to the same limit -- and at the same rate -- as the estimator based on the original~sequence~\eqref{eqn:online_batch_means}.~Consequently, the problem reduces to analyzing the estimator derived from smooth dynamics.

\item Due to the local nature of both the manifold and the shadow sequence, the above argument~holds only when the iterates remain within a local neighborhood of the solution. To~address this,~we introduce a stopping time. Under light-tailed noise, we apply a martingale~concentration~inequality to show that, with high probability, the original iterates stay within the local neighborhood after a certain number of iterations. Consequently, the shadow sequence always exists, and the stopping time can finally be dropped in the convergence guarantee.
\end{enumerate}
 We should mention that our above techniques extend beyond the covariance estimation problem, offering a template for analyzing other nonsmooth SA algorithms whose dynamics are implicitly~governed by an underlying local smooth structure.

\paragraph{Paper organization.}
In Section \ref{sec:2}, we introduce the notations and preliminaries, including smooth manifolds and nonsmooth analysis. In Section \ref{sec:3}, we present the assumptions and main results. In Section \ref{sec:highprob}, we address the issue of the stopping time involved in our main results by providing a high-probability guarantee. In Section \ref{sec:SA_algorithms}, we present specific examples of SA algorithms for~nonsmooth problems, and we conclude and discuss future work in Section \ref{sec:6}. Concrete examples of nonsmooth variational inclusion problems satisfying our assumptions, as well as the proofs of theoretical results, are deferred to the appendix.

\section{Notations and preliminaries}\label{sec:2}

\paragraph{Notations.} 
Throughout the paper, the symbol $\R^d$ denotes a Euclidean space with 
inner product $\langle\cdot,\cdot \rangle$ and the induced norm $\|x\|_2=\sqrt{\langle x,x\rangle}$. The symbol $\mathbb{B}$ denotes the closed unit ball in $\R^d$, while $B_r(x)$ denotes the closed ball of radius $r$ around a point $x$. When $A \in \RR^{m\times n}$ is a matrix, $\|A\|_2$ denotes the spectral norm of $A$.
For any function $f\colon\R^d\to\R\cup\{+\infty\}$,  its {\em domain} is defined as $\dom\, f:=\{x\in \R^d: f(x)<\infty\}.$ We say $f$ is {\em closed} if its epigraph is a closed set, or equivalently if $f$ is lower-semicontinuous. The {\em proximal map of $f$ with parameter $\alpha>0$} is given by 
$$\prox_{\alpha f}(x):=\argmin_y \left\{f(y)+\frac{1}{2\alpha}\|y-x\|_2^2\right\}.$$
The {\em distance} and the {\em projection} of a point $x\in\R^d$ onto a set $Q\subset\R^d$ are, respectively, 
\begin{align*}
d(x,Q):=\inf_{y\in Q}\|y-x\|_2 \qquad\textrm{and}\qquad P_Q(x):=\argmin_{y\in Q}\|y-x\|_2.
\end{align*}
The indicator function of  $Q$, denoted by $\delta_Q(\cdot)$, is defined to be zero on $Q$ and $+\infty$ off it.  The symbol $o(h)$ stands for any function $o(\cdot)$ satisfying $o(h)/h\to 0$ as $h\searrow 0$.

\paragraph{Smooth manifold.} To be self-contained, we make a few definitions for smooth manifold; we~refer the reader to \cite{lee2013smooth,boumal2020introduction} for details. Throughout the paper, all smooth manifolds $\mathcal{M}$ are assumed to be embedded in $\R^d$, and we consider the tangent and normal spaces to $\mathcal{M}$ as subspaces of $\R^d$. In particular, for any $x\in \M$, we denote the tangent and normal spaces of $M$ at $x$ by $T_\cM(x)$ and $N_\cM(x)$, respectively.  A map $F\colon\cM\to\R^m$ is called $C^p$ ($p\ge 1$) smooth near a point $ x$ if there exists a $C^p$-smooth map $\hat F\colon U\to\R^d$ defined on some neighborhood $U\subset\R^d$ of $ x$ that agrees with $F$ on $\cM$ near $ x$. In this case, we define the {\em covariant Jacobian} $\nabla_\cM F( x)\colon T_{\cM}( x)\to\R^m$ by the expression $\nabla_{\cM} F(x)(u)=\nabla \hat F( x) u$ for all $u\in T_{\cM}(x)$.

\paragraph{Nonsmooth analysis.}\label{sec:normalcones}
Next, we introduce a few terminologies used in nonsmooth and variational analysis. The introduction follows \cite{rockafellar2009variational}. Consider a function $f\colon\R^d\to\R\cup\{+\infty\}$ and a point $x\in \dom\, f$. The  {\em Fr\'{e}chet subdifferential  of $f$ at $x$}, denoted $\hat \partial f(x)$, consists~of~all vectors $v\in \R^d$ satisfying the approximation property:
$$f(y)\geq f(x)+\langle v,y-x\rangle+o(\|y-x\|)\quad \textrm{as}\quad y\to x.$$
The {\em limiting subdifferential  of $f$ at $x$}, denoted $\partial f(x)$, consists of all vectors $v\in \R^d$ such that~there exist sequences $x_i\in \R^d$ and Fr\'{e}chet subgradients $v_i\in \hat \partial f(x_i)$ satisfying $(x_i,f(x_i),v_i)\to (x,f(x),v)$ as $i\to\infty$. A point $ x$ satisfying $0\in\partial f(x)$ is called {\em critical} for $f$. For any set $Q$ and $x\in Q$, the {\em Fr\'{e}chet normal cone  of $Q$ at $x$} is defined by $\hat N_Q(x):= \hat \partial \delta_Q(x),$ where $\delta_Q$ is the indicator function of $Q$. Similarly, the {\em limiting normal cone  of $Q$ at $x$} is defined by $N_Q(x):= \partial \delta_Q(x).$

\section{Assumptions and main results}\label{sec:3}

Setting the stage, our goal is to find a point $x$ satisfying the inclusion
\begin{align}\label{eqn:inclusion_problem}
0 \in F(x),
\end{align}
where $F: \RR^d \rightrightarrows \RR^d$ is a set-valued map. Throughout, we fix one such solution $x^\star$ of~\eqref{eqn:inclusion_problem}.~We~assume the existence of a distinctive manifold $\cM$ that contains $x^\star$ and satisfies the property that~the map $x \mapsto P_{T_\cM(x)} F(x)$ is single-valued and $C^p$-smooth on $\cM$ near $x^\star$. The following assumption provides a precise statement of this assumption.

\begin{assumption}[Smooth structure]\label{ass:basic_assumpt}
{\rm 
Suppose that there exists a $C^p$ $(p \ge 1)$ manifold $\mathcal{M}\subset\R^d$~such that the map $F_{\cM}\colon \mathcal{M}\to \R^d$ defined by $F_{\cM}(x) \coloneqq P_{T_{\cM}(x)}F(x)$ is single-valued and $C^p$ smooth on some neighborhood $V$ of $x^\star$ in $\cM$.  Moreover, there exists $\strong >0$ and $\lm >0$ such that $F_\cM$ is $\lm$-Lipschitz in $V \cap \cM$, and for any $x \in V \cap \cM$, 
\begin{equation}\label{eqn:QG}
\dotp{F_\cM(x), x - x^\star} \ge \strong \|x - x^\star\|^2.
\end{equation}
}
\end{assumption}

Note that in the case when $F = \nabla f$ for some smooth function $f$, the manifold $\cM$ is simply $\RR^d$, and the condition~\eqref{eqn:QG} is equivalent to the local quadratic growth condition \citep{davis2022nearly}. To illustrate the role of manifold $\cM$ for nonsmooth map $F$, we consider the following two examples: $\ell_1$-regularization problems and nonlinear programming. A detailed discussion of these and more examples can be found in Appendix~\ref{sec:examples_active}.

\begin{example}[$\ell_1$-regularization]\label{ex:l1reg_intro}
{\rm
Consider the stochastic optimization problem with $\ell_1$ regularization
\begin{equation} \label{eqn:l1reg_intro}
\min_{x}~ g(x) = f(x) + \lambda \|x\|_1,
\end{equation}
where $f(x) = \EE_{\nu \in \cP}[f(x, \nu)]$ is a $C^p$-smooth function in $\RR^d$. Consider now $x^\star \in \RR^d$, a critical point of the function $g$, and define the index set $\cI = \{i \colon x^\star_i = 0\}$. Then, the set $\cM = \{ x\colon x_i = 0,~\forall i \in \cI\}$ is an affine space, hence a smooth manifold. It is easy to show that when $\nabla^2 f(x^\star)$ is positive definite restricted onto $T_\cM(x^\star)$, the map $F = \partial g$ satisfies Assumption~\ref{ass:basic_assumpt} with manifold $\cM$.
}
\end{example}

\begin{example}[Nonlinear programming]\label{ex:nlp_intro}
{\rm
Consider the problem of nonlinear programming
\begin{equation}\label{eqn:nl0_intro}
\begin{aligned}
\min_{x}~ &f(x),\\
{\rm s.t.}~\,&g_i(x)\leq 0\qquad \textrm{for }i=1,\ldots,m,\\
&g_i(x)= 0\qquad \textrm{for }i=m+1,\ldots,n,
\end{aligned}
\end{equation}
where $f$ and $g_i$ are $C^p$-smooth functions on $\R^d$. Let $\cX$ denote the set of all feasible points to the problem. Consider now a point $x^\star\in \cX$ that is critical for the function $f+\delta_{\cX}$, and define~the~active index set $\mathcal{I}=\{i: g_i(x^\star)=0\}$. Suppose the Linear Independence Constraint Qualification (LICQ) condition holds, i.e., the gradients $\{\nabla g_i(x^\star)\}_{i\in \mathcal{I}}$ are linearly independent. Then, the~set~$\cM=\{x: g_i(x)=0~\forall i\in \cI\}$ is a $C^p$ smooth manifold locally around $x^\star$. In the literature on nonlinear programming, the manifold $\cM$ is also referred to as the active set \citep{nocedal2006numerical}.~Define the Lagrangian function
$$\mathcal{L}(x,y):=f(x)+\sum_{i=1}^{n+m} y_i g_i(x).$$
The criticality of $x^\star$ and LICQ ensure that there exists a (unique) Lagrange multiplier vector~$y^\star \in \R^{m}_+\times \R^n$ satisfying $\nabla_x \mathcal{L}(x^\star,y^\star)=0$ and $ y_i^\star=0$ for all $i\notin \cI$. Assume in addition that $\nabla^2_{xx} \mathcal{L}(x^\star,y^\star)$ is positive definite when restricted onto $T_\cM(x^\star)$, often called the  Second-Order Sufficient Condition (SOSC); we can then show that  $F = \nabla f + N_\cX$ satisfies Assumption~\ref{ass:basic_assumpt} with the manifold~$\cM$.
}
\end{example}

The stochastic approximation (SA) algorithms we consider in this work assume access to a \emph{generalized gradient mapping} $G: \RR_{++} \times \RR^d \times \RR^d \mapsto \RR^d$. As stated in Section \ref{sec:1}, given~$x_0$,~our~generic~SA algorithm iterates as
\begin{align}\label{eqn: updaterule}
x_{k+1}=x_k-\eta_{k+1}G_{\eta_{k+1}}(x_k,\nu_{k+1}), \quad \forall k \ge 0,
\end{align}
where $\eta_{k+1} >0$ is a stepsize sequence and $\nu_k$ is stochastic noise. We now state two assumptions on $G$ that are required in~\cite{davis2024asymptotic} for establishing the asymptotic normality of the averaged iterates of \eqref{eqn: updaterule}. The first assumption is similar to classical Lipschitz assumptions and ensures~that~the stepsize length can only scale linearly in $\|\nu\|$.

\begin{assumption}[Steplength] \label{assumption:localbound}
{\rm
We suppose there exist a constant $\localconstant>0$ and a neighborhood $\cU$ of $x^\star$ such that the map $G$ satisfies $\sup_{x \in \cU_F} \|G_\eta(x, \perturb)\| \leq \localconstant(1+\|\perturb\|)$ for any $\nu \in \RR^d$ and $\eta > 0$, where  we set $\cU_F \coloneqq \cU \cap \dom F$.
}
\end{assumption}

The second assumption precisely characterizes the relationship between two mappings, $G$ and $F_\cM$. For simplicity, we abuse the notation $C$ to denote a general upper bound.

\begin{assumption}\label{assumption:Aproposed}
{\rm 
We suppose that there exist constants $\localconstant, \mu > 0$, a manifold $\cM$ containing $x^\star$, and a neighborhood $\cU$ of $x^\star$ such that the following hold for any $\nu \in \RR^d$ and $\eta > 0$, where~we~set~$\cU_F \coloneqq \cU \cap \dom F$:
\begin{enumerate}
\item\label{assumption:smoothcompatibility} {\bf (Tangent comparison)}
For any $x \in \cU_F$, we have
\begin{align*}
\|P_{T_\cM({P_{\cM}(x))}}(G_\eta(x, \perturb) - F(P_{\cM}(x)) - \nu)\| \leq C (1 + \|\perturb\|)^2(\dist(x, \cM) +\eta).
\end{align*}
\item \label{assumption:aiming} {\bf (Proximal Aiming)} For any $x \in \cU_F$, we have
\begin{align*}
\inner{G_\eta(x, \perturb) - \nu, x - P_{\cM}(x)} &\geq \mu \cdot \dist(x, \cM) - (1+\|\perturb\|)^2(o(\dist(x, \cM)) + C\eta).
\end{align*}
\end{enumerate}}
\end{assumption}

In the above assumption, Item~\ref{assumption:smoothcompatibility} asserts that in the tangent directions of $\cM$, the gradient map~$G$ accurately approximates the map $F$; while Item~\ref{assumption:aiming} asserts that in the normal directions, the gradient map $G$ points outward from $\cM$. 
In the context of stochastic optimization, Assumptions~\ref{ass:basic_assumpt}--\ref{assumption:Aproposed} neither imply global strong convexity nor global convexity. See Example~\ref{example:nonconvex} in Appendix~\ref{sec:examples_active} for a concrete example. These broader and weaker assumptions extend the scope of existing online inference works, which have focused solely on strongly convex problems \citep{chen2020statistical,zhu2023online,roy2023online}.

In the next two assumptions, we consider the choice of stepsize and the conditions on stochastic noise for online covariance estimation.

\begin{assumption}\label{assumption:zero}
{\rm~We assume the following conditions hold.
\begin{enumerate}[noitemsep]
\item The map $G_\eta$ is measurable.
\item The stepsize $\eta_k =\eta k^{-\alpha}$ for some $\eta > 0$ and $\alpha \in (\frac{1}{2}, 1)$.
\item $\{\perturb_{k+1}\}$ is a martingale difference sequence w.r.t.\ to the increasing sequence of $\sigma$-fields 
$\cF_k = \sigma(x_{0:k}, \perturb_{1:k})$. Furthermore, there exists a function $q \colon \RR^d \rightarrow \RR_+$ that is bounded on bounded sets satisfying $\expec{\|\perturb_{k+1}\|^8}{k} \leq q(x_{k})$, where $\expec{\cdot}{k}=\expec{\cdot\mid \cF_k}{}$.
\item The inclusion $x_{k} \in \dom F$ holds for all $k \geq 0$.
\end{enumerate}}
\end{assumption}

Assumption \ref{assumption:zero} on the stepsize and noise is almost identical to \cite[Assumption~I]{davis2024asymptotic} for establishing asymptotic normality guarantees. The only difference is the requirement of the~eighth moment of $\|\nu_k\|$, whereas \cite{davis2024asymptotic} requires only the fourth moment. A stricter noise~moment condition appears to be natural for the covariance estimation problem. For example, the noise moment condition for covariance estimation of simple SGD method is also stricter than the moment condition needed for asymptotic normality; see \cite{polyak1992acceleration} and \cite{chen2020statistical,zhu2023online} for comparisons.

We next impose an additional assumption concerning the covariance of the stochastic noise $\nu_k$. Similar assumptions also widely appear in the literature on both first-order methods \citep{duchi2021asymptotic,davis2024asymptotic,chen2020statistical,zhu2023online,roy2023online} and second-order methods \citep{bercu2020efficient, na2022statistical}.

\begin{assumption}\label{assumption:martinagle}
{\rm 
Fix $x^\star \in \dom F$ at which Assumption~\ref{ass:basic_assumpt} holds and let $U$ be a matrix whose~columns form an orthogonal basis of $T_\cM(x^\star)$. We assume the gradient noise can be decomposed~as $\nu_{k+1} = \nu_{k+1}^{(1)} + \nu_{k+1}^{(2)}(x_k)$, where $\nu_{k+1}^{(2)} \colon \dom F \rightarrow \RR^d$ is a random function satisfying for some $C>0$,
$$\EE_k[\| \nu_{k+1}^{(2)}(x)\|^2] \leq C\|x - x^\star\|^2 \qquad \text{for all $x\in \dom F$},$$
and $\EE_k[\perturb_{k+1}^{(2)}(x)] = \EE_k[\perturb_{k+1}^{(1)}] = 0$. In addition, we assume the following covariance matrix is~constant for all $k \ge 1$:
\begin{align}\label{eqn:def_of_S}
S \coloneqq \EE_k[ U^\top\perturb_{k}^{(1)} {\nu_k^{(1)}}^\top U].
\end{align}
}
\end{assumption}

Note that all the previous assumptions regulate only the local behavior of the maps $F$ and $G$. To control the behavior of the iterates far from $x^\star$, we impose the following mild assumption and rigorously show that it holds for a variety of nonsmooth SA methods in Appendix~\ref{sec:global_guarantee_lemma}.

\begin{assumption}[Bounded sequence in expectation] \label{assum:bounded_seq}
{\rm 
There exists a constant $\cub>0$ such that $\EE[\|x_k - x^\star\|^2 ] \le \cub$.
}
\end{assumption}

Let $U$ be a matrix whose columns form an orthonormal basis of $T_\cM(x^\star)$. We recall that~the~limiting covariance matrix in the nonsmooth asymptotic normality result takes the following form \cite[Theorem 5.1]{davis2024asymptotic}:
\begin{align}
\Sigma:=  U(U^\top \nabla_{\cM} F_\cM(x^\star) U)^{-1} S(U^\top \nabla_{\cM} F_\cM(x^\star) U)^{-\top}U^\top,\numberthis\label{eq:covdef}
\end{align}
where $\nabla_\cM F_\cM(x^\star)$ is the covariant Jacobian of $F_\cM$, and $S$ is defined in~\eqref{eqn:def_of_S}.

We are now ready to state our main result on the convergence of the online batch-means covariance estimator~\eqref{eqn:online_batch_means}. The formal statement of our result crucially relies on local arguments and frequently refers to the following stopping time: given an index $k \geq 0$ and a constant $\delta \in (0,1)$, we define
\begin{align*}
\tau_{k, \delta} := \inf\{l \geq k \colon x_l  \notin B_{\delta}(x^\star)\},
\end{align*}
which is the first time after $k$ that the iterate leaves $B_{\delta}(x^\star)$. The following is our main convergence theorem, with its proof provided in Appendix~\ref{sec:proof_of_main}.

\begin{theorem}\label{th:mainthmcov}
Under Assumptions~\ref{ass:basic_assumpt}--\ref{assum:bounded_seq}, let us set $a_m  = \floor{Cm^{\beta}}$ for some constant $C\ge 1$ and $\beta > \frac{1}{1-\alpha}$. Then, for the iteration scheme \eqref{eqn: updaterule} and any $\ks \le n$, we have\footnote{In the rest of the paper, we use $a_n \lesssim b_n$ to denote $a_n \le C b_n$ for some constant $C$ independent of $\ks$ (if applicable), $d$ and $n$, and $a_n \asymp b_n$ to denote $a_n \lesssim b_n$ and $b_n \lesssim a_n$.}
\begin{align*}
\EE[\norm{\hat{\Sigma}_n-\Sigma }_2 \indic{n}] \lesssim \ks^3 (dn^{\frac{(\alpha-1) +\beta}{\beta}} + \sqrt{d} n^{\frac{(\alpha-1) +\beta}{2\beta}} + \sqrt{d} n^{-\frac{1}{2\beta}}).
\end{align*}
\end{theorem}

\begin{remark}
Choosing $\beta=\frac{2}{1-\alpha}$, we have
\begin{align*}
\EE[\norm{\hat{\Sigma}_n-\Sigma }_2 \indic{n}] \lesssim  \ks^3(d n^{-\frac{1-\alpha}{2}} + \sqrt{d} n^{\frac{1-\alpha}{4}}).
\end{align*}
Further choosing $\alpha=\frac12+4\varepsilon$ for some arbitrarily small $\varepsilon>0$, we have
\begin{align*}
\EE[\norm{\hat{\Sigma}_n-\Sigma }_2 \indic{n}] 
\lesssim  \ks^3 (dn^{-\frac14+2\varepsilon}+\sqrt{d}n^{-\frac18+\varepsilon}).\numberthis\label{eq:bestrate}
\end{align*}
A comparison of Theorem ~\ref{th:mainthmcov} with related settings is in order. In particular, \eqref{eq:bestrate} shows that as long as $\ks$ is a constant,  we recover the convergence rate in the smooth case with an i.i.d data~stream~\citep{zhu2023online}. In Section~\ref{sec:highprob}, we show that under mild assumptions, the probability that the iterates leave the local neighborhood after $\ks$ decays exponentially in $\ks$. Moreover, by allowing $\ks\asymp \log^2 n$,~we recover the best-known convergence rate $O(n^{-1/8})$ in the smooth case up to logarithmic factors.~More interestingly, this rate also matches the rate obtained in the smooth case for exponentially mixing Markovian data streams \citep{roy2023online}.

\end{remark}

\paragraph{Proof ideas.} 
Our key insight is that, by Item~\ref{assumption:aiming} of Assumption~\ref{assumption:Aproposed}, the iteration sequence $x_k$ generated by the dynamics~\eqref{eqn: updaterule} can be locally but closely approximated by its projection onto $\cM$, namely,~the ``shadow sequence” defined as
$$y_k = P_\cM(x_k).$$
By carefully quantifying the distance between $x_k$ and $y_k$, we show that this error decays sufficiently fast so that the hypothetical batch-means estimator constructed with the shadow sequence $y_k$, similar to~\eqref{eqn:online_batch_means}, converges to the same limit -- and at the same rate -- as the estimator constructed with $x_k$. Consequently, it suffices to analyze the convergence of the batch-means estimator applied to $y_k$.

Another crucial implication of Assumption~\ref{assumption:Aproposed} is that the update rule of $y_k$ can be interpreted as an inexact Riemannian SA algorithm operating on the restriction of $F$ to the manifold $\cM$. More precisely, we show that the shadow sequence exhibits the recursion
$$ y_{k+1} = y_k - \eta_{k+1} F_{\cM}(y_k) - \eta_{k+1} P_{T_{\cM}(y_k)}(\nu_k) + \text{Error}_k.$$
For the sake of illustration, let us first assume that $\text{Error}_k = 0$. Due to Assumption~\ref{ass:basic_assumpt}, the dynamics of $y_k$ are smooth, allowing us to adapt the analysis of batch-means estimators developed in the context of stochastic smooth optimization \citep{chen2020statistical,zhu2023online}. In the more general~setting, we derive sharp upper bounds on the error terms and demonstrate that their contribution to the~covariance estimation error is dominated by the convergence rate established in the smooth case.

\vskip 0.3cm
    
Note that our main result is local and relies on the stopping time $\tau_{\ks, \delta}$. In this regard, we show in the following section that, under sub-Gaussian noise conditions, the iterates remain near the solution with high probability. Our analysis leverages martingale concentration inequalities applied to \eqref{eqn: updaterule}.$\quad$

\section{High probability guarantee}\label{sec:highprob}

So far, we have only made assumptions on $F$ and $G$ locally near $x^\star$, except for assuming the sequence $x_k$ is bounded in expectation (as proved in Appendix~\ref{sec:global_guarantee_lemma}). To establish global~convergence guarantees, we require the following assumption.

\vspace{-0.1cm}

\begin{assumption}\label{assum: aimingtosol}
{\rm
We assume that there are constants $\gamma, C> 0$ such that:
\begin{enumerate}
\item \textbf{(Aiming towards solution)}\label{item:aiming_solution} For any $x \in \RR^d$, we have
$\dotp{G_\eta(x,\nu) - \nu, x - x^\star} \ge \gamma \|x - x^\star\|_2^2 - C\eta(1 +\|x- x^\star\|_2^2 + \|\nu\|_2^2)$.
\item \textbf{(Global steplength)} \label{item:global_steplength} For any $x \in \RR^d$, we have
 $\|G_\eta (x, \nu)\|_2^2 \le C(1+\|x - x^\star\|_2^2+ \|\nu\|_2^2)$.
\end{enumerate}
}
\end{assumption}

Assumption~\ref{assum: aimingtosol} extends the standard strong convexity and Lipschitz gradient conditions commonly assumed in stochastic smooth optimization. In particular, we have $G_\eta(x,\nu) = \nabla f(x) + \nu$ in the case of minimizing a $\gamma$-strongly convex function $f$. Therefore, Item~\ref{item:aiming_solution} is ensured by the \(\gamma\)-strong convexity, since $\langle G_\eta(x,\nu) - \nu,\, x - x^\star \rangle = \langle \nabla f(x),\, x - x^\star \rangle \ge \gamma \|x - x^\star\|_2^2$.
Moreover, the Lipschitz gradient condition implies Item~\ref{item:global_steplength}, as we observe that $
\|G_\eta(x,\nu)\|_2 = \|\nabla f(x) + \nu\| \lesssim \|x - x^\star\| + \|\nu\|$.~Beyond the smooth case, we show in Appendix~\ref{sec:global_guarantee_lemma} that Assumption~\ref{assum: aimingtosol} holds for various nonsmooth SA~methods.

We additionally impose the following light-tail assumption on the noise.

\vspace{-0.1cm}

\begin{assumption}[Light tail]\label{assum: lighttail}
{\rm 
The noise $\nu_{k+1}$ is mean-zero norm sub-Gaussian conditioned on $\cF_k$ with parameter $\sigma/2$, i.e., $\EE_k[\nu_{k+1}] = 0$ and 
$\pro_k\{\|\nu_{k+1}\| \ge \tau \} \le 2\exp(-2\tau^2/\sigma^2)$ for all $\tau >0$. 
}
\end{assumption}

By standard results in high-dimensional statistics~\cite[Lemma 3]{jin2019short}, we know that~$\|\nu_{k+1}\|^2$ is sub-exponential with parameter $c \sigma^2$ conditioned on $\cF_k$, where $c$ is some absolute constant. Below is a high-probability guarantee demonstrating that $x_k$ stays within $B_\delta(x^\star)$ for all sufficiently large $k$. We present its proof in Appendix~\ref{sec:proof_high_prob}. 

\vspace{-0.1cm}

\begin{proposition}\label{prop:highprob}
Suppose Assumptions~\ref{assum: aimingtosol} and~\ref{assum: lighttail} hold. Let $c$ be the universal constant defined above. Suppose also $\eta \le \min\left\{\frac{\gamma}{3C}, \frac{1}{3c\gamma C}\right\}$. Then, for any radius $\delta$ and any $k$ such that
$$\textstyle k \ge \max \left\{ \left( \frac{\log(4\|x_0 - x^\star\|^2/\delta)}{ C_\alpha\gamma \eta }\right)^{1/(1-\alpha)}, \left(\frac{\log\left(\frac{16 \tilde C \alpha \eta^2}{(2\alpha-1)\delta }\right)}{C_\alpha \gamma \eta} \right)^{1/(1-\alpha)}, \left(\frac{2^{2\alpha+2}\tilde C\eta^2 }{(2\alpha - 1)\delta}\right)^{1/(2\alpha - 1)}\right\},$$ 
where $\tilde C= 3cC \sigma^2 + 3C$ and  $C_\alpha = \frac{1-0.5^{1-\alpha}}{2(1-\alpha)}$, we have \vskip-0.2cm
$$\textstyle \pro(\|x_i -x^\star\| < \delta, \forall i \ge k) \ge 1- \frac{32 \eta^2 \sigma^4 \exp\left(- \frac{\gamma\delta\sqrt{k}}{4\eta\sigma^2}\right)}{\gamma^2 \delta^2} - \frac{8\eta \delta \sqrt{k} \exp\left(-\frac{\gamma \delta\sqrt{k}}{4\eta\sigma^2}\right)}{\gamma }.$$
\end{proposition}

With the above high-probability guarantee, we strengthen the local result in Theorem \ref{th:mainthmcov} to a~global result by suppressing the stopping time involved in the theorem statement. Our global result is stated in Theorem~\ref{thm:cov_light_tail}. The proof can be found in Appendix~\ref{sec:proof_cov_lighttail}.

\begin{theorem}\label{thm:cov_light_tail}
Under the assumptions of Theorem~\ref{th:mainthmcov} along with Assumptions~\ref{assum: aimingtosol} and~\ref{assum: lighttail}, for the SA update of \eqref{eqn: updaterule}, for $a_M\leq n\le a_{M+1}$,  we have
\begin{align*}
\textstyle \EE[\|\hat \Sigma_n - \Sigma\|_{op}]  \textstyle \lesssim_{\log}  \sqrt{d}M^{-\frac{1}{2}} + \sqrt{d} M^{\frac{(\alpha-1)\beta +1}{2}} &\textstyle \lesssim \sqrt{d} n^{-\frac{1}{2\beta}} + \sqrt{d} n^{-\frac{(\alpha -1)\beta +1}{2\beta}},
\end{align*} 
where $\|\cdot\|_{op}$ is the operator norm, and ``$\lesssim_{\log}$" hides logarithmic terms of $n$.
\end{theorem}

\vspace{-0.1cm}

Taking $\beta = \frac{2}{1-\alpha}$ in Theorem~\ref{thm:cov_light_tail}, we have $\EE[\|\hat \Sigma_n - \Sigma\|_{op}] \lesssim_{\log} \sqrt{d} n^{-\frac{1-\alpha}{4}}$. Ignoring the logarithmic factors, this matches the best-known rate in the smooth case~\citep{chen2020statistical,zhu2023online}.

\section{Examples of stochastic approximation algorithms}\label{sec:SA_algorithms}

In this section, we illustrate the broad applicability of our generic SA update in \eqref{eqn: updaterule} and the mildness of our required assumptions. In particular, we consider solving nonsmooth problems using different SA algorithms and provide sufficient conditions for Assumptions~\ref{ass:basic_assumpt}--\ref{assumption:Aproposed} to hold.
More~concretely,~let~us consider the variational inclusion problem:
\begin{align}\label{eqn:variation_inclusion}
0 \in  A(x) + \partial g(x) + \partial f(x),
\end{align}
where $A: \RR^d \rightarrow \RR^d$ is any single-valued continuous map, $g:\RR^d \rightarrow \RR\cup\{+\infty\}$ is a closed function, and $f:\RR^d \rightarrow \RR\cup\{+\infty\}$ is a closed function that is bounded from below\footnote{In particular, $\prox_{\alpha f}(x)$ is nonempty for all $x \in \RR^d$ and all $\alpha >0$.}. The problem~\eqref{eqn:variation_inclusion} is a special case of~\eqref{eqn:inclusion_problem} since one can take $F(x) := A(x) + \partial g(x) + \partial f(x)$.
First, the~local boundedness condition of $G$ in Assumption~\ref{assumption:localbound} is widely used in the literature, with a variety of known sufficient conditions. The following lemma describes several such conditions,~which~we~will~use~in~what~\mbox{follows}.

\begin{lemma}[{Lemma 4.2 in \cite{davis2024asymptotic}}]\label{lem:basic_level_bound}
Suppose $A(\cdot)$ and $s_g(\cdot)$ are locally bounded around $x^\star$. Then Assumption~\ref{assumption:localbound} holds in any of the following settings.
\begin{enumerate}[noitemsep]
\item\label{lb:2} $f$ is the indicator function of a closed set $\cX$.
\item\label{lb:3} $f$ is convex and the function $x\mapsto\dist(0,\partial f(x))$ is bounded on $\dom f$ near $x^\star$.
\item\label{lb:4} $f$ is Lipschitz continuous on $\dom g\cap \dom f$.
\end{enumerate}
\end{lemma}

Then, we investigate Assumptions \ref{ass:basic_assumpt} and~\ref{assumption:Aproposed}. Recall that both assumptions require the existence~of a distinctive manifold $\cM$ that captures the hidden smoothness of the problem. One candidate~of such a manifold is the \textit{active manifold}, which has been modeled in various ways, including identifiable surfaces \citep{wright1993identifiable}, partial smoothness \citep{lewis2002active}, $\mathcal{UV}$-structures \citep{lemarecha2000,mifflin2005algorithm}, $g\circ F$ decomposable functions \citep{shapiroreducible}, and minimal identifiable sets \citep{drusvyatskiy2014optimality}. In this work, we adopt the characterization of~active manifold used in~\cite{drusvyatskiy2014optimality}.

\begin{figure}[h]
\centering
\includegraphics[width=.4\linewidth]{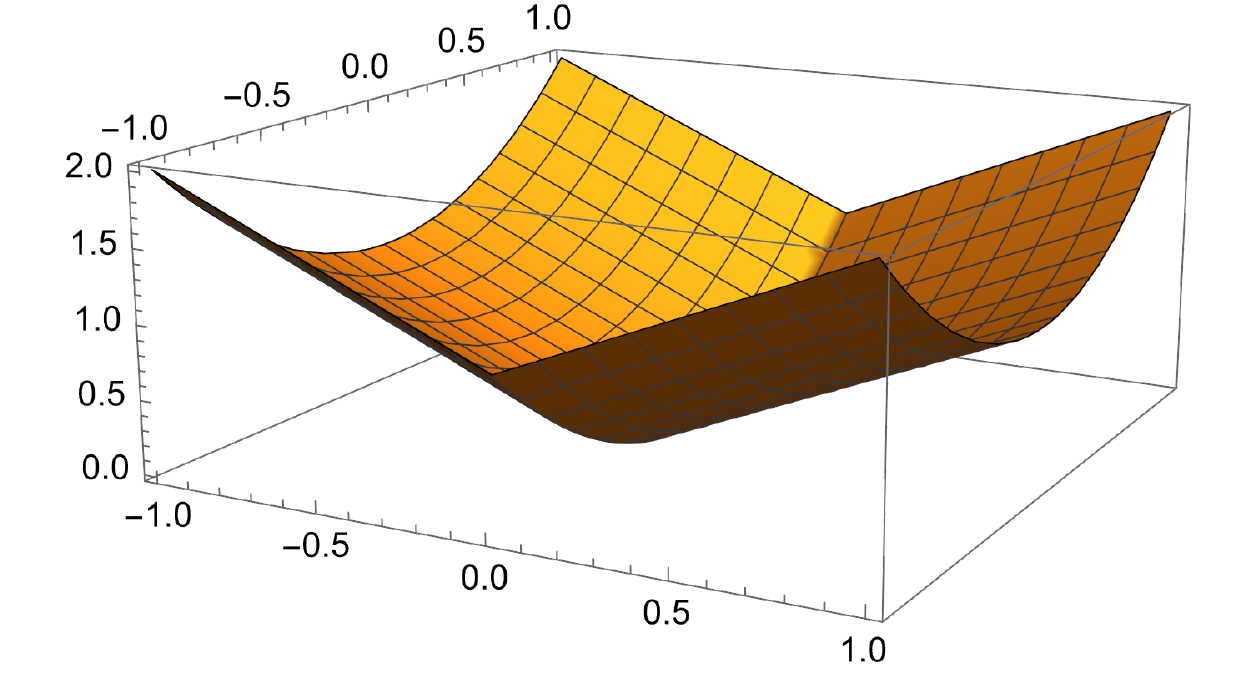}
\captionof{figure}{$f(x_1,x_2)=|x_1|+x^2_2$ with $x_2$-axis as an active manifold.}
\label{fig:test1}
\end{figure}
	
\vspace{-0.6cm}
\begin{definition}[Active manifold]\label{defn:ident_man}{\rm 
Consider a function $f\colon\R^d\to\R\cup\{+\infty\}$ and fix a set $\mathcal{M} \subset \dom f$ that contains a critical point $x^\star$ with $0\in \partial f(x^\star)$. Then $\mathcal{M}$ is called an {\em active} $C^p${\em-manifold around} $x^\star$ if there exists a constant $\chi>0$ satisfying the following conditions.
\begin{itemize}[noitemsep]
\item {\textbf (smoothness)}  
Near $x^\star$, the set $\mathcal{M}$ is a $C^p$ manifold and  the restriction of 	$f$ to $\mathcal{M}$ is $C^p$-smooth.
\item {\textbf (sharpness)} 
The lower bound holds:
$$\inf \{\|v\|: v\in \partial f(x),~x\in U\setminus \cM\}>0$$
where $U=\{x\in B_{\chi}(x^\star):|f(x)-f(x^\star)|<\chi\}$.
\end{itemize}
More generally, we say {\em $\cM$ is an active manifold for $f$ at $x^\star$ for $\bar v \in \partial f(x^\star)$} if $\cM$ is an active~manifold for the tilted function $f_{\bar v}(x)=f(x)-\langle \bar v,x\rangle$ at $x^\star$.
}
\end{definition}

The sharpness condition simply means that the subgradients of $f$ remain uniformly bounded~away from zero at points off the manifold that are sufficiently close to $x^\star$ in both distance and function~value. The localization in function value can be omitted, for example, if $f$ is weakly convex or if $f$~is~continuous on its domain; see \cite{drusvyatskiy2014optimality} for details. Figure~\ref{fig:test1} is an example of~active manifold of a nonsmooth function.

To proceed, we introduce two extra conditions along the active manifold that tightly couple the subgradients of $f$ on and off the manifold. These two conditions were first introduced in \cite[Section 3]{davis2025active} to prove saddle point avoidance in nonsmooth optimization.
They are~very~mild~conditions and hold for a wide range of examples. 
We verify these regularity conditions in detail for the cases of $\ell_1$-regularization, nonlinear programming, and two-player game in Appendix~\ref{sec:examples_active}.

\begin{definition}[$(b_{\leq})$-regularity and strong $(a)$-regularity]{\rm 
Consider a function $f\colon\R^d\to\R\cup\{+\infty\}$ that is locally Lipschitz continuous on its domain. Fix a set $\cM\subset\dom f$ that is a $C^1$ manifold around $x^\star$ and such that the restriction of $f$ to $\cM$ is $C^1$-smooth near $x^\star$. We say that $f$ is {\em $(b_{\leq})$-regular along $\cM$ at $x^\star$} if there exists $\chi>0$ such that 
\begin{align}
f(y)&\geq f(x)+\langle v, y-x\rangle+(1+\|v\|)\cdot o(\|y-x\|)\label{eqn:bd1}
\end{align}
holds for all $x\in \dom f\cap B_{\chi}(x^\star)$, $y\in \cM\cap B_{\chi}(x^\star)$, and $v\in \partial f(x)$. Additionally, we say that $f$ is {\em strongly $(a)$-regular along $\cM$ near $x^\star$} if there exist constants $C,\chi>0$ satisfying
\begin{align}
\|P_{T_{\cM}(y)}(v-\nabla_{\cM} f(y))\|&\leq C(1+\|v\|) \|x-y\| \label{eqn:str_a1}
\end{align}
for all $x\in \dom f\cap B_{\chi}(x^\star)$, $y\in \cM\cap B_{\chi}(x^\star)$, and $v\in \partial f(x)$.
}
\end{definition}

Roughly speaking, $(b_\le)$-regularity condition is a weakening of Taylor's theorem for nonsmooth functions; strong $(a)$-regularity condition is a weakening of Lipschitz continuity of the gradient.~We next provide sufficient conditions of Assumptions~\ref{ass:basic_assumpt} and~\ref{assumption:Aproposed} in several popular settings,~including~projected SGD (hence Subgradient Descent) and projected Stochastic Gradient~Descent~Ascent~methods.

\subsection{Stochastic (projected) forward algorithm $(f=\delta_{\cX})$}

First, we focus on the particular instance of \eqref{eqn:variation_inclusion} where  $f$ is an indicator function of a closed set $\cX$. In this case, the iteration \eqref{eqn: updaterule} reduces to a stochastic projected forward algorithm:
\begin{align}\label{eqn:project_forward}
x_{k+1} \in P_\cX(x_k - \eta_{k+1}(A(x_k) + s_g(x_k) + \nu_{k+1})).
\end{align} 
The map $G$ takes the form $G_{\eta}(x, \perturb) \coloneqq (x - s_\cX(x - \eta( A(x)+s_g(x) + \perturb)))/\eta$, where $s_\cX(x)$ is any selection of the projection map $P_{\cX}(x)$.
	
The following proposition shows that Assumptions~\ref{ass:basic_assumpt} and~\ref{assumption:Aproposed} hold when $g+f$ admits an active manifold at $x^\star$ with certain regularity conditions. Its proof is a combination of Corollary 4.7 and Lemma 10.3 in~\cite{davis2024asymptotic}. 
\begin{proposition}\label{prop:projected_forward}
Suppose $f$ is the indicator function of a closed set $\cX$ and both $g (\cdot)$ and $A(\cdot)$ are Lipschitz continuous around $x^\star$. 
Moreover, suppose the inclusion $-A(x^\star)\in \hat \partial (g+f)(x^\star)$ holds, $g+f$ admits a $C^2$ active manifold around $x^\star$ for the vector $\bar v=-A(x^\star)$, and both $g$ and $f$ are $(b_{\leq})$-regular and strongly $(a)$-regular along $\cM$ at $x^\star$. Then Assumption~\ref{assumption:Aproposed} holds. Furthermore, if there exists $\gamma >0$ such that $\dotp{\nabla_\cM(A + \partial g) (x^\star) v, v} \ge \gamma \|v\|_2^2,~\text{for all } v \in T_\cM(x^\star)$, then Assumption~\ref{ass:basic_assumpt} holds with manifold $\cM$.	
\end{proposition}

\subsection{Stochastic forward-backward method $(g=0)$}\label{prop:forward_backward}

Second, we focus on the particular instance of \eqref{eqn:variation_inclusion} where  $g=0$. In this case, the iteration \eqref{eqn: updaterule} reduces to a stochastic forward-backward algorithm:
\begin{align}\label{eqn:forward_backward}
x_{k+1} \in \prox_{\eta_{k+1} f}(x_k - \eta_{k+1}(A(x_{k}) + \nu_{k+1})).
\end{align}
The map $G$ becomes $G_{\eta}(x, \perturb) := (x - s_f(x - \eta( A(x) + \perturb))/\eta)$, where $s_f$ is any selection of the proximal map $\prox_{\eta f}(x)$ (cf. Section \ref{sec:2}). 
	
The following proposition shows that Assumptions~\ref{ass:basic_assumpt} and~\ref{assumption:Aproposed} hold when $f$ admits an active manifold at $x^\star$ with certain regularity conditions. Its proof is a combination of Corollary 4.9 and Lemma 10.3 in~\cite{davis2024asymptotic}.  

\begin{proposition}
Suppose $g=0$ and both $f$ and $A(\cdot)$ are Lipschitz continuous on $\dom f$ near $x^\star$. Moreover, suppose the inclusion $-A(x^\star)\in \hat \partial f(x^\star)$ holds, $f$ admits a $C^2$ active manifold around $x^\star$ for $\bar v=-A(x^\star)$, and $f$ is both $(b)_{\leq}$-regular and strongly $(a)$-regular along $\cM$ at $x^\star$. Then~Assumption~\ref{assumption:Aproposed} holds. Furthermore, if there exists $\gamma >0$ such that
$$\dotp{\nabla_\cM(A + \partial f) (x^\star) v, v} \ge \gamma \|v\|_2^2, \qquad  \text{for all } v \in T_\cM(x^\star),$$
then Assumption~\ref{ass:basic_assumpt} holds with manifold $\cM$.
\end{proposition}

\section{Conclusion and future work}\label{sec:6}

In this paper, we studied covariance estimation for nonsmooth stochastic approximation (SA)~methods. The estimator was initially proposed for SGD in \cite{zhu2023online} for smooth, strongly convex optimization problems.
The key idea is to group iterates into blocks of increasing size, with each block providing an approximately independent estimate of the covariance matrix. This estimator can be computed fully online, with both computation and memory scaling quadratically in dimension. Our work demonstrated that, with a properly chosen batch size control sequence, the same estimator achieves the expected convergence rate of order $O(\sqrt{d}n^{-1/8+\varepsilon})$ for any $\varepsilon>0$ in nonsmooth and potentially non-monotone (nonconvex) setting.
Our analysis involves highly nontrivial extensions of \cite{zhu2023online}, where we developed a localization technique and constructed a shadow sequence to address the challenges arising from the lack of smoothness. Additionally, we established high-probability guarantees on the stopping time at which iterates leave the local neighborhood.
The consistency of our covariance estimator enables asymptotically valid statistical inference for stochastic nonsmooth variational inclusion problems, covering numerous examples~as~provided~in~\mbox{Appendix}~\ref{sec:examples_active}.

One future research direction is studying covariance estimation for nonsmooth SA methods~under Markovian noise, inspired by reinforcement learning applications. In addition, an open and challenging question is establishing the lower bound of covariance estimation and investigating whether the estimator \eqref{eqn:online_batch_means} for first-order methods is minimax optimal. Finally, designing non-asymptotically optimal (nonsmooth) SA methods along with suitable covariance estimators is also a promising topic for future research.

\section*{Acknowledgment}
Liwei Jiang is partly supported by the Gatech ARC postdoc fellowship. Krishna Balasubramanian is supported by NSF DMS-2413426 and NSF DMS 2053918 awards. Damek Davis is supported by an Alfred P. Sloan research fellowship and NSF
DMS 2047637 award. Dmitriy Drusvyatskiy is supported by NSF DMS-2306322, NSF CCF 1740551, AFOSR FA9550-24-1-0092 awards.
\clearpage

\bibliographystyle{unsrt}
\bibliography{covest}

\appendix
\newpage

\setcounter{example}{0}

\section{Concrete Examples }\label{sec:examples_active}

In this section, we expand on the discussion in Section \ref{sec:3} and provide some concrete examples that satisfy Assumptions~\ref{ass:basic_assumpt}--\ref{assumption:Aproposed}. 

\begin{example}[$\ell_1$-regularization]\label{ex:l1reg}
{\rm
Consider the stochastic optimization problem with $\ell_1$ regularization
\begin{equation} \label{eqn:l1reg}
\min_{x}~ g(x) = f(x) + \lambda \|x\|_1,
\end{equation}
where $f(x) = \EE_{\nu \in \cP}[f(x, \nu)]$ is a $C^p$-smooth function in $\RR^d$. Consider now a point $x^\star \in \RR^d$ that is critical for the function $g$ and define the index set $\cI = \{i \colon x^\star_i = 0\}$. Then, the set
$$\cM = \{ x\colon x_i = 0,~\forall i \in \cI\}$$
is an affine space, hence a smooth manifold. Note that the definition of criticality ensures that~$0 \in \partial g(x^\star)$, so we always have 
$$-(\nabla f(x^\star))_{i} \in [-\lambda, \lambda],\qquad \forall i \in \cI.$$
Suppose the following condition is true:
\begin{itemize}
\item {\bf (Strict complementarity)} $- (\nabla f(x^\star))_i \in (-\lambda, \lambda)$ for all $i\in \mathcal{I}$.
\end{itemize}
Then $\cM$ is indeed an active $C^p$ manifold of $g$ at $x^\star$. Moreover, $(b_{\le})$-regularity and strong (a)-regularity hold trivially for $g$ along $\cM$ at $x^\star$. If, in addition, $\nabla^2 f(x^\star)$ is positive definite when restricted to the tangent space of $\cM$, then Proposition~\ref{prop:projected_forward} and Lemma~\ref{lem:basic_level_bound} imply that Assumptions~\ref{ass:basic_assumpt}--\ref{assumption:Aproposed} hold for the stochastic subgradient method; similarly, Proposition~\ref{prop:forward_backward} and Lemma~\ref{lem:basic_level_bound} imply that these assumptions also hold for the stochastic proximal gradient method.
We mention that there is typically a bias between the center of the asymptotic normality, $x^\star$, and the minimizer of $f$ due to the presence of the regularization term.
}
\end{example}

\begin{example}[Nonlinear programming]\label{ex:nlp}
{\rm
Consider the problem of nonlinear programming
\begin{equation}\label{eqn:nlp}
\begin{aligned}
\min_{x}~ &f(x),\\
{\rm s.t.}~\,&g_i(x)\leq 0\qquad \textrm{for } i=1,\ldots,m,\\
&g_i(x)= 0\qquad \textrm{for }i=m+1,\ldots,n,
\end{aligned}
\end{equation}
where $f$ and $g_i$ are $C^p$-smooth functions on $\R^d$. Let us denote the set of all feasible points to the problem as
\begin{equation*}
\cX = \{x: g_i(x)\leq 0\;\text{ for }\; 1\leq i\leq m \quad \text{ and }\quad g_i(x)= 0 \; \text{ for }\; m+1\leq i\leq n\}.
\end{equation*}
Consider now a point $x^\star\in \cX$ that is critical for the function $f+\delta_{\cX}$ and define the active index set 
$$\mathcal{I}=\{i: g_i(x^\star)=0\}.$$
Suppose the following is true:
\begin{itemize}
\item {\bf (LICQ)} the gradients $\{\nabla g_i(x^\star)\}_{i\in \mathcal{I}}$ are linearly independent.
\end{itemize}
Then the set 
$$\cM=\{x: g_i(x)=0, ~\forall i\in \cI\}$$
is a $C^p$ smooth manifold locally around $x^\star$. Moreover, all three functions  $f$, $\delta_{\cX}$, and $f+\delta_{\cX}$ are $(b_{\leq})$-regular and strongly $(a)$-regular along $\cM$ near $x^\star$. To ensure that $\cM$ is an active manifold of $f+\delta_{\cX}$, an extra condition is required. Define the Lagrangian function 
$$\mathcal{L}(x,y):=f(x)+\textstyle\sum_{i=1}^{n+m} y_i g_i(x).$$
The criticality of $x^\star$ and LICQ ensure that there exists a (unique) Lagrange multiplier vector~$y^\star\in \R^{m}_+\times \R^n$ satisfying $\nabla_x \mathcal{L}(x^\star, y^\star)=0$ and $y_i^\star=0$ for all $i\notin \cI$.
Suppose the following standard assumption is true:
\begin{itemize}
\item {\bf (Strict complementarity)} $y_i^\star>0$ for all $i\in \mathcal{I}\cap \{1,\ldots, m\}$.
\end{itemize}
Then $\cM$ is indeed an active $C^p$ manifold for $f+\delta_{\cX}$ at $x^\star$.  Assume in addition that $\nabla^2_{xx} \mathcal{L}(x^\star,y^\star)$ is positive definite when restricted onto $T_\cM(x^\star)$, often called the  Second-Order Sufficient Condition (SOSC) in nonlinear programming literature \citep{nocedal2006numerical}; Proposition~\ref{prop:projected_forward} and Lemma~\ref{lem:basic_level_bound} imply that Assumptions~\ref{ass:basic_assumpt}--\ref{assumption:Aproposed} hold for stochastic projected gradient method. 
}
\end{example}

\begin{example}[Entropy-regularized zero-sum two-player matrix game]
{\rm 
Consider the following optimization problem that arises in an zero-sum two-player matrix game \citep{cen2021fast,li2022convergence}
\begin{align*}
\argmin_{z\in\Delta^{d-1}}\argmax_{w\in\Delta^{d-1}} f(z,w)\coloneqq z^\top\EE\left[\mathcal{A}_\xi\right]w+\lambda \mathcal{H}(z)-\lambda \mathcal{H}(w),\numberthis\label{eq:gameex}
\end{align*}
where $\Delta^{d-1}$ is the $d$-dimensional probability simplex, $\lambda$ is the regularization parameter, and $\mathcal{H}(\mu)=-\sum_{i=1}^d\mu_i\log \mu_i$ is the entropy regularization. The regularization is often imposed to account for the imperfect knowledge about the payoff matrix $\mathcal{A}=\EE\left[\mathcal{A}_\xi\right]$~\citep{mertikopoulos2016learning}. The solution of the above problem is known as the Quantal Response Equilibrium (QRE) in game theory \citep{mckelvey1995quantal}. In particular, the solution of \eqref{eq:gameex} turns out to be the solution of the following fixed point equation:
\begin{align*}
z_i^\star\propto \exp([Aw^*]_i/\lambda)\qquad w_i^\star\propto \exp(-[Az^\star]_i/\lambda) \qquad \forall 1\leq i\leq d. \numberthis\label{eq:gamesol}
\end{align*}
Let $\cX = \Delta^{d-1} \times \Delta^{d-1} \subset \RR^{2d}$, then problem~\eqref{eq:gameex} can be reformulated as the following variational inclusion problem:
\begin{align*}
0 \in \begin{bmatrix}
\nabla_z f(z,w)\\
-\nabla_w f(z,w)
\end{bmatrix} + N_\cX(z, w).
\end{align*}
Observe that $(z^\star,w^\star)$ lies in the relative interior of $\Delta^{d-1}\times\Delta^{d-1}$.  Consequently, $$\mathcal{M}\coloneqq \left\{(z,w): \textstyle\sum_{i=1}^{d} z_i = 1, \sum_{i=1}^{d}w_i =1 \right\}$$ is an active manifold of $\delta_\cX$ at $(z^\star, w^\star)$ for $- \begin{bmatrix}
\nabla_z f(z,w)\\
-\nabla_w f(z,w)
\end{bmatrix}$. Also, it is trivial to show that $\delta_\cX$ is both $(b_\le)$-regular and strong $(a)$-regular along $\cM$ at $(z^\star, w^\star)$. Moreover, ~\cite{cen2021fast} showed that $f$ is strongly-convex strongly-concave locally near $(z^\star, w^\star)$, so a combination of Proposition~\ref{prop:projected_forward} and Lemma~\ref{lem:basic_level_bound} implies that Assumptions~\ref{ass:basic_assumpt}--\ref{assumption:Aproposed} hold for stochastic projected forward method.
}
\end{example}

The following is a nonconvex and nonsmooth function satisfying Assumptions~\ref{ass:basic_assumpt}--\ref{assumption:Aproposed} for the stochastic subgradient method.

\begin{example}[Nonconvex example]\label{example:nonconvex}
{\rm 
Consider the function with the origin as the minimizer:
$$ f(x,y) = |x-y^2| + \frac{x^2 + y^2}{2}.$$
Note that for any $ 0< t<1$, we have
\begin{align*}
f(t^2, t) + f(t^2, -t) = (t^4+t^2) <  2t^2 + t^4 = 2f(t^2, 0),
\end{align*}
which implies that $f$ is not convex in any local neighborhood of the origin.
Meanwhile, one can easily check that $\cM  = \{(x,y) \colon x= y^2\}$ is an active manifold of $f$ at the origin, and $f$ is both $(b_\le)$-regularity and strong $(a)$-regularity along $\cM$ at the origin. Moreover, $\nabla_\cM f(0,0)$ is positive definite on the $y$-axis. A combination of Proposition~\ref{prop:projected_forward} and Lemma~\ref{lem:basic_level_bound} implies that Assumptions~\ref{ass:basic_assumpt}--\ref{assumption:Aproposed} hold for stochastic subgradient method.
}
\end{example}

	\section{Proof of Theorem~\ref{th:mainthmcov}}\label{sec:proof_of_main}
	We introduce some more notations for the rest of this section. First, by our choice that $a_m = \floor{Cm^\beta}$, we have $n_m \asymp m^{\beta-1}$. Let $M$ be an integer such that $a_M \le n <a_{M+1}$. Let $H\coloneqq U^\top\nabla_{\cM}F_\cM(x^\star)U$.	Note that $H$ is not necessarily a symmetric matrix~\cite{davis2024asymptotic}. Define 
	\begin{align}
	W_i^j&\coloneqq\textstyle\prod_{k=i+1}^j(\Id-\eta_kH) \quad \text{for}~~ j>i \quad\text{with}\quad W_i^i\coloneqq\Id,\label{eq:wijdef}\\[2.5pt]
	&S_i^j\coloneqq\textstyle\sum_{k=i+1}^jW_i^k \quad \text{for}~~ j>i \numberthis\label{eq:sijdef}\quad\text{with}\quad S_i^i\coloneqq0.
	\end{align}
	 Let $\delta >0$ be small enough so that Assumption~\ref{ass:basic_assumpt} -- \ref{assumption:Aproposed} hold inside $B_\delta(x^\star)$. We consider the shadow sequence 
	\begin{align*}
	y_k = \begin{cases}
	P_{\cM}(x_k) & \text{if $x_k \in B_{2\delta}(x^\star)$} \\
	x^\star &  \text{otherwise.}
	\end{cases}
	\end{align*}
	By Proposition 6.3 in \cite{davis2024asymptotic}, there exists $\cF_{k+1}$-measurable random vectors $E_k \in \RR^d$ such that the shadow sequence satisfies $y_k \in B_{4\delta}(x^\star) \cap \cM$ for all $k$ and the recursion holds:
	\begin{align}\label{eqn:shadow:eq:iteration}
	y_{k+1} = y_k - \eta_{k+1} F_{\mathcal{M}}(y_k) - \eta_{k+1} P_{\tangentM{y_k}}(\perturb_{k+1}) + \eta_{k+1} E_k \qquad \text{for all $k \geq 1.$}
	\end{align}
	Define an auxiliary sequence $z_k=x^\star+U\Delta_k$ where $\Delta_k :=U^\top (y_k-x^\star)$. 
	Consider the following two estimators defined in terms of $z_k$ and $\Delta_k$ respectively.  
	\begin{align}
	{\Sigma}_n'=\frac{\sum_{i=1}^n\left(\sum_{k=t_i}^i(z_k-x^\star)-l_i(\bar{z}_n-x^\star)\right)\left(\sum_{k=t_i}^i(z_k-x^\star)-l_i(\bar{z}_n-x^\star)\right)^
		\top}{\sum_{i=1}^nl_i}. \label{eq:covestimatorzk}
	\end{align}
    \begin{align}	\tilde{\Sigma}_n=\frac{\sum_{i=1}^n\left(\sum_{k=t_i}^i\Delta_k-l_i\bar{\Delta}_n\right)\left(\sum_{k=t_i}^i\Delta_k-l_i\bar{\Delta}_n\right)^
		\top}{\sum_{i=1}^nl_i}. \label{eq:covestimatortilde}
	\end{align}
    Observing that $\Sigma_n'=U\tilde{\Sigma}_nU^\top$ and recalling from~\eqref{eq:covdef} that $\Sigma=UH^{-1}SH^{-\top}U^\top$, we have, 
    \begin{align}
        \expec{\norm{{\Sigma}_n'-\Sigma}_2\indic{n}}{} &=\expec{\norm{U(\tilde{\Sigma}_n-H^{-1}SH^{-\top})U^\top}_2\indic{n}}{}\\
        &\le \expec{\norm{\tilde{\Sigma}_n-H^{-1}SH^{-\top}}_2\indic{n}}{}.\label{eq:sigmaprimetildesame}
    \end{align}
    Using triangle inequality and \eqref{eq:sigmaprimetildesame}, we have
    \begin{align*}
        \expec{\norm{\hat{\Sigma}_n-\Sigma}_2\indic{n}}{}\leq~ & \expec{\norm{{\Sigma}_n'-\Sigma}_2\indic{n}}{}+\expec{\norm{\hat{\Sigma}_n-\Sigma_n'}_2\indic{n}}{}\\
        \le~ & \underset{\textsf{Lemma~\ref{lm:sigmanhatsigmadiff}}}{\colorboxed{black}{\expec{\norm{\tilde{\Sigma}_n-H^{-1}SH^{-\top}}_2\indic{n}}{} }} +\underset{\textsf{Lemma~\ref{lm:sigmanhatsigmanprimeconv}}}{\colorboxed{black}{\expec{\norm{\hat{\Sigma}_n-\Sigma_n'}_2\indic{n}}{}}}.
    \end{align*}
    	On the one hand, by Lemma~\ref{lm:sigmanhatsigmadiff} and the assumption that $\beta > \frac{1}{1-\alpha}$,
	\begin{align*}
	&\quad \textstyle \expec{\norm{\tilde{\Sigma}_n-H^{-1}SH^{-\top}}_2\indic{n}}{}\\
	&\lesssim  d \ks^{\alpha} M^{(\alpha-1)\beta+1}+\sqrt{d}\ks^2 M^{-\frac12}  +\sqrt{d} \ks^\frac{\alpha}{2}M^{\frac{(\alpha-1)\beta + 1}{2}}+ \ks^{\alpha+\frac{1}{2}}M^{-\frac{1}{2}} + \ks^{2\alpha +1} M^{-1}\\
    &\lesssim \ks^3 (dM^{(\alpha-1)\beta +1} + \sqrt{d} M^{\frac{(\alpha-1)\beta +1}{2}} + \sqrt{d} M^{-\frac{1}{2}})\numberthis\label{eq:sigmantildesigmadiff}
	\end{align*}
	On the other hand, by Lemma~\ref{lm:sigmanhatsigmanprimeconv} and the assumption that $\beta > \frac{1}{1-\alpha}$,  
	\begin{align*}
	&\textstyle \quad \expec{\norm{\hat{\Sigma}_n-\Sigma_n'}_2\indic{n}}{}\\
    &\lesssim\sqrt{d}\ks^{\frac32+\frac{\alpha}{2}} M^{\frac{(\alpha-1)\beta}{2}} +d^\frac14 \ks^\frac52M^{-\frac34}+d^\frac14 \ks^{\frac{\alpha}{4}+\frac32} M^{\frac{(\alpha-1)\beta -1}{4}}+ \ks^{\frac{3}{2}}M^{-\frac{1}{2}} + \ks^{3} M^{-1}\\
    &\lesssim \ks^3\sqrt{d}M^{-\frac{1}{2}} \numberthis\label{eq:sigmanhatsigmanprimeconv}
	\end{align*}

	Combining ~\eqref{eq:sigmantildesigmadiff}, and~\eqref{eq:sigmanhatsigmanprimeconv} and using the fact that $n\asymp M^\beta$, we conclude the proof of Theorem~\ref{th:mainthmcov}.

	\begin{lemma}\label{lm:sigmanhatsigmadiff}
		Let the conditions of Theorem~\ref{th:mainthmcov} be true. We have,
		\begin{align*}
		&\quad  \textstyle \expec{\norm{\tilde{\Sigma}_n-H^{-1}SH^{-\top}}_2\indic{n}}{}\\
		&\lesssim d \ks^{\alpha} M^{(\alpha-1)\beta+1}+\sqrt{d}\ks^2 M^{-\frac12}  +\sqrt{d} \ks^\frac{\alpha}{2}M^{\frac{(\alpha-1)\beta + 1}{2}}+ \ks^{\alpha+\frac{1}{2}}M^{-\frac{1}{2}} + \ks^{2\alpha +1} M^{-1}.\numberthis\label{eq:sigmanhatsigmadiff}
		\end{align*}
	\end{lemma}   
	\begin{proof}\label{pf:sigmanhatsigmadiff}
    Following the proof of Lemma 10.7 in~\cite{davis2024asymptotic}, we have
{\small
\begin{align}\label{eq:Deltadef}
  \Delta_{k+1} =(I-\eta_{k+1} H)\Delta_k  
  -\eta_{k+1} \Bigl(U^\top \bigl(\perturb_{k+1}^{(1)}
  + \perturb_{k+1}^{(2)}(x_k)\bigr)\Bigr)
  - \eta_{k+1} \Bigl(R(y_k) + \zeta_{k+1} - U^\top E_k \Bigr),
\end{align}
}
	where $\zeta_{k+1} = U^\top P_{T_\cM(y_k)}(\nu_{k+1}) - U^\top P_{T_\cM(x^\star)}(\nu_{k+1})$, and 
	$$
	R(y) = U^\top F_\cM(y) - U^\top \nabla_\cM F_{\cM}(x^\star) UU^\top (y - x^\star).
	$$
		
    Summing both sides of \eqref{eq:Deltadef} from $k=i$ to $j$, we get
		\begin{align}
		\textstyle\sum_{k=i}^j\Delta_k=&S_{i-1}^j\Delta_{i-1}+\textstyle\sum_{k=i}^j(\Id+S_k^j)\eta_k\left(U^\top (\perturb_{k+1}^{(1)}+\perturb_{k+1}^{(2)}(x_k))+R(y_k)- U^\top E_k+ \zeta_{k+1} \right)\\
		=& \lambda_i^j+e_i^j, \label{eqn:sum_of_Delta_k}
		\end{align}
		where we define
		\begin{align*}
		&\lambda_i^j:=S_{i-1}^j\Delta_{i-1}+\textstyle	\sum_{k=i}^j(\Id+S_k^j)\eta_k\left(U^\top (\perturb_{k+1}^{(1)}+\perturb_{k+1}^{(2)}(x_k))\right),\\ 
		&e_i^j:=\textstyle	\sum_{k=i}^j(\Id+S_k^j)\eta_k\left(R(y_k) - U^\top E_k+ \zeta_{k+1}\right). 
	\end{align*}
		Plugging~\eqref{eqn:sum_of_Delta_k} into the definition of $\tilde \Sigma_n$ in~\eqref{eq:covestimatortilde}, we write and divide $\tilde \Sigma_n$ into four parts. 
		{\small \begin{align*}
		&\tilde{\Sigma}_n=(\textstyle\sum_{i=1}^n l_i)^{-1}[\underbrace{\textstyle\sum_{i=1}^n(\lambda_{t_i}^i-n^{-1}l_i\lambda_1^n)(\lambda_{t_i}^i-n^{-1}l_i\lambda_1^n)^\top}_{\textsf{I}}+\underbrace{\textstyle\sum_{i=1}^n(e_{t_i}^i-n^{-1}l_ie_1^n)(\lambda_{t_i}^i-n^{-1}l_i\lambda_1^n)^\top}_{\textsf{II}}]\\
		&+(\textstyle\sum_{i=1}^n l_i)^{-1}[\underbrace{\textstyle\sum_{i=1}^n(\lambda_{t_i}^i-n^{-1}l_i\lambda_1^n)(e_{t_i}^i-n^{-1}l_ie_1^n)^\top}_{\textsf{III}}+\underbrace{\textstyle\sum_{i=1}^n(e_{t_i}^i-n^{-1}l_ie_1^n)(e_{t_i}^i-n^{-1}l_ie_1^n)^\top}_{\textsf{IV}}].\numberthis\label{eq:sigmatildedecomp}
		\end{align*}
        }
        In what follows, we will provide upper bounds on $\EE[\|(\sum_{i=1}^n l_i)^{-1} \textsf{I} - H^{-1} SH^{-\top}\|_2 \indic{n}]$, $\EE[\|(\sum_{i=1}^n l_i)^{-1} \textsf{II}\|_2\indic{n}]$, $\EE[(\sum_{i=1}^n l_i)^{-1} \textsf{III}\|_2 \indic{n}]$, and $\EE[(\sum_{i=1}^n l_i)^{-1} \textsf{IV}\|_2 \indic{n}]$ separately. The lemma then follows from the triangle inequality.\\
        
         \noindent \textbf{Analysis of term $\textsf{I}$:}		Note that the goal is to bound 
         \begin{align*}
            &\quad \textstyle \EE[\|(\sum_{i=1}^n l_i)^{-1} \textsf{I} - H^{-1} SH^{-\top}\|_2 \indic{n}]\\
            &\textstyle = \EE[\|(\sum_{i=1}^n l_i)^{-1} \sum_{i=1}^{n}\lambda_{t_i}^i {\lambda_{t_i}^i}^\top - H^{-1} SH^{-\top}\|_2 \indic{n}] + \EE[(\sum_{i=1}^n l_i)^{-1} n^{-1}\|\sum_{i=1}^{n}l_i\lambda_{t_i}^i {\lambda_{1}^n}^\top \|_2 \indic{n}]\\
            &\textstyle\quad +\EE[(\sum_{i=1}^n l_i)^{-1} n^{-1}\|\sum_{i=1}^{n}l_i\lambda_{1}^n {\lambda_{t_i}^i }^\top \|_2 \indic{n}] + \expec{(\sum_{i=1}^nl_i)^{-1}n^{-2}\sum_{i=1}^nl_i^2\norm{\lambda_1^n{\lambda_1^n}^\top}_2\indic{n}}{}.
         \end{align*}
         We bound terms on the RHS one by one. 
         \begin{itemize}
             \item The first term $\EE[\|(\sum_{i=1}^n l_i)^{-1} \sum_{i=1}^{n}\lambda_{t_i}^i {\lambda_{t_i}^i}^\top - H^{-1} SH^{-\top}\|_2 \indic{n}]$. To this end, we rewrite 
		\begin{align*}
		\textstyle\sum_{i=1}^n\lambda_{t_i}^i{\lambda_{t_i}^i}^\top
		=\textstyle\sum_{i=1}^n (\upsilon_i+\omega_i)(\upsilon_i+\omega_i)^\top,
		\end{align*}
		where 
        $$\upsilon_{i} :=S_{t_i-1}^i\Delta_{t_i-1}+\sum_{k=t_i}^i(\eta_k\Id+\eta_kS_k^i-H^{-1})(U^\top (\perturb_{k+1}^{(1)}+\perturb_{k+1}^{(2)}(x_k)))$$ 
        and 
        $$\omega_i:=\sum_{k=t_i}^iH^{-1} U^\top (\perturb_{k+1}^{(1)}+\perturb_{k+1}^{(2)}(x_k))).$$ 
		Note that 
		\begin{align*}
		\expec{\norm{\upsilon_{i}\upsilon_{i}^\top}_2\indic{n}}{}\leq \trace{\expec{\upsilon_{i}\upsilon_{i}^\top\indic{n}}{}} \leq d\norm{\expec{\upsilon_{i}\upsilon_{i}^\top}{}\indic{n}}_2.\numberthis\label{eq:upsupsnormout}
		\end{align*}
		On the other hand, direct calculation shows
		\begin{align*}
		&\quad \norm{\expec{\upsilon_{i}\upsilon_{i}^\top\indic{n}}{}}_2\\
        & \leq \norm{\expec{\upsilon_{i}\upsilon_{i}^\top\indic{t_i-1}}{}}_2\\ 
		&\le \norm{S_{t_i-1}^i}_2^2\norm{\expec{\Delta_{t_i-1}\Delta_{t_i-1}^\top\indic{t_i-1}}{}}_2 \numberthis\label{eq:upsups}\\
		&\quad +\textstyle\sum_{k=t_i}^i\norm{\eta_k\Id+\eta_kS_k^i-H^{-1}}_2^2\norm{U^\top \expec{(\perturb_{k+1}^{(1)}+\perturb_{k+1}^{(2)}(x_k))(\perturb_{k+1}^{(1)}+\perturb_{k+1}^{(2)}(x_k))^\top}{}U}_2,
		\end{align*}
        where the first inequality follows from the definition of the stopping time, and the second inequality follows from  Assumption~\ref{assumption:martinagle} that $\{\nu_{k+1}^{(1)}\}$ and $\{\nu_{k+1}^{(2)}(x_k)\}$ are martingale difference sequences. We then bound the RHS of ~\eqref{eq:upsups}.
		For the first term in~\eqref{eq:upsups}, we consider two cases:
		\begin{enumerate}
			\item $ t_i - 1 \ge k_s.$ 
			Using Lemma~\ref{lm:yk4momentconv} and Lemma~\ref{lm:Sijbound}, we have, 
			{\small \begin{align*}
			\norm{S_{t_i-1}^i}_2^2\norm{\expec{\Delta_{t_i-1}\Delta_{t_i-1}^\top\indic{t_i-1}}{}}_2 &\le 	\norm{S_{t_i-1}^i}_2^2 \EE[\|\Delta_{t_i-1}\|_2^2 \indic{t_i-1}]\\
            &\lesssim k_s^{\alpha} t_i^\alpha. \numberthis\label{eq:upsups1}
			\end{align*}
            }
			\item $t_i -1 < k_s$. By the definition of $y_i$, we always have $\|\Delta_{t_i-1}\| \le 4\delta$. Applying Lemma~\ref{lm:Sijbound}, we have 
			\begin{align*}
			\norm{S_{t_i-1}^i}_2^2\norm{\expec{\Delta_{t_i-1}\Delta_{t_i-1}^\top\indic{t_i-1}}{}}_2\lesssim t_i^{2\alpha} \lesssim k_s^{2\alpha}. \numberthis\label{eq:upsups1+}
			\end{align*}
		\end{enumerate}
		Next, we consider the second term on the RHS of~\eqref{eq:upsups}. By Assumption~\ref{assumption:martinagle} and~\ref{assum:bounded_seq}, we have
		\begin{align*}
		\norm{U^\top\expec{(\perturb_{k+1}^{(1)}+\perturb_{k+1}^{(2)}(x_k))(\perturb_{k+1}^{(1)}+\perturb_{k+1}^{(2)}(x_k))^\top}{}U}_2 &\lesssim  \EE[\|\perturb_{k+1}^{(1)}\|_2^2] + \EE[\|\perturb_{k+1}^{(2)}(x_k)\|_2^2]\\
		&\lesssim \EE[\|x_k - x^\star\|_2^2]\\
		&\lesssim \cub.
		\end{align*}
		In addition, following the same proof of \cite[Lemma B.3]{zhu2023online} we obtain, 
		\begin{align*}
		\textstyle\sum_{k=t_i}^i\norm{\eta_k\Id+\eta_kS_k^i-H^{-1}}_2^2\lesssim l_it_i^{2\alpha-2}+i^\alpha.
		\end{align*}
		Combining, we have
		 \begin{align*}
		\textstyle & \quad\sum_{k=t_i}^i\norm{\eta_k\Id+\eta_kS_k^i-H^{-1}}_2^2\norm{U^\top \expec{(\perturb_{k+1}^{(1)}+\perturb_{k+1}^{(2)}(x_k))(\perturb_{k+1}^{(1)}+\perturb_{k+1}^{(2)}(x_k))^\top}{}U}_2\\
        &\textstyle \lesssim l_it_i^{2\alpha-2}+i^\alpha.\numberthis\label{eq:upsups2}
		\end{align*}
		By basic calculus and our choice of $a_m$ and $n_m$, we can easily verify the following three inequalities:
		\begin{align*}
		\textstyle\sum_{i=1}^n l_i\asymp \textstyle\sum_{m=1}^M n_m^2 \asymp \textstyle\sum_{m=1}^M m^{2\beta-2}\asymp M^{2\beta-1}; \numberthis\label{eq:lisumbound}
		\end{align*} 
		\begin{align*}
		\textstyle\sum_{i=1}^n l_i^2\asymp \textstyle\sum_{m=1}^M n_m^3 \asymp \textstyle\sum_{m=1}^M m^{3\beta-3}\asymp M^{3\beta-2}; \numberthis\label{eq:lisqrsumbound}
		\end{align*}
		\begin{align*}
		\textstyle \sum_{m=1}^M a_m^{-2\alpha} n_m^3 \asymp \sum_{m=1}^M m^{3\beta - 2 \alpha \beta -3} \asymp M^{3\beta - 2\alpha \beta - 2}. \numberthis \label{eq:amnmbound}
		\end{align*}
		Combining \eqref{eq:upsupsnormout}, \eqref{eq:upsups}, \eqref{eq:upsups2}, \eqref{eq:upsups1}, and \eqref{eq:lisumbound}, we have 
		\begin{align*}
		\textstyle\sum_{i=1}^n\expec{\norm{\upsilon_{i}\upsilon_{i}^\top}_2\indic{n}}{} &\lesssim d\left[\textstyle\sum_{i=1}^n(l_it_i^{2\alpha-2}+i^\alpha+\ks^\alpha t_i^\alpha+\ks^{2\alpha})\right]\\
		=&d \left[\textstyle\sum_{m=1}^M\sum_{i=a_m}^{a_{m+1}-1}(l_ia_m^{2\alpha-2}+i^\alpha+\ks^\alpha a_m^\alpha+\ks^{2\alpha})\right]  \\
		\lesssim & d\left[M^{2\alpha\beta-1}+M^{\beta(1+\alpha)}+\ks^\alpha M^{\beta(1+\alpha)}+\ks^{2\alpha}M^\beta \right].
		\end{align*}
		Then, by~\eqref{eq:lisumbound} and  the assumption that $n\geq \ks$,
		{\small \begin{align*} \textstyle(\sum_{i=1}^nl_i)^{-1}\sum_{i=1}^n\expec{\norm{\upsilon_{i}\upsilon_{i}^\top}_2\indic{n}}{} &\lesssim d \ks^\alpha M^{(\alpha-1)\beta+1} + d\ks^{2\alpha} M^{1-\beta}\\
		&\lesssim d \ks^{\alpha} M^{(\alpha-1)\beta+1}.\numberthis\label{eq:upsupsbound}
		\end{align*}
        }
		Define $\hat \omega_i = \sum_{k = t_i}^{i} H^{-1} U^\top \nu_{k+1}^{(1)}$.  
		Using the same proof of Step 1 of~\cite[Lemma B.2]{zhu2023online}, we have, 
		\begin{align*}
		\expec{\textstyle\norm{(\sum_{i=1}^n l_i)^{-1}\sum_{i=1}^n \hat \omega_i \hat \omega_i^\top-H^{-1}SH^{-\top}}_2}{}\lesssim \sqrt{d}M^{-\frac12}.\numberthis\label{eq:omegaomegasigmabound}
		\end{align*}
		Following the proof of Step 2 of~\cite[Lemma B.2]{zhu2023online}, we have 
	\begin{equation}\label{eqn:decomp(i)(ii)}
		\begin{aligned}
		&\quad \EE[\textstyle\norm{(\sum_{i=1}^n l_i)^{-1}\sum_{i=1}^n \hat \omega_i \hat \omega_i^\top - \textstyle(\sum_{i=1}^n l_i)^{-1}\sum_{i=1}^n  \omega_i  \omega_i^\top}_2 \indic{n}]\\
		&\le   2\cdot \underbrace{\EE[\textstyle\norm{(\sum_{i=1}^n l_i)^{-1}\sum_{i=1}^n H^{-1} U^\top\left( \sum_{k=t_i}^{i} \nu_{k+1}^{(1)}\right)\left( \sum_{k=t_i}^{i} \nu_{k+1}^{(2)}(x_k)\right)^\top U H^{-\top}}_2 \indic{n}]}_{(i)} \\
		&\quad + \underbrace{\EE[\textstyle\norm{(\sum_{i=1}^n l_i)^{-1}\sum_{i=1}^n H^{-1} U^\top\left( \sum_{k=t_i}^{i} \nu_{k+1}^{(2)}(x_k)\right)\left( \sum_{k=t_i}^{i} \nu_{k+1}^{(2)}(x_k)\right)^\top U H^{-\top}}_2\indic{n}]}_{(ii)}
		\end{aligned} 
		\end{equation}	
		By Cauchy-Schwarz inequality, 
		\begin{align*}
		(i) \le \sqrt{\EE[\textstyle\norm{(\sum_{i=1}^n l_i)^{-1}\sum_{i=1}^n \hat \omega_i \hat \omega_i^\top}_2 ]} \cdot \sqrt{(ii)}. \numberthis \label{eqn:(i)(ii)Cauchy}
		\end{align*}
		By~\eqref{eq:omegaomegasigmabound}, we have 
		$$
		\EE[\textstyle\norm{(\sum_{i=1}^n l_i)^{-1}\sum_{i=1}^n \hat \omega_i \hat \omega_i^\top}_2 ] \lesssim 1.
		$$  Therefore, it suffices to bound $(ii)$. 	By triangle inequality and the inequality that $\|C\|_2 \le \trace{C}$,  for any positive semi-definite matrix $C$, 
		\begin{align*}
		(ii) &\le \textstyle (\sum_{i=1}^n l_i)^{-1} \EE\left[\trace{\sum_{i=1}^n H^{-1} U^\top \left( \sum_{k=t_i}^{i} \nu_{k+1}^{(2)}(x_k)\right)\left( \sum_{k=t_i}^{i} \nu_{k+1}^{(2)}(x_k)\right)^\top U H^{-\top}} \indic{n}\right]\\
		&= \textstyle (\sum_{i=1}^n l_i)^{-1}\sum_{i=1}^{n} \EE[\|\sum_{k= t_i}^{i}H^{-1} U^\top \nu_{k+1}^{(2)}(x_k)\|_2^2 \indic{n}] \numberthis \label{eqn:(ii)upper_bound}.
		\end{align*}
		Since $\nu_k^{(2)}$ is a martingale difference sequence and $\tau_{\ks,\delta}$ is a stopping time, for any $i \le n$, we have 
		\begin{align*}
		\textstyle
		& \quad \EE[\|\textstyle\sum_{k= t_i}^{i} H^{-1}U^\top \nu_{k+1}^{(2)}(x_k)\|_2^2 \indic{n}]\\
		&\le \textstyle\EE[\|\sum_{k= t_i}^{i} H^{-1}U^\top \nu_{k+1}^{(2)}(x_k)\|_2^2 \indic{i}]\\
		&= \textstyle\EE[\|\sum_{k= t_i}^{i-1} H^{-1}U^\top \nu_{k+1}^{(2)}(x_k)\|_2^2 \indic{i}] + \EE[\|H^{-1}U^\top \nu_{i+1}^{(2)}(x_i)\|_2^2 \indic{i}]\\
		&\quad \vdots\\
		&\le  \textstyle\sum_{k= t_i}^{i}\EE[\|H^{-1}U^\top \nu_{k+1}^{(2)}(x_k)\|_2^2 \indic{k}].
		\end{align*} 
		When $k \ge k_s$, by Lemma~\ref{lem: second_fourth_moment} and Lemma~\ref{lm:yk4momentconv}, we have $\EE[\|x_k - x^\star\|_2^2 \indic{k}] \lesssim  \ks^\alpha k^{-\alpha}$; on the other hand, when $k < \ks$, by Assumption~\ref{assum:bounded_seq}, we always have $\EE[\|x_k - x^\star\|_2^2 \indic{k}] \le\cub$. Combining, we have
		\begin{align*}
		\textstyle \EE[\|\sum_{k= t_i}^{i} H^{-1}U^\top \nu_{k+1}^{(2)}(x_k)\|_2^2 \indic{n}] \lesssim \begin{cases}
		\sum_{k = t_i}^{i} \ks^{\alpha} k^{-\alpha} & t_i \ge \ks\\
		l_i \cub & t_i < \ks. 
		\end{cases}
		\end{align*}
		By~\eqref{eq:lisumbound}, ~\eqref{eqn:(ii)upper_bound}, and  $\beta > \frac{1}{1-\alpha}$, we have
		\begin{align*}
		\textstyle (ii) &\lesssim  \textstyle (\sum_{i=1}^n l_i)^{-1} \left(\sum_{i=1}^{n} \sum_{k = t_i}^{i} \ks^{\alpha} k^{-\alpha} + \ks^2 \cub \right)\\
		&\lesssim \ks^\alpha M^{-\alpha \beta} + \ks^2 M^{1-2\beta}\\
		&\lesssim \ks^2 M^{-\alpha \beta}. \numberthis \label{eqn:(ii)bound_final}
		\end{align*}
		Combining~\eqref{eq:omegaomegasigmabound},~\eqref{eqn:decomp(i)(ii)},~\eqref{eqn:(i)(ii)Cauchy}, and~\eqref{eqn:(ii)bound_final}, we have
		\begin{align*}
		&\quad \expec{\textstyle\norm{(\sum_{i=1}^n l_i)^{-1}\sum_{i=1}^n  \omega_i  \omega_i^\top-H^{-1}SH^{-\top}}_2 \indic{n}}{}\\
		&\lesssim \sqrt{d}M^{-\frac12} + ((d/M)^{\frac14} + 1) \ks M^{-\alpha \beta/2} + \ks^2 M^{-\alpha \beta}\\
		&\lesssim \sqrt{d}M^{-\frac12}  + \ks^2 M^{-\frac{\alpha \beta}{2}}. \numberthis\label{eq:omegaomegasigmabound_final}
		\end{align*}
		Then by triangle inequality,
		\begin{align*}
		\expec{\textstyle\norm{(\sum_{i=1}^n l_i)^{-1}\sum_{i=1}^n\omega_i\omega_i^\top}_2 \indic{n}}{} \lesssim \sqrt{d}M^{-\frac12} +\ks^2 M^{-\alpha \beta}  +1.\numberthis\label{eq:omegaomegabound}
		\end{align*}
		Combining \eqref{eq:upsupsbound}, and \eqref{eq:omegaomegabound}, and using Cauchy-Schwarz inequality, we have, 
		{\small \begin{align*}
		\hspace{-5mm}\textstyle(\sum_{i=1}^n l_i)^{-1}\sum_{i=1}^n\expec{\norm{\upsilon_i\omega_i^\top}_2\indic{n}}{}\lesssim \sqrt{d}\ks^{\frac{\alpha}{2}} M^{\frac{(\alpha-1)\beta+1}{2}} (d^{\frac{1}{4}} M^{-\frac14} +\ks M^{-\frac{\alpha \beta}{2}}  +1).\numberthis\label{eq:upomegabound}
		\end{align*}
        }
		Similarly, 
		{\small \begin{align*}
		\hspace{-5mm}\textstyle(\sum_{i=1}^n l_i)^{-1}\sum_{i=1}^n\expec{\norm{\omega_i\upsilon_i^\top}_2\indic{n}}{}\lesssim \sqrt{d}\ks^{\frac{\alpha}{2}} M^{\frac{(\alpha-1)\beta+1}{2}} (d^{\frac{1}{4}} M^{-\frac14} +\ks M^{-\frac{\alpha \beta}{2}}  +1).\numberthis\label{eq:omegaupbound}
		\end{align*}
        }
		Then, combining \eqref{eq:upsupsbound}, \eqref{eq:omegaomegasigmabound_final}, \eqref{eq:upomegabound}, and \eqref{eq:omegaupbound}, we have
		\begin{align*}
		&\quad \expec{ \norm{(\textstyle\sum_{i=1}^n l_i)^{-1}\sum_{i=1}^n\lambda_{t_i}^i{\lambda_{t_i}^i}^\top-H^{-1}SH^{-\top}}_2\mathbbm{1}_{\tau_{k_s}>n}}{}\\
		&\lesssim  d \ks^{\alpha} M^{(\alpha-1)\beta+1}+\sqrt{d} M^{-\frac12} +\ks^2 M^{-\frac{\alpha \beta}{2}}+\sqrt{d} \ks^{\frac{\alpha}{2}} M^{\frac{(\alpha-1)\beta + 1}{2}} (d^{\frac{1}{4}} M^{-\frac14} +\ks M^{-\frac{\alpha \beta}{2}}  +1)\\
		&\lesssim  d \ks^{\alpha} M^{(\alpha-1)\beta+1}+\sqrt{d} M^{-\frac12} +\ks^2 M^{-\frac{\alpha \beta}{2}} +\sqrt{d}\ks^{\frac{\alpha}{2}} M^{\frac{(\alpha-1)\beta+1}{2}}.\numberthis\label{eq:lamtilamtibound}
		\end{align*}
               
		\item The fourth term $\expec{(\sum_{i=1}^nl_i)^{-1}n^{-2}\sum_{i=1}^nl_i^2\norm{\lambda_1^n{\lambda_1^n}^\top}_2\indic{n}}{}.$  We have
		\begin{align*}
		&\quad \expec{\norm{\lambda_{1}^n{\lambda_{1}^n}^\top}_2\indic{n}}{}\\
		&\leq  \expec{\norm{\lambda_{1}^n}_2^2 \indic{n}}{}\\
		&= \textstyle\expec{\norm{S_{0}^n\Delta_{0}+\sum_{k=1}^n(\Id+S_k^n)\eta_k(U^\top (\perturb_{k+1}^{(1)}+\perturb_{k+1}^{(2)}(x_k)))
			}_2^2 \indic{n}}{}\\
		&\le   \textstyle\expec{\norm{ S_{0}^n\Delta_{0}}_2^2}{}+\textstyle \sum_{k=1}^n\norm{\Id+S_k^n}_2^2\eta_k^2 \expec{\norm{U^\top (\perturb_{k+1}^{(1)}+\perturb_{k+1}^{(2)}(x_k))
			}_2^2 \indic{k}}{} \numberthis \label{eqn:martingale_argument}\\
		&\lesssim  \textstyle 1+\textstyle \sum_{k=1}^n \expec{\norm{U^\top (\perturb_{k+1}^{(1)}+\perturb_{k+1}^{(2)}(x_k))
			}_2^2}{} \numberthis \label{eqn:applySnbound}\\
		&\lesssim n \numberthis \label{eqn:lambda1n_final}.
		\end{align*}
		where the estimate~\eqref{eqn:martingale_argument} follows from the martingale difference property of $\nu_k^{(1)}$ and $\nu_k^{(2)}$, the estimate~\eqref{eqn:applySnbound} follows from Lemma~\ref{lm:Sijbound}, and the estimate~\eqref{eqn:lambda1n_final} follows from Assumption~\ref{assumption:martinagle} and~\ref{assum:bounded_seq}.
		
		Then, by \eqref{eq:lisumbound} and \eqref{eq:lisqrsumbound}, we have
		{\small \begin{align*}
        \textstyle\expec{(\sum_{i=1}^nl_i)^{-1}n^{-2}\sum_{i=1}^nl_i^2\norm{\lambda_1^n{\lambda_1^n}^\top}_2\indic{n}}{} &\le \textstyle(\sum_{i=1}^nl_i)^{-1}n^{-2}\sum_{i=1}^nl_i^2\expec{\norm{\lambda_1^n{\lambda_1^n}^\top}_2\indic{n}}{}\\
        \textstyle &\lesssim n^{-1}(\sum_{i=1}^nl_i)^{-1}\sum_{i=1}^nl_i^2\lesssim M^{-1}.\numberthis\label{eq:lam1nlam1nbound}
		\end{align*}
        }
        \item The second term
		 $\EE[(\sum_{i=1}^n l_i)^{-1} n^{-1}\|\sum_{i=1}^{n}l_i\lambda_{t_i}^i {\lambda_{1}^n}^\top \|_2 \indic{n}]$. Note that		
         \begin{align*}
		&\quad \textstyle \EE[(\sum_{i=1}^nl_i)^{-1} n^{-1} \norm{\sum_{i=1}^nl_i \lambda_{t_i}^i{\lambda_1^n}^\top}_2\indic{n}]\\
		&\le \textstyle \EE\left[\sqrt{(\sum_{i=1}^nl_i)^{-1} \norm{\sum_{i=1}^n \lambda_{t_i}^i{\lambda_{t_i}^i}^\top}_2\indic{n}} \cdot  \sqrt{(\sum_{i=1}^nl_i)^{-1} n^{-2}\norm{\sum_{i=1}^nl_i^2\lambda_1^n{\lambda_1^n}^\top}_2\indic{n}} \right]\\
		&\le \textstyle \sqrt{(\sum_{i=1}^nl_i)^{-1} \EE[\norm{\sum_{i=1}^n \lambda_{t_i}^i{\lambda_{t_i}^i}^\top}_2\indic{n}]} \cdot  \sqrt{(\sum_{i=1}^nl_i)^{-1} n^{-2}\EE[\norm{\sum_{i=1}^nl_i^2\lambda_1^n{\lambda_1^n}^\top}_2\indic{n}]} \\
		&\lesssim \sqrt{d\ks^{\alpha} M^{(\alpha-1)\beta + 1} + \sqrt{d}M^{-\frac{1}{2}} + \ks M^{-\alpha \beta/2}+\sqrt{d} M^{\frac{(\alpha-1)\beta}{2}+\frac12} + 1} \cdot  M^{-\frac12} \\
		&\lesssim d^{\frac12}\ks^{\frac{\alpha}{2}} M^{\frac{(\alpha-1)\beta}{2}} + d^{\frac{1}{4}}M^{-\frac34} +\ks^{\frac12} M^{-\frac{\alpha \beta}{2} - \frac{1}{2}} +  d^{\frac{1}{4}}\ks^\frac{\alpha}{4} M^{\frac{(\alpha-1)\beta -1}{4}} +  M^{-\frac{1}{2}} \numberthis\label{eq:lamtiilam1nbound},
		\end{align*}
		where the first inequality follows from Cauchy-Schwarz, the second inequality follows from Holder's inequality, and the third inequality follows from 	 \eqref{eq:lamtilamtibound} and \eqref{eq:lam1nlam1nbound}. 
		\item The third term $ \EE[(\sum_{i=1}^nl_i)^{-1} n^{-1} \norm{\sum_{i=1}^nl_i {\lambda_1^n} {\lambda_{t_i}^i}^\top}_2\indic{n}]$. By the same calculation as the second term, we have
		\begin{align*}
		&\textstyle \quad \EE[(\sum_{i=1}^nl_i)^{-1} n^{-1} \norm{\sum_{i=1}^nl_i {\lambda_1^n} {\lambda_{t_i}^i}^\top}_2\indic{n}] \\ 
		& \textstyle \lesssim d^{\frac12}\ks^{\frac{\alpha}{2}} M^{\frac{(\alpha-1)\beta}{2}} + d^{\frac{1}{4}}M^{-\frac34} +\ks^{\frac12} M^{-\frac{\alpha \beta}{2} - \frac{1}{2}} +  d^{\frac{1}{4}}\ks^\frac{\alpha}{4}  M^{\frac{(\alpha-1)\beta}{4} -\frac{1}{4}} +  M^{-\frac{1}{2}} \numberthis \label{eq:lamtiilam1nbound2}
		\end{align*}
		Combining \eqref{eq:lamtilamtibound}, \eqref{eq:lam1nlam1nbound}, \eqref{eq:lamtiilam1nbound}, and \eqref{eq:lamtiilam1nbound2}, we have,                                     
		\begin{align*}
		&\textstyle \quad \expec{\norm{(\textstyle\sum_{i=1}^n l_i)^{-1}\textsf{I}-H^{-1}SH^{-\top}}_2 \indic{n}}{}\\
		&\lesssim  d \ks^{\alpha} M^{(\alpha-1)\beta+1}+\sqrt{d}\ks^2 M^{-\frac12}  +\sqrt{d}\ks^\frac{\alpha}{2} M^{\frac{(\alpha-1)\beta + 1}{2}}.
		\numberthis\label{eq:termIbound}
		\end{align*}
		  \end{itemize}
		
	Next, we bound term $\textsf{IV}$. We then bound terms $\textsf{II}$, and $\textsf{III}$ using Cauchy-Schwarz inequality and the bounds on term $\textsf{I}$, and term $\textsf{IV}$. \newline
\noindent \textbf{Bound on term $\textsf{IV}$:}
		Note that
		\begin{align*}
		\expec{\norm{\textsf{IV}}_2\indic{n}}{}\leq &\textstyle\sum_{i=1}^n\expec{\norm{e_{t_i}^n-n^{-1}l_ie_1^n}_2^2\indic{n}}{}\\
		\leq& 2\textstyle\sum_{i=1}^n(\expec{\norm{e_{t_i}^n}_2^2\indic{n}}{}+\expec{n^{-2}l_i^2\norm{e_1^n}_2^2\indic{n}}{}).\numberthis\label{eq:termI2component}
		\end{align*}
		First, we bound the first term in the RHS of \eqref{eq:termI2component}. 
		Note that for any $j>i\geq k_s$, using Lemma~\ref{lem:sizeofE_k}, we have
		\begin{align*}
		\textstyle\sum_{k=i}^j\expec{\norm{E_k}_2^2\indic{n}}{} &\leq  \textstyle\sum_{k=i}^j\expec{\norm{E_k}_2^2\indic{k}}{}\\
		& \lesssim \ks^{2\alpha}\textstyle\sum_{k=i}^j \eta_k^2. \numberthis\label{eq:Eksumbound}
		\end{align*}
		We also have, by Claim 2 in the proof of Lemma D.5 of \cite{davis2024asymptotic} and Lemma~\ref{lm:yk4momentconv},
		\begin{align*}
		\expec{\norm{R(y_k)}_2^2\indic{n}}{}\leq \expec{\norm{R(y_k)}_2^2\indic{k}}{}\leq \expec{\norm{y_k-x^\star}_2^4\indic{k}}{}\lesssim \ks^{2\alpha}\eta_k^2.\numberthis\label{eq:Ryksumbound}
		\end{align*}
		Direct calculation shows
		\begin{align*}
		&\textstyle\quad \expec{\norm{\sum_{k=i}^j(\Id+S_k^j)\eta_k \zeta_{k+1}}_2^2\indic{n}}{}\\
		&= \textstyle\expec{\norm{(\Id+S_j^j)\eta_j\zeta_{j+1}+\sum_{k=i}^{j-1}(\Id+S_k^j)\eta_k \zeta_{k+1}}_2^2\indic{j}}{}\\
		&=\textstyle\expec{\norm{(\Id+S_j^j)\eta_j\zeta_{j+1}}_2^2\indic{j}}{}+\expec{\norm{\sum_{k=i}^{j-1}(\Id+S_k^j)\eta_k \zeta_{k+1}}_2^2\indic{j}}{}\\
        &\textstyle\quad +2\expec{(\Id+S_j^j)\eta_j\zeta_{j+1}^\top \sum_{k=i}^{j-1}(\Id+S_k^j)\eta_{k+1} \zeta_{k+1}\indic{j}}{}\\
		&=\textstyle\expec{\norm{(\Id+S_j^j)\eta_j\zeta_{j+1}}_2^2\indic{j}}{}+\expec{\norm{\sum_{k=i}^{j-1}(\Id+S_k^j)\eta_k \zeta_{k+1}}_2^2\indic{j}}{}\\
		&\quad \vdots \\
		&=\textstyle\sum_{k=i}^j\expec{\norm{(\Id+S_k^j)\eta_k \zeta_{k+1}}_2^2\indic{k}}{}\\
		&\lesssim \textstyle\sum_{k=i}^j\expec{\norm{\zeta_{k+1}}_2^2\indic{k}}{} \numberthis\label{eqn:bound_skj}\\ 
        &\lesssim \textstyle\sum_{k=i}^j\expec{\norm{y_k-x^\star}_2^2\indic{k}}{}\numberthis\label{eqn:zetak_def}\\
		&\lesssim \ks^{\alpha}\textstyle\sum_{k=i}^j\eta_k \numberthis\label{eqn:yk_dist_bound}\\
		&\le (j-i+1) \ks^{\alpha} i^{-\alpha}, \numberthis\label{eq:zetaksumbound}
		\end{align*}
        where the first several equalities follows from the fact that $\{\zeta_k\}_k$ is a martingale-difference sequence, and we have $\expec{\zeta_i^\top\zeta_{j+1}}{}=0$ for $i\neq j$, the estimate~\eqref{eqn:bound_skj} follows from Lemma~\ref{lm:Sijbound}, the estimate~\eqref{eqn:zetak_def} follows from the definition of $\zeta_k$ and Lipschitz continuity of $P_{T_{\cM}}(\cdot)$, and the estimate~\eqref{eqn:yk_dist_bound} follows from Lemma~\ref{lm:yk4momentconv}. 
		Combining \eqref{eq:Eksumbound}, \eqref{eq:Ryksumbound}, and \eqref{eq:zetaksumbound}, and using Lemma~\ref{lm:Sijbound}, for $i$ such that $t_i \ge \ks$, we have 
		\begin{align*}
		\expec{\norm{e_{t_i}^i}_2^2\indic{n}}{}\lesssim  & l_i\textstyle\sum_{k=t_i}^i\expec{\norm{(\Id+S_k^i)\eta_k\left(R(y_k) - U^\top E_k\right)}_2^2\indic{n}}{}+ \ks^{\alpha}\textstyle\sum_{k=t_i}^i\eta_k\\
		\lesssim & l_i\textstyle\sum_{k=t_i}^i(\expec{\norm{R(y_k)}_2^2\indic{n}}{}+\expec{\norm{U^\top E_k}_2^2\indic{n}}{})+l_i\ks^{\alpha}t_i^{-\alpha}\\
		\lesssim & l_i \ks^{2\alpha}\textstyle\sum_{k=t_i}^i \eta_k^2+l_i \ks^{\alpha} t_i^{-\alpha}\\
		\leq & l_i^2 \ks^{2\alpha} t_i^{-2\alpha}+l_i \ks^{\alpha}t_i^{-\alpha},
		\end{align*}
        where the third inequality follows from $\|R(y_k)\|_2 \lesssim \|y_k - x^\star\|_2^2$, Lemma~\ref{lm:yk4momentconv}, and Lemma~\ref{lem:sizeofE_k}.
		On the other hand, for $i$ such that $t_i < \ks$, we have
		\begin{align*}
		\textstyle	\expec{\norm{e_{t_i}^i}_2^2\indic{n}}{} &\textstyle \le l_i\sum_{k=t_i}^i\expec{\norm{(\Id+S_k^i)\eta_k\left(R(y_k) - U^\top E_k + \zeta_{k+1}\right)}_2^2\indic{n}}{}\\
		&\textstyle \lesssim l_i \sum_{k=t_i}^{i} (\EE[\|R(y_k)\|_2^2] + \EE[\|E_k\|_2^2] + \EE[\|\zeta_{k+1}\|^2])\\
		& \textstyle \le l_i \sum_{k=t_i}^{i} (1 + k^{2\alpha})\\
		&\textstyle \lesssim \ks^{2\alpha+2},
		\end{align*}
		where the first inequality follows from Lemma~\ref{lm:Sijbound},  the third inequality follows from Lemma~\ref{lem:sizeofE_k}, and the last inequality follows from $l_i \le i\lesssim \ks$.
		As a result, for $\beta>(1-\alpha)^{-1}$,
		\begin{align*}
		\textstyle\sum_{i=1}^n\expec{\norm{e_{t_i}^i}_2^2\indic{n}}{} &\lesssim \textstyle\sum_{m=1}^M\sum_{i=a_m+1}^{a_{m+1}}(l_i^2\ks^{2\alpha} a_m^{-2\alpha}+l_i \ks^{\alpha} a_m^{-\alpha}) + \ks^{2\alpha+3}\\
		&\lesssim  \textstyle\sum_{m=1}^M(n_m^3 \ks^{2\alpha} a_m^{-2\alpha}+n_m^2 \ks^{\alpha} a_m^{-\alpha})+ \ks^{2\alpha+3}\\
		&\lesssim  \ks^{2\alpha} M^{3\beta-2\alpha\beta-2}+ \ks^{\alpha} M^{2\beta-\alpha\beta-1} + \ks^{2\alpha + 3}\\
		&\lesssim  \ks^{2\alpha+3} M^{3\beta-2\alpha\beta-2}.\numberthis\label{eq:eijsumbound}
		\end{align*}
		Combining \eqref{eq:eijsumbound} and \eqref{eq:lisumbound}, we have 
		\begin{align*}
		(\textstyle\sum_{i=1}^n l_i)^{-1}\textstyle\sum_{i=1}^n\expec{\norm{e_{t_i}^i}_2^2\indic{n}}{}\lesssim \ks^{2\alpha+3} M^{(1-2a)\beta-1}.\numberthis\label{eq:termIfirstcompbound}
		\end{align*}
		Next, we look at the second term in the RHS of \eqref{eq:termI2component}. Note that
		\begin{align*}
		\textstyle &\quad \expec{\norm{e_1^n}_2^2\indic{n}}{}\\ &\textstyle = \EE[\|\sum_{k=1}^{n} (I+ S_k^i) \eta_k (R(y_k) - U^\top E_k + \zeta_{k+1})\|_2^2\indic{n}]\\
		&\textstyle \lesssim  n\EE[\sum_{k=1}^{n} \|(I+ S_k^i) \eta_k (R(y_k) - U^\top E_k)\|_2^2\indic{n}] +\EE[\|\sum_{k=1}^{n} (I+ S_k^i) \eta_k \zeta_{k+1}\|_2^2\indic{n}]  \\
		&\textstyle \le  n\sum_{k=1}^{n} \|(I+ S_k^i)\|_2^2 \eta_k^2 \EE[\|R(y_k) - U^\top E_k \|_2^2 \indic{n}]  +\EE[\|\sum_{k=1}^{n} (I+ S_k^i) \eta_k \zeta_{k+1}\|_2^2\indic{n}]  \\
		&\textstyle \lesssim n \left(\sum_{k=1}^{n} \EE[\|R(y_k) - U^\top E_k \|_2^2 \indic{n}] \right) + n\\
		& \textstyle \lesssim n (\sum_{k=1}^{n} \ks^{2\alpha} \eta_k^2 + \ks^{1+ 2\alpha}) + n\\
		&\lesssim n\ks^{1+2\alpha},
		\end{align*}
		where the first inequality follows from Young's inequality and Cauchy-Schwarz, the second inequality follows from Jensen's inequality, the third inequality follows by Lemma~\ref{lm:Sijbound}, and the same calculation as~\eqref{eq:zetaksumbound}, and the fourth inequality follows from Lemma~\ref{lm:Sijbound}, Lemma~\ref{lm:yk4momentconv}, and Lemma~\ref{lem:sizeofE_k}.   By~\eqref{eq:lisqrsumbound}, we have 
		\begin{align*}
		\textstyle n^{-2}\sum_{i=1}^nl_i^2\expec{\norm{e_1^n}_2^2\indic{n}}{}\lesssim \ks^{2\alpha+1}n^{-1}\sum_{i=1}^nl_i^2\lesssim \ks^{2\alpha + 1} n^{-1}M^{3\beta-2}.
		\end{align*}
		Using the fact $n\asymp M^\beta$, and \eqref{eq:lisumbound}, we get, 
		\begin{align*}
		\textstyle n^{-2}(\sum_{i=1}^nl_i)^{-1}\sum_{i=1}^nl_i^2 \expec{\norm{e_1^n}_2^2\indic{n}}{}\lesssim \ks^{2\alpha + 1} M^{-1}.\numberthis\label{eq:termIsecondcompbound}
		\end{align*}
		Combining \eqref{eq:termIfirstcompbound}, and \eqref{eq:termIsecondcompbound}, we have
		\begin{align*}
		\textstyle \expec{(\sum_{i=1}^nl_i)^{-1}\norm{IV}_2\indic{n}}{}\lesssim \ks^{2\alpha +1} M^{-1}.\numberthis\label{eq:termIVbound}
		\end{align*}
		\textbf{Bound on term $\textsf{II}$:}
		Combining \eqref{eq:termIbound}, and \eqref{eq:termIfirstcompbound}, and using Cauchy-Schwarz inequality, we obtain, 
		\begin{align*}
		&\quad \textstyle \expec{(\sum_{i=1}^nl_i)^{-1} \norm{II}_2\indic{n}}{}\\
		&\leq \textstyle  \left(\expec{(\sum_{i=1}^nl_i)^{-1}\norm{I}_2 \indic{n}}{}\right)^{1/2}\left(\expec{(\sum_{i=1}^nl_i)^{-1}\norm{IV}_2\indic{n}}{}\right)^{1/2}\\
		&\lesssim \ks^{\alpha+\frac12}M^{-\frac12}  .\numberthis\label{eq:termIIbound}
		\end{align*}
		\textbf{Bound on term $\textsf{III}$:} Similar to term $\textsf{II}$ we have, 
		\begin{align*}
		\textstyle	\quad \expec{(\sum_{i=1}^nl_i)^{-1} \norm{III}_2\indic{n}}{}
		\lesssim \ks^{\alpha+\frac12}M^{-\frac12} .\numberthis\label{eq:termIIIbound}
		\end{align*}
		Combining \eqref{eq:termIbound}, \eqref{eq:termIVbound}, \eqref{eq:termIIbound}, and \eqref{eq:termIIIbound},  we have 
		\begin{align*}
		&\textstyle \quad \expec{\norm{\tilde{\Sigma}-H^{-1}SH^{-\top}}_2\indic{n}}{}\\
		&\textstyle \lesssim  d \ks^{\alpha} M^{(\alpha-1)\beta+1}+\sqrt{d}\ks^2 M^{-\frac12}  +\sqrt{d} \ks^\frac{\alpha}{2}M^{\frac{(\alpha-1)\beta + 1}{2}}+ \ks^{\alpha+\frac{1}{2}}M^{-\frac{1}{2}} + \ks^{2\alpha +1} M^{-1}.\numberthis\label{eq:sigmatildesigmaerrorbound}
		\end{align*}
	\end{proof}
	
	\begin{lemma}\label{lm:sigmanhatsigmanprimeconv}
		Let the conditions of Theorem~\ref{th:mainthmcov} be true. Then,
		\begin{align*}
		\textstyle \expec{\norm{\hat{\Sigma}_n-\Sigma_n'}_2\indic{n}}{}\lesssim \sqrt{d}\ks^{\frac32+\frac{\alpha}{2}} M^{\frac{(\alpha-1)\beta}{2}} +d^\frac14 \ks^\frac52M^{-\frac34}+d^\frac14 \ks^{\frac{\alpha}{4}+\frac32} M^{\frac{(\alpha-1)\beta -1}{4}}+ \ks^{\frac{3}{2}}M^{-\frac{1}{2}} + \ks^{3} M^{-1}.
		\end{align*}
	\end{lemma}
	\begin{proof}\label{pf:sigmanhatsigmanprimeconv}
		Define $\rho_k\coloneqq x_k-z_k$. We have the following expansion:
		\begin{align*}
		\hat{\Sigma}_n-{\Sigma}_n'=&\underbrace{\frac{\sum_{i=1}^n\left(\sum_{k=t_i}^i\rho_k-l_i\bar{\rho}_n\right)\left(\sum_{k=t_i}^i\rho_k-l_i\bar{\rho}_n\right)^
				\top}{\sum_{i=1}^nl_i}}_{\textsf{V}}
		+\underbrace{\frac{\sum_{i=1}^n\left(\sum_{k=t_i}^i\rho_k-l_i\bar{\rho}_n\right)\left(\sum_{k=t_i}^iz_k-l_i\bar{z}_n\right)^
				\top}{\sum_{i=1}^nl_i}}_{\textsf{VI}}\\
		&+\underbrace{\frac{\sum_{i=1}^n\left(\sum_{k=t_i}^iz_k-l_i\bar{z}_n\right)\left(\sum_{k=t_i}^i\rho_k-l_i\bar{\rho}_n\right)^
				\top}{\sum_{i=1}^nl_i}}_{\textsf{VII}}. 
		\end{align*}
        In what follows, we bound them separately.
        
		\noindent \textbf{Bound on Term~\textsf{V}:}
		First, we calculate
		\begin{align*}
		&\textstyle\expec{\norm{\textstyle\sum_{i=1}^n\left(\sum_{k=t_i}^i\rho_k-l_i\bar{\rho}_n\right)\left(\sum_{k=t_i}^i\rho_k-l_i\bar{\rho}_n\right)^
				\top}_2\indic{n}}{}\\
		\leq &\textstyle\sum_{i=1}^n\expec{\norm{\left(\sum_{k=t_i}^i\rho_k-l_i\bar{\rho}_n\right)\left(\sum_{k=t_i}^i\rho_k-l_i\bar{\rho}_n\right)^
				\top}_2\indic{n}}{}\\
		=&\textstyle\sum_{i=1}^n\expec{\norm{\sum_{k=t_i}^i(\rho_k-\bar{\rho}_n)}_2^2\indic{n}}{}\\
		\lesssim &\textstyle\sum_{i=1}^nl_i\sum_{k=t_i}^i(\expec{\norm{\rho_k}_2^2\indic{n}}{}+\expec{\norm{\bar{\rho}_n}_2^2\indic{n}}{}), \numberthis \label{eqn:rhok_decomp}
		\end{align*}
		where the last inequality follows from the Cauchy-Schwarz inequality. Using Equation 10.6 in \cite{davis2024asymptotic}, we have
		\begin{align*}
		\norm{\rho_k}_2\leq \norm{x_k-y_k}_2+\norm{y_k-z_k}_2\lesssim \norm{D_k}_2+\norm{y_k-x^\star}_2^2.\numberthis\label{eq:delkexpansion}
		\end{align*}
		Applying Lemma~\ref{lm:sumDksqrbounmd} and Lemma~\ref{lm:yk4momentconv}, for $i$ such that $t_i \ge \ks$, we have
		\begin{align*}
		\textstyle \sum_{k=t_i}^i\expec{\norm{\rho_k}_2^2\indic{n}}{}&\textstyle \lesssim \sum_{k=t_i}^i(\expec{\norm{D_k}_2^2\indic{k}}{}+\expec{\norm{y_k-x^\star}_2^4\indic{k}}{})\\
		&\textstyle \lesssim \ks^{2\alpha}\sum_{k=t_i}^i\eta_k^2\\
		& \textstyle \lesssim \ks^{2\alpha} l_it_i^{-2\alpha}.\numberthis\label{eq:delksqrsumbound}
		\end{align*}
		On the other hand, for $i$ such that $l_i<t_i < \ks$, we have 
		\begin{align*}
		\textstyle \sum_{k=t_i}^i\expec{\norm{\rho_k}_2^2\indic{n}}{}&\textstyle \lesssim \sum_{k=t_i}^i(\expec{\norm{D_k}_2^2\indic{k}}{}+\expec{\norm{y_k-x^\star}_2^4\indic{k}}{})\\
		&\textstyle \lesssim \ks,\numberthis\label{eq:delksqrsumbound_II},
		\end{align*}
		where the second inequality follows from Assumption~\ref{assum:bounded_seq} and the definition of $y_k$.
		Similar to \eqref{eq:delksqrsumbound} and~\eqref{eq:delksqrsumbound_II}, we have
		\begin{align*}
		\hspace{-5mm}\textstyle \expec{\norm{\bar{\rho}_n}_2^2\indic{n}}{}\leq n^{-1}\sum_{k=1}^n\expec{\norm{\rho_k}_2^2\indic{k}}{}\lesssim n^{-1} (\ks^{2\alpha}\sum_{k=\ks}^n\eta_k^2 + \ks) \lesssim \ks n^{-1}. \numberthis\label{eq:delnbarsqrbound}
		\end{align*}
		Combining~\eqref{eqn:rhok_decomp},  \eqref{eq:delksqrsumbound}, ~\eqref{eq:delksqrsumbound_II}, and \eqref{eq:delnbarsqrbound}, and using the fact that $t_i \asymp i$, we have  
		\begin{align*}
		&\textstyle \quad \expec{\norm{\textstyle\sum_{i=1}^n\left(\sum_{k=t_i}^i\rho_k-l_i\bar{\rho}_n\right)\left(\sum_{k=t_i}^i\rho_k-l_i\bar{\rho}_n\right)^
				\top}_2\indic{n}}{}\\
		&\textstyle\lesssim\sum_{i=1}^nl_i^2(\ks^{2\alpha}t_i^{-2\alpha}+ \ks n^{-1}) + \ks^3\\
		&\textstyle =\sum_{m=1}^M\sum_{i=a_m+1}^{a_{m+1}}l_i^2\ks^{2\alpha} a_m^{-2\alpha} + \ks n^{-1}\sum_{i=1}^{n} l_i^2 + \ks^3\\
		&\textstyle \le \ks^{2\alpha}\sum_{m=1}^{M} a_m^{-2\alpha}n_m^3  + \ks n^{-1}\sum_{i=1}^{n} l_i^2 + \ks^3  \numberthis\label{eq:intermedli2ninvbound}
		\end{align*}
		
		Combining~\eqref{eq:intermedli2ninvbound}, \eqref{eq:lisqrsumbound}, and~\eqref{eq:amnmbound}, and observing $n\asymp M^{\beta}$, we obtain, 
		\begin{align*}
		\expec{\norm{\textstyle\sum_{i=1}^n\left(\sum_{k=t_i}^i\rho_k-l_i\bar{\rho}_n\right)\left(\sum_{k=t_i}^i\rho_k-l_i\bar{\rho}_n\right)^
				\top}_2\indic{n}}{}\lesssim \ks^{2\alpha} M^{3\beta-2\alpha\beta-2}+ \ks M^{2\beta-2} + \ks^3.
		\end{align*}
		Then, by \eqref{eq:lisumbound}, 
		\begin{align*}
		\textstyle	\expec{\norm{V}_2\indic{n}}{} & =(\textstyle\sum_{i=1}^n l_i)^{-1}\expec{\norm{\textstyle\sum_{i=1}^n\left(\sum_{k=t_i}^i\rho_k-l_i\bar{\rho}_n\right)\left(\sum_{k=t_i}^i\rho_k-l_i\bar{\rho}_n\right)^
				\top}_2\indic{n}}{}\\
		&\lesssim \ks^{2\alpha} M^{\beta-2\alpha\beta-1}+ \ks M^{-1} + \ks^3 M^{1-2\beta}\\
		&\lesssim \ks^3 M^{-1}.\numberthis\label{eq:termVbound}
		\end{align*}
		\textbf{Bound on Term~\textsf{VI}:}
		By Lemma~\ref{lm:sigmanhatsigmadiff}, we have, 
		\begin{align*}
		&\quad \textstyle \expec{\norm{{\Sigma}_n'-\Sigma }_2\indic{n}}{}\\
		&\lesssim d \ks^{\alpha} M^{(\alpha-1)\beta+1}+\sqrt{d}\ks^2 M^{-\frac12}  +\sqrt{d} \ks^\frac{\alpha}{2}M^{\frac{(\alpha-1)\beta + 1}{2}}+ \ks^{\alpha+\frac{1}{2}}M^{-\frac{1}{2}} + \ks^{2\alpha +1} M^{-1}.
		\end{align*}
		Then, by Cauchy-Schwarz inequality, we have
		{\small \begin{align*}
		\hspace{-3mm}&\quad \textstyle \expec{\norm{\textsf{VI}}_2\indic{n}}{}\\
		\hspace{-3mm}&\leq \textstyle\left(\expec{\norm{V}_2\indic{n}}{}\right)^{1/2} \left((\sum_{i=1}^{n}l_i)^{-1}\EE[\|\sum_{i=1}^n\left(\sum_{k=t_i}^iz_k-l_i\bar{z}_n\right)\left(\sum_{k=t_i}^iz_k-l_i\bar{z}_n\right)^
		\top\|_2 \indic{n}]\right)^{1/2}\\
		\hspace{-3mm}&\lesssim \textstyle \ks^{\frac{3}{2}}M^{-1/2} \sqrt{ d \ks^{\alpha} M^{(\alpha-1)\beta+1}+\sqrt{d}\ks^2 M^{-\frac12}  +\sqrt{d} \ks^\frac{\alpha}{2}M^{\frac{(\alpha-1)\beta + 1}{2}}+ \ks^{\alpha+\frac{1}{2}}M^{-\frac{1}{2}} + \ks^{2\alpha +1} M^{-1}+ 1} .\numberthis\label{eq:termVIbound}
		\end{align*}
        }
		\textbf{Bound on Term~\textsf{VII}:}
		Similar to Term~\textsf{VI}, we have, 
		\begin{align*}
		&\quad \textstyle \expec{\norm{\textsf{VII}}_2\indic{n}}{}\\
		&\lesssim  \textstyle \ks^{\frac{3}{2}}M^{-\frac{1}{2}} \sqrt{ d \ks^{\alpha} M^{(\alpha-1)\beta+1}+\sqrt{d}\ks^2 M^{-\frac12}  +\sqrt{d} \ks^\frac{\alpha}{2}M^{\frac{(\alpha-1)\beta + 1}{2}}+ \ks^{\alpha+\frac{1}{2}}M^{-\frac{1}{2}} + \ks^{2\alpha +1} M^{-1} + 1}.\numberthis\label{eq:termVIIbound}
		\end{align*}
		Combining \eqref{eq:termVbound}, \eqref{eq:termVIbound}, and \eqref{eq:termVIIbound}, we obtain,
		\begin{align*}
		&\quad \textstyle \expec{\norm{\hat{\Sigma}_n-\Sigma_n'}_2\indic{n}}{}\\
		&\lesssim \sqrt{d}\ks^{\frac32+\frac{\alpha}{2}} M^{\frac{(\alpha-1)\beta}{2}} +d^\frac14 \ks^\frac52M^{-\frac34}+d^\frac14 \ks^{\frac{\alpha}{4}+\frac32} M^{\frac{(\alpha-1)\beta -1}{4}}+ \ks^{\frac{3}{2}}M^{-\frac{1}{2}} + \ks^{3} M^{-1}.
		\end{align*}
	\end{proof}

	\section{Proof of Proposition~\ref{prop:highprob}}\label{sec:proof_high_prob}
	The basic probabilistic tool we use to achieve high probability bound was originally developed by Harvey et al.~\cite{harvey2019tight} and then generalized by ~\cite{cutler2023stochastic}.
	\begin{proposition}[Proposition 29 in \cite{cutler2023stochastic}]\label{prop: stochastic process}
		Consider scalar stochastic processes $(V_k)$, $(D_k)$, and $(X_k)$ on a probability space with Filtration $(\cH_k)$ such that $V_k$ is nonnegative and $\cH_k$ measurable and the inequality 
		\begin{align*}
		V_{k+1} \le \alpha_k V_k + D_k \sqrt{V_k} + X_k + \kappa_k 
		\end{align*}
		holds for for some deterministic constants $\alpha_k \in (-\infty, 1]$ and $\kappa_k \in \RR$. Suppose that the moment generating functions of $D_k$ and $X_k$ conditioned on $\cH_k$ satisfy the following inequalities for some deterministic constants $\sigma_k, \nu_k > 0$:
		\begin{itemize}
			\item $\EE[\exp(\lambda D_k) \mid \cH_k] \le \exp(\lambda^2 \sigma_k^2/2)$ for all $\lambda \ge 0$. (e.g., $D_k$ is mean-zero sub-Gaussian
			conditioned on $\cH_k$ with parameter $\sigma_k$).
			
			\item $\EE[\exp(\lambda X_k) \mid \cH_k] \le \exp(\lambda \nu_k)$ for all $0\le \lambda \le \frac{1}{\nu_k}$. (e.g.,  $X_k$ is nonnegative and subexponential conditioned on $\cH_k$ with parameter $\nu_k$).
			
		\end{itemize}
		Then, the inequality 
		\begin{align*}
		\EE[\exp(\lambda V_{k+1})] \le \exp(\lambda(\nu_k +\kappa_k)) \EE \left[\exp\left(\lambda \left(\frac{1+\alpha_k}{2} V_k\right)\right)\right]
		\end{align*}
		holds for all $0\le \lambda \le  \min \left\{ \frac{1-\alpha_k}{2\sigma_k^2}, \frac{1}{2\nu_k} \right\}$.
	\end{proposition}
	Now we prove Proposition~\ref{prop:highprob}.
	Recall that we let $v_k = G_{\eta_{k+1}}(x_k, \nu_{k+1})$. We have 
	\begin{align*}
	\|x_{k+1} - x^\star\|^2 &=\norm{x_k - \stepsize_{k+1} v_{k}-x^\star}^2 \notag\\
	&= \norm{x_k - x^\star}^2 - 2 \stepsize_{k+1}\dotp{v_k, x_k - x^\star} +  \stepsize_{k+1}^2 \|v_k\|^2 \notag\\
	&\leq \|x_k - x^\star\|^2  - 2\gamma\stepsize_{k+1} \|x_k - x^\star\|^2 + 2C\stepsize_{k+1}^2 (1 + \|x_k - x^\star\|^2 + \|\nu_{k+1}\|^2) \notag \\
	&\hspace{20pt}- 2\stepsize_{k+1}\dotp{\perturb_{k+1}, x_k - x^\star}+ C\eta_{k+1}^2 (1 + \|x_k - x^\star\|^2+  \|\nu_{k+1}\|^2)\\
	&\le  (1 - \gamma \stepsize_{k+1}) \|x_k - x^\star\|^2 - 2 \eta_{k+1} \dotp{\nu_{k+1}, x_k - x^\star}\\
	&\hspace{20pt} + 3C\eta_{k+1}^2 \|\nu_{k+1}\|^2 + 3C \eta_{k+1}^2,
	\end{align*}
	where the first inequality follows from Assumption~\ref{assum: aimingtosol} and the second inequality follows from the upper bound on $\eta$. 
    Define $$\psi_k = \begin{cases}
	\frac{x_k - x^\star}{\|x_k - x^\star\|} & x_k \neq x^\star\\
	0 & \text{otherwise}
	\end{cases}.$$ 
    Note that $2\eta_{k+1} \dotp{\nu_{k+1}, \psi_k}$ is mean-zero sub-Gaussian conditioned on $\cF_k$ with parameter $\eta_{k+1} \sigma$, and $3C\eta_{k+1}^2\|\nu_{k+1}\|^2$ is sub-exponential with parameter $3cC\eta_{k+1}^2 \sigma^2$. We can apply Proposition~\ref{prop: stochastic process} with $$V_k = \|x_k - x^\star\|, \quad \alpha_k = 1-\gamma \eta_{k+1}, \quad D_k = -2 \eta_{k+1} \dotp{\nu_{k+1}, \psi_k}$$
    and $$X_k = 3C\eta_{k+1}^2 \|\nu_{k+1}\|^2,  \quad \kappa_{k} = 3C \eta_{k+1}^2.$$ Recalling $\tilde C= 3cC \sigma^2 +3C$ we have from  Proposition~\ref{prop: stochastic process} that
	\begin{align}\label{eqn: onestephighprob}
	\EE[\exp(\lambda \|x_{k+1} -x^\star\|^2)]&\le \exp(\lambda \tilde C \eta_{k+1}^2) \EE[\exp(\lambda (1-\gamma \eta_{k+1}/2) \|x_{k} - x^\star\|)]
	\end{align}
	for all $0 \le \lambda \le \min \left\{ \frac{\gamma}{2\eta_{k+1} \sigma^2}, \frac{1}{6C \eta_{k+1}^2 \sigma^2}\right\} = \frac{\gamma}{2\eta_{k+1} \sigma^2}$. Define 
	$$
	p_i^j := \begin{cases}
	\prod_{k=i}^{j} \left(1- \frac{\gamma\eta_i}{2}\right) & i \le j\\
	1 & i = j + 1.
	\end{cases} 
	$$
	Applying~\eqref{eqn: onestephighprob} recursively, we deduce
	\begin{align}\label{eqn: highprobMGF}
	\EE[\exp(\lambda \|x_{k} - x^\star\|)] &\le \exp\left(\lambda p_1^k \|x_0 -x^\star\|^2 + \lambda \tilde C \left(\textstyle\sum_{i=1}^{k} p_{i+1}^k\eta_i^2\right)\right) 
	\end{align}
	
	Recall that $C_\alpha = \frac{1-(1/2)^{1-\alpha}}{2(1-\alpha)}$. By Lemma~\ref{lem: upperboundpij}, for $1 \le i \le \floor{k/2}$, 
	\begin{align}\label{eqn: upperboundpij}
	p_{i}^k \le  \exp(- C_\alpha \gamma \eta (k+1)^{1-\alpha}).  
	\end{align}
	
	Consequently, we have 
	\begin{align*}
	\lambda \tilde C \left(\textstyle\sum_{i=1}^{k} p_{i+1}^k\eta_i^2\right) &= \lambda \tilde C \left(\textstyle\sum_{i=1}^{\floor{k/2}} p_{i+1}^k\eta_i^2 + \sum_{i=\floor{k/2}+1}^{k} p_{i+1}^k\eta_i^2\right)\\
	&\le \lambda \tilde C \left( \exp\left( -C_\alpha \gamma \eta (k+1)^{1-\alpha}\right)  \textstyle\sum_{i=1}^{\infty}\eta_i^2  + \textstyle\sum_{i=\floor{k/2}+1}^{k} \eta_i^2 \right)\\
	&\le \lambda \tilde C \eta^2 \underbrace{\left(\left(1+ \frac{1}{2\alpha-1}\right)\exp\left( - C_\alpha \gamma \eta (k+1)^{1-\alpha} \right) + \frac{1}{(2\alpha-1)2^{1-2\alpha} } k^{1-2\alpha} \right)}_{:= H_k},
	\end{align*}
	where the first inequality follows from~\eqref{eqn: upperboundpij} and the fact that $p_{i+1}^k \le 1$, and the second inequality follows from Lemma~\ref{lem: infiniteseries}. 
	By~\eqref{eqn: highprobMGF}, we have
	\begin{align*}
	\EE[\exp(\lambda \|x_k -x^\star\|)] \le \exp\left(\lambda \exp(- C_\alpha\gamma \eta  k^{1-\alpha})\|x_0 - x^\star\|^2 + \lambda\tilde C\eta^2 H_k \right)
	\end{align*}
	By our assumption on $k$, we have
	$$\exp(- C_\alpha\gamma \eta  k^{1-\alpha})\|x_0 - x^\star\|^2  \le \frac{\delta}{4} \qquad \text{and} \quad \tilde C\eta^2 H_k \le \frac{\delta}{4}.$$ 
	Then, by Markov's inequality, we have 
	\begin{align*}
	\pro(\|x_k - x^\star\| \ge \delta) &\le \exp(-\lambda\delta)\EE[\exp(\lambda \|x_k -x^\star\|)] \\
	&\le \exp(-\lambda \delta /2)
	\end{align*}
	Note that by taking $\lambda = \frac{\gamma}{2\eta_{k+1}\sigma^2}$, we have 
	\begin{align}\label{eqn: onestepbound}
	\pro(\|x_k -x^\star\| \ge \delta)\le \exp\left(- \frac{\gamma (k+1)^{\alpha} \delta}{4\eta \sigma^2}\right),
	\end{align}
	which is summable. Combining, we have
	\begin{align*}
	\pro(\|x_i -x^\star\| < \delta, \forall i \ge k) &\ge 1- \textstyle\sum_{i=k}^{\infty} \pro(\|x_i -x^\star\| \ge \delta) \\
	&\ge 1- \textstyle\sum_{i=k}^{\infty} \exp\left(- \frac{\gamma (k+1)^{\alpha} \delta}{4\eta \sigma^2}\right)\\
	&\ge 1- \frac{32 \eta^2 \sigma^4 \exp\left(- \frac{\gamma\delta\sqrt{k}}{4\eta\sigma^2}\right)}{\gamma^2 \delta^2} - \frac{8\eta \delta^2 \sqrt{k} \exp\left(-\frac{\gamma \delta\sqrt{k}}{4\eta\sigma^2}\right)}{\gamma \delta},
	\end{align*}
	where the first inequality follows from the union bound, the second inequality follows from~\eqref{eqn: onestepbound}, and the last inequality follows from Lemma~\ref{lem: tailbound}.
	
	\section{Proofs of Theorem~\ref{thm:cov_light_tail}}\label{sec:proof_cov_lighttail}
	\begin{lemma}\label{lem:cov_bound}
		Let $\hat \Sigma_n$ be defined as in~\eqref{eqn:online_batch_means}. Suppose that Assumption~\ref{assum:bounded_seq} holds. Then we have
		\begin{align*}
		\textstyle \EE[\norm{\hat \Sigma_n}_{op}] \le 4\cub n.
		\end{align*}
	\end{lemma}
	\begin{proof}
		Note that
		\begin{align*}
		\hat{\Sigma}_n &=  \frac{\sum_{i=1}^n\left(\sum_{k=t_i}^ix_k-l_i\bar{x}_n\right)\left(\sum_{k=t_i}^ix_k-l_i\bar{x}_n\right)^
			\top}{\sum_{i=1}^nl_i} \\
		&=  \frac{\sum_{i=1}^n\left(\sum_{k=t_i}^i(x_k- x^\star)-l_i(\bar{x}_n -  x^\star)\right)\left(\sum_{k=t_i}^i(x_k - x^\star)-l_i(\bar{x}_n - x^\star)\right)^
			\top}{\sum_{i=1}^nl_i},
		\end{align*}
		we can without loss of generality assume that $x^\star = 0$ and $\EE[\|x_k\|_2^2] \le \cub$ for all $k \ge 0$.  Note that by Jensen's inequality, $\EE[\|x^\star_n\|_2^2] \le \cub$. We have
		\begin{align*}
		\EE[\|\hat \Sigma_n\|_{op}] &  \le  \frac{\sum_{i=1}^n\EE[\norm{\sum_{k=t_i}^ix_k-l_i\bar{x}_n}_2^2]}{\sum_{i=1}^nl_i} \\
		&\le   \frac{\sum_{i=1}^n l_i \sum_{k=t_i}^i\EE[\norm{x_k-\bar{x}_n}_2^2]}{\sum_{i=1}^nl_i}\\
		&\le  \frac{4\cub \sum_{i=1}^n l_i^2 }{\sum_{i=1}^nl_i}\\
		&\le  4 \cub n,
		\end{align*}
		where the first inequality follows from triangle inequality, the second inequality follows from Jensen's inequality, the third inequality follows from $\EE[\|x^\star_n\|_2^2] \le \cub$, and the last inequality follows from $l_i \le i \le n$. The conclusion then follows.
	\end{proof}
	Now we prove Theorem~\ref{thm:cov_light_tail}.
	By Proposition~\ref{prop:highprob}, for any $\ks \gtrsim 1$ and $n \ge \ks$, we have
	\begin{align*}
	P(\tau_{\ks,\delta} \le n) \le   \frac{32 \eta^2 \sigma^4 \exp\left(- \frac{\gamma\delta\sqrt{\ks}}{4\eta\sigma^2}\right)}{\gamma^2 \delta^2}+ \frac{8\eta \delta \sqrt{\ks} \exp\left(-\frac{\gamma \delta\sqrt{\ks}}{4\eta\sigma^2}\right)}{\gamma }
	\end{align*} 
	For $n\gtrsim 1$, taking $\ks \asymp \log^2n$ so that $P(\tau_{\ks,\delta} \le n)  \lesssim n^{-2}$, we have
	\begin{align*}
	\EE[\|\hat \Sigma_n - \Sigma\|_{op}] &= 	\EE[ \|\hat \Sigma_n - \Sigma\|_{op}\indic{n}] +  \EE[ \|\hat \Sigma_n - \Sigma\|_{op}\mathbbm{1}_{\tau_{\ks,\delta} \le n}]\\
	&\lesssim_{\log} \sqrt{d}M^{-\frac{1}{2}} + \sqrt{d} M^{\frac{(\alpha-1)\beta +2}{2}}  + nP(\tau_{\ks,\delta} \le n) \\
	&\lesssim  \sqrt{d}M^{-\frac{1}{2}} + \sqrt{d} M^{\frac{(\alpha-1)\beta +2}{2}} \\
	&\lesssim   \sqrt{d} n^{-\frac{1}{2\beta}} + \sqrt{d} n^{-\frac{(\alpha -1)\beta +1}{2\beta}},
	\end{align*}
	where the first inequality follows from Theorem~\ref{th:mainthmcov} and Lemma~\ref{lem:cov_bound}, the second inequality follows from $P(\tau_{\ks,\delta} \le n)  \lesssim n^{-2}$, and the last inequality follows from $n\approx M^\beta$.
	\section{Extra assumption verification for stochastic approximation}\label{sec:global_guarantee_lemma}
  The following proposition shows that  under convexity (monotonicity), Assumption~\ref{assum:bounded_seq} holds for all the stochastic approximation algorithms in Section~\ref{sec:SA_algorithms}.   \begin{proposition}\label{prop:bounded_seq_sufficient}
		Suppose that the variational inclusion problem takes the form of~\eqref{eqn:variation_inclusion}, and Assumption~\ref{assumption:zero} and~\ref{assumption:martinagle} holds. 
		Moreover, suppose that $A$ is a Lipschitz and monotone map and we are in one of the following scenarios:
		\begin{enumerate}
			\item One applies the stochastic forward algorithm to the case $f = 0$ and  $g$ is Lipschitz and convex.
			\item One applies the stochastic projected forward algorithm to the case $f$ is the indicator function of a closed convex set $\cX$ and $g$ is Lipschitz and convex.
			\item One applies the stochastic forward-backward algorithm to the case $f$ is Lipschitz in its domain and $g = 0$. 
		\end{enumerate}
		Then Assumption~\ref{assum:bounded_seq} holds.
	\end{proposition}
\begin{proof} Note that the first scenario is a special case of the second one, we only prove it for the second and third cases.
	\paragraph{Stochastic projected forward algorithm.} By the definition of $x^\star$, there exists $v^\star \in \partial g(x^\star)$ and $w^\star \in N_\cX(x^*)$ such that 
	$$
	0 = A(x^\star) + v^\star + w^\star.
	$$
	By monotonicity of $A$ and convexity of $g$, for any $x_k$ and $s_g(x_k) \in \partial g(x_k)$, we have 
	\begin{align}
	\dotp{A(x_k) + s_g(x_k) + w^\star, x_k - x^\star} = \dotp{A(x_k) + s_g(x_k) - A(x^\star) - v^\star, x_k - x^\star} \ge 0.
	\end{align}
	Note also that $w^\star \in N_\cX(x^\star)$, we have
	\begin{align}\label{eqn:A_and_partial_g}
	\dotp{A(x_k) + s_g(x_k), x_k - x^\star} \ge  - \dotp{w^\star, x_k - x^\star} \ge 0.
	\end{align}
	As a result, there exists some constant $C>0$ such that
	\begin{align*}
	\EE[\|x_{k+1} - x^\star\|_2^2] &= \EE[\|P_\cX(x_k - \eta_{k+1}(A(x_k) + s_g(x_k) + \nu_{k+1})) - x^\star\|_2^2]\\
	&\le \EE[\|x_k - \eta_{k+1}(A(x_k) + s_g(x_k) + \nu_{k+1}) - x^\star\|_2^2]\\
	&\le  \EE[\|x_k - x^\star\|^2]  -2 \eta_{k+1} \EE[\dotp{A(x_k) + s_g(x_k) + \nu_{k+1}, x_k - x^*}] + C\eta_{k+1}^2 (1+ \EE[\|x_k - x^\star\|_2^2])\\
	&\le (1+ C \eta_{k+1}^2) \EE[\|x_k - x^\star\|^2]  + C\eta_{k+1}^2,
	\end{align*}
	where the first inequality follows from the fact that $P_\cX$ is 1-Lipschitz, and the last inequality follows from~\eqref{eqn:A_and_partial_g}. The results then follow from Lemma~\ref{lem:bounded_seq_lemma}.
	
	\paragraph{Stochastic forward-backward algorithm.} By definition of $x^\star$, there exists $w \in \partial f(x^\star)$ such that 
	$$
	0 = A(x^\star) + w^\star.
	$$
	For any $x_k$, we denote $x_k - \eta_{k+1}(A(x_k) + \nu_{k+1})$ by $x_k^+$ and $\frac{x_k^+ - \prox_{\eta_{k+1} f}(x_k^+)}{\eta_{k+1}}$ by $w_k^+$. By the property of the proximal operator, we have $w_k^+ \in \partial f(x_k^+)$. Moreover, by monotonicity of $A$ and convexity of $f$,  we have
	\begin{align}\label{eqn:proximal_grad_aiming}
	\dotp{A(x_k^+) + w_k^+, x_k^+ - x^*}= \dotp{ A(x_k^+) + w_k^+ - A(x^\star) - w^\star, x_k^+ - x^*}\ge 0.
	\end{align}
	Next, we bound $\|x_{k+1} - x_k\|$. By definition of $x_{k+1}$ and Lipschitz property of $f$ and $A$, there exists some constant $C >0$ (may change from line to line) such that
	\begin{align*}
	\frac{1}{2\eta_{k+1}}\| x_{k+1} - x_k\|_2^2 &\le f(x_k) - f(x_{k+1}) - \dotp{A(x_k) + \nu_{k+1}, x_{k+1} - x_k}\\
	&\le C(1 + \|x_k - x^\star\|_2 + \|\nu_{k+1}\|_2) \|x_{k+1} - x_k\|_2.
	\end{align*}
	As a consequence, there exists a constant $C >0$ such that
	\begin{align}\label{eqn:step_bound}
	\|x_{k+1} - x_k\|_2 \le C\eta_{k+1} (1 + \|x_k - x^\star\|_2 + \|\nu_{k+1}\|_2).
	\end{align}
	In addition, by Lipschitz continuity of $A$ and $f$, there exists some constant $C >0$ ) such that 
	\begin{align}\label{eqn:one_step_bound}
	\|x_k - x_k^+\|_2 \le C \eta_{k+1}(1 + \|x_k - x^\star\|_2 + \|\nu_{k+1}\|_2).
	\end{align}
	As a result of~\eqref{eqn:step_bound} and~\eqref{eqn:one_step_bound}, there exists a constant $C>0$ such that
	\begin{align}\label{eqn:w_k_bound}
	\|w_k^+\|_2 \le \frac{1}{\eta_{k+1}}(\|x_k - x_k^+\|_2 + \|x_{k+1} - x_k\|_2) \le  C (1 + \|x_k - x^\star\|_2 + \|\nu_{k+1}\|_2)
	\end{align}
	Consequently, there exists a constant $C>0$ such that
	\begin{align*}
	\EE[\|x_{k+1} - x^\star\|_2^2] &= \EE[\|x_k - x^\star - \eta_{k+1}(A(x_k) + \nu_{k+1} + w_{k}^+) \|_2^2]\\
	&= \EE[\|x_k - x^\star\|_2^2] - 2\eta_{k+1} \EE[\dotp{x_k - x^\star, A(x_k) + w_{k}^+}] + C\eta_{k+1}^2 (1 + \EE[\|x_k - x^\star\|_2^2]).
	\end{align*}
	Next, we show that $2\eta_{k+1} \EE[\dotp{x_k - x^\star, A(x_k) + w_{k}^+}]$ is lower bound. By~\eqref{eqn:proximal_grad_aiming}, we have
	\begin{align*}
	&\quad 2\eta_{k+1} \EE[\dotp{x_k - x^\star, A(x_k) + w_{k}^+}] \\
	&= 2\eta_{k+1} (\EE[\dotp{x_k - x_k^+, A(x_k) + w_{k}^+}] + \EE[\dotp{x_k^+ - x^\star, A(x_k) - A(x_k^+)}] +  \EE[\dotp{x_k^+ - x^\star, A(x_k^+) + w_{k}^+}])\\
	&\ge  - 2\eta_{k+1} (\underbrace{\EE[\|x_k - x_k^+\|_2 \|A(x_k) + w_{k}^+\|_2]}_{(I)} + \underbrace{\EE[\|x_k^+ - x^\star\|_2 \|A(x_k) - A(x_k^+)\|_2]}_{(II)}).
	\end{align*}
	We bound $(I)$ and $(II)$ separately.  By Holder's inequality,
	\begin{align*}
	(I) & \le (\EE[\|x_k - x^+\|_2^2])^{\frac{1}{2}} (\EE[\|A(x_k) + w_k^+\|_2^2])^{\frac{1}{2}}\\
	&\le C \eta_{k+1} (1 +  \EE[\|x_k - x^\star\|^2]),
	\end{align*}
	where the second inequality follows from~\eqref{eqn:one_step_bound}. On the other hand,
	\begin{align*}
	(II) &\le C\cdot \EE[\|x_k^+ - x^\star\|_2 \|x_k - x_k^+\|_2]\\
	&\le C (\EE[\|x_k - x_k^+\|_2^2] + \EE[\|x_k - x^\star\|_2\|x_k - x_k^+\|_2])\\
	&\le  C \left(\EE[\|x_k - x_k^+\|_2^2] + (\EE[\|x_k - x^\star\|_2^2])^{\frac{1}{2}} (\EE[\|x_k - x_k^+ \|_2^2])^{\frac{1}{2}}\right)\\
	&\le C \eta_{k+1} ( 1+ \EE[\|x_k - x^\star\|_2^2]).
	\end{align*}
	Combining,  We have 
	$$
	2\eta_{k+1} \EE[\dotp{x_k - x^\star, A(x_k) + w_{k}^+}]  \ge -  C \eta_{k+1} ( 1+ \EE[\|x_k - x^\star\|_2^2]).
	$$
	Consequently, there exists constant $C>0$ such that 
	$$
	\EE[\|x_{k+1} - x^\star\|_2^2] \le (1+ C \eta_{k+1}^2) \EE[\|x_k - x^\star\|^2]  + C\eta_{k+1}^2.
	$$
	The results then follow from Lemma~\ref{lem:bounded_seq_lemma}.
\end{proof}

The following proposition shows that under strong convexity (monotonicity), Assumption~\ref{assum: aimingtosol} holds for all the stochastic approximation algorithms in Section~\ref{sec:SA_algorithms}. 
    \begin{proposition}\label{prop:high_prob_assumption}
    	Suppose that the variational inclusion problem takes the form of~\eqref{eqn:variation_inclusion}. Assume that $A$ is strongly monotone and Lipschitz. Suppose we are in one of the following scenarios:
    	\begin{enumerate}
    		\item One applies the stochastic forward algorithm to the case $f = 0$ and  $g$ is Lipschitz and convex.
    		\item One applies the stochastic projected forward algorithm to the case  $f$ is the indicator function of a closed set $\cX$ and $g$ is Lipschitz and convex.
    		\item One applies the stochastic forward-backward algorithm to the case  $f$ is Lipschitz in its domain and $g = 0$. 
    	\end{enumerate}
    	Then Assumption~\ref{assum: aimingtosol} holds.
    \end{proposition}
    \begin{proof}
		Since the stochastic forward algorithm is a special case of the stochastic projected forward algorithm, it suffices to prove the result for both the stochastic projected forward algorithm (case 2) and the stochastic forward-backward algorithm (case 3.)
		\paragraph{Stochastic projected forward algorithm.} Recall  $s_g$ is a  selection of $\partial g$.  There exists some constant $C>0$ (it may change from line to line through the proof) such that
		\begin{align*}
		\|G_\eta(x, \nu)\|_2 &= \left\|\frac{ x - P_\cX( x - \eta(A(x) + s_g(x) + \nu))}{\eta}\right\|_2\\
		&\le  \|A(x) + s_g(x) + \nu\|_2\\
		&\le C(1 + \|x - x^\star\|_2),		
		\end{align*}
		where the first inequality follows from the fact that $P_\cX$ is $1$-Lipschitz and the second inequality follows from the Lipschitz continuity of $A$ and $g$. Item~\ref{item:global_steplength} follows. On the other hand, by the definition of $x^\star$, there exists $v^\star \in \partial g(x^\star)$ and $w^\star \in N_\cX(x^*)$ such that 
		$$
		0 = A(x^\star) + v^\star + w^\star.
		$$
		By strong monotonicity of $A$ and convexity of $g$, there exists $\gamma >0$ such that for any $x$ and $s_g(x) \in \partial g(x)$, we have 
		\begin{align*}
		\dotp{A(x) + s_g(x) + w^\star, x - x^\star} &= \dotp{A(x) + s_g(x) - A(x^\star) - v^\star, x - x^\star}\\
		&\ge \gamma \|x - x^\star\|_2^2.
		\end{align*}
		As a result of $w^\star \in N_\cX(x^*)$, we have 
		\begin{align}\label{eqn:strong_aim_sol}
		\dotp{A(x) + s_g(x), x - x^\star} \ge \gamma \|x_k - x^\star\|_2^2.
		\end{align}
		Next, we denote $x -\eta(A(x) + s_g(x) + \nu)$ by $x^+$ and $\frac{x^+ - P_\cX(x^+)}{\eta}$ by $w$. Note that $w \in N_\cX(x^+)$ and $G_\eta(x,\nu) = w + A(x) + s_g(x) + \nu$, so we have
		\begin{align}\label{eqn:upper_bound_u}
		\|w\|_2 \le C(1 + \|x-x^\star\|_2 + \nu).
		\end{align}
		Therefore,
		\begin{align*}
		\dotp{G_\eta(x,\nu) -\nu, x- x^\star} &= \dotp{w + A(x) + s_g(x), x-x^\star}\\
		&\ge \gamma \|x-x^\star\|^2 + \dotp{w, x-x^+} + \dotp{w, x^+-x^\star}\\
		&\ge  \gamma \|x-x^\star\|^2 + \dotp{w, x-x^+},
		\end{align*}
		where the first inequality follows from~\ref{eqn:strong_aim_sol} and the second inequality follows from $w \in N_\cX(x^+)$. Note also that 
		\begin{align*}
		\|x - x^+\|_2 \le C \eta (1+ \|x-x^\star\|_2 + \|\nu\|_2),
		\end{align*}
		we have 
		$$
		|\dotp{w, x- x^+}| \le \|w\|_2\|x - x_+\|_2 \le C \eta (1 + \|x-x^\star\|_2^2 + \|\nu\|_2^2). 
				$$
		Combining, we have
		\begin{align*}
		\dotp{G_\eta(x,\nu) -\nu, x- x^\star} \ge  \gamma \|x-x^\star\|^2 -  C \eta (1 + \|x-x^\star\|_2^2 + \|\nu\|_2^2).
		\end{align*}
		\paragraph{Stochastic forward-backward algorithm.} 
			First, we bound $\|G_\eta(x,\nu)\|_2$. By definition of proximal operator and Lipschitz property of $f$ and $A$, there exists some constant $C >0$ (may change from line to line) such that
		\begin{align*}
		\frac{\eta}{2}\|G_\eta(x,\nu)\|_2^2 &\le f(x) - f(x - \eta G_\eta(x,\nu)) + \eta \dotp{A(x) + \nu, G_\eta(x,\nu)}\\
		&\le C \eta \|G_\eta(x,\nu)\|_2 (1 + \|x - x^\star\|_2 + \|\nu\|_2).
		\end{align*}
		As a consequence, there exists a constant $C >0$ such that
		\begin{align}\label{eqn:step_bound_strong}
		\|G_\eta(x,\nu)\|_2 \le C (1 + \|x - x^\star\|_2 + \|\nu\|_2).
		\end{align}
		Therefore, item~\ref{item:global_steplength} follows.
		 Next, by the definition of $x^\star$, there exists $w^\star \in \partial f(x^\star)$ such that 
		$$
		0 = A(x^\star) + w^\star.
		$$
		For any $x$, we denote $x - \eta(A(x) + \nu)$ by $x^+$ and $\frac{x^+ - \prox_{\eta f}(x^+)}{\eta}$ by $w^+$. By the property of the proximal operator, we have $w^+ \in \partial f(x^+)$. Moreover, by strong monotonicity of $A$ and convexity of $f$,  we have
		\begin{align*}
		\dotp{A(x^+) + w^+, x^+ - x^\star} &= \dotp{ A(x^+) + w^+ - A(x^\star) - w^\star, x^+ - x^\star}\\
		&\ge \gamma \|x^+ - x^\star\|^2. \numberthis \label{eqn:proximal_grad_aiming_strong}
		\end{align*}
	
		In addition, by Lipschitz continuity of $A$ and $f$, there exists some constant $C >0$ ) such that 
		\begin{align}\label{eqn:one_step_bound_strong}
		\|x - x^+\|_2 \le C \eta(1 + \|x - x^\star\|_2 + \|\nu\|_2).
		\end{align}
		As a result of~\eqref{eqn:step_bound_strong} and~\eqref{eqn:one_step_bound_strong}, there exists constant $C>0$ such that
		\begin{align}\label{eqn:w_k_bound_strong}
		\|w^+\|_2 \le \frac{1}{\eta}\|x - x^+\|_2 + \|G_\eta(x,\nu)\|_2 \le  C (1 + \|x - x^\star\|_2 + \|\nu\|_2)
		\end{align}
		Note that $G_\eta(x,\nu) = w^+ + A(x) + \nu$, we have
		\begin{align*}
		\dotp{G_\eta(x,\nu) - \nu, x - x^\star}  &= \dotp{w^+ + A(x), x - x^\star}\\
		&= \underbrace{\dotp{A(x) - A(x^+), x - x^\star}}_{(I)} + \underbrace{\dotp{w^+ + A(x^+), x^+ - x^\star}}_{(II)} + \underbrace{\dotp{w^+ + A(x^+), x - x^+}}_{(III)}.
		\end{align*}
		We lower-bound each term separately. By Lipschitz continuity of $A$ and~\eqref{eqn:one_step_bound_strong}, we have
		\begin{align*}
		\|(I)\|_2 \le C (1 +\|x - x^\star\|_2^2  + \| \nu\|_2^2). 
		\end{align*}
		By~\eqref{eqn:proximal_grad_aiming_strong}, we have
		\begin{align*}
		(II) &\ge \gamma \|x^+ - x^\star\|^2 \\
		&\ge  \gamma\|x - x^\star\|_2^2 - 2\|x_+ - x\|_2 \|x- x^\star\|_2\\
		&\ge \gamma \|x - x^\star\|_2^2 - C(1+ \|x-x^\star\|_2^2 + \|\nu\|_2^2).
		\end{align*}
		Moreover, 
		\begin{align*}
		\|(III)\|_2 &\le  (\|w^+\|_2 + \|A(x_+)\|_2) \|x - x^+\|_2\\
		&\le  C(1+ \|x-x^\star\|_2^2 + \|\nu\|_2^2),
		\end{align*}
		where the last inequality follows from the Lipschitz continuity of $A$ and ~\eqref{eqn:step_bound_strong}. The results then follows by combining $(I), (II)$, and $(III)$.
\end{proof}

	\section{Technical lemmas}
	Recall that for a given index $k \geq 0$ and a constant $\delta \in (0,1)$, the stopping time is defined as
	\begin{align}\label{def:stoppingtime}
	\tau_{k, \delta} := \inf\{l \geq k \colon x_l  \notin B_{\delta}(x^\star)\},
	\end{align}
	which is the first time after $k$ that the iterate leaves $B_{\delta}(x^\star)$. Now, define $D_k := \dist(x_k, \cM)$, $v_k:= G_{\eta_{k+1}}(x_k,\nu_{k+1})$ for all $k \geq 0$. In what follows, $C$ denotes constant and may change from line to line.
	\begin{lemma}\label{lm:sumDksqrbounmd}
		Suppose that Assumptions~\ref{assumption:localbound},~\ref{assumption:Aproposed}, and~\ref{assumption:zero} hold. If $\alpha \in (1/2,1)$, then for any sufficiently small $\delta>0$, any $\ks \ge 0$, there exists a constant $C$ depending on $\delta$, $k_s$ and $\alpha$ such that for any $l\ge s \ge \ks$, 
		$$
		\textstyle\sum_{k=s}^{l} \EE[D_{k}^2 1_{\tau_{k_s, \delta} > k}] \le C k_s^{2\alpha}\textstyle\sum_{k=s}^{l} \stepsize_k^2.
		$$ 
	\end{lemma}
	\begin{proof}
		First, we note that it suffices to show the result for all $\ks \ge \left(\frac{4\alpha}{\mu \eta}\right)^{1/(1-\alpha)}$ since the cases when $\ks \le \left(\frac{4\alpha}{\mu \eta}\right)^{1/(1-\alpha)}$ can be handled by enlarging $C$ properly.
		Define $A_k := \{\tau_{k_s, \delta} > k\}$ for all $k \ge k_s$. We require that $\delta$ is small enough so that $B_{\delta}(x^\star)$ is contained in the neighborhood where Assumption~\ref{assumption:Aproposed} holds with probability 1. Note that we require $k_s$ (or $\eta$) to be large enough so the conclusions of  Lemma~\ref{lem: supportthirdorder} holds for all $k\ge k_s$. We first prove a recurrence relation satisfied by the sequence $D_k$. To that end, recall the update rule~\eqref{eqn: updaterule}, for all $k \ge 0$, when $x_k \in B_\delta(x^\star)$, we have
		\begin{equation}\label{eqn: stayinball1}
		\begin{aligned}
		D_{k+1}^2 &\le \norm{x_{k+1} - P_{\cM}(x_{k})}^2\\
		&=\norm{x_k - \stepsize_{k+1} v_{k}-P_{\cM}(x_k)}^2 \\
		&= \norm{x_k - P_{\cM}(x_k)}^2 - 2 \stepsize_{k+1}\dotp{v_k, x_k - P_{\cM}(x_{k})} +  \stepsize_{k+1}^2 \|v_k\|^2 \\
		&\leq D_k^2  - 2\stepsize_{k+1} \mu D_k + 2\stepsize_{k+1} (1+\|\nu_{k+1}\|)^2o(D_k)  \\
		&\hspace{20pt}- 2\stepsize_{k+1}\dotp{\perturb_{k+1}, x_{k} - \proj_{\cM}(x_k)}+ \underbrace{C(1+\|\perturb_{k+1}\|)^2}_{:=B_{k+1}}\stepsize_{k+1}^2,
		\end{aligned}
		\end{equation}
		where the second inequality follows from Assumption~\ref{assumption:localbound} and Condition~\ref{assumption:aiming} of Assumption~\ref{assumption:Aproposed}. Note that the bound $\EE_k[\|\nu_{k+1}\|^4]1_{A_k} \leq q(x_k)1_{A_k}$ implies that there exists $C > 0$ such that 
		$$
		\EE_k[B_{k+1}]1_{A_k} \leq C,
		$$
		meaning the conditional expectation is bounded for all $i$. Moreover, by shrinking $\delta$ if necessary, we have 
		$$
		\EE_k[(1+\|\perturb_{k+1}\|)^2o(D_k)1_{A_k}] \leq \frac{\mu}{2}D_k1_{A_k}.
		$$
		Thus, for each $k \ge k_s$, we have
		\begin{align}
		\EE_{k}[D_{k+1}^2 1_{A_{k+1}}] &\leq \EE_k[D_{k+1}^2 1_{A_{k}}] \notag\\
		&\le D_{k}^21_{A_k}  - \mu\stepsize_{k+1} D_{k} 1_{A_k} + C\stepsize_{k+1}^2 \label{eq:decrease}
		\end{align}
		where the first inequality follows from $1_{A_{k+1}}\le 1_{A_k}$, the second inequality the assumption that $\{\nu_k\}$ is a martingale difference sequence and $A_k$ is $\cF_k$ measurable. 
		Taking expectations on both sides, we have 
		\begin{align}\label{eqn: telescope}
		\EE[D_{k+1}^21_{A_{k+1}}]\le  \EE[D_{k}^21_{A_k}] - \mu \eta_{k+1}\EE[D_{k}1_{A_{k}}] + C\stepsize_{k+1}^2. 
		\end{align}
		Summing~\eqref{eqn: telescope} from $k = s$ to $l$ and using Lemma~\ref{lem: second_fourth_moment}, we have
		\begin{align}\label{eqn: firstmomentsum}
		\textstyle\sum_{k=s}^{l} \stepsize_{k+1}\EE[D_{k}1_{A_k}] \lesssim \EE[D_s^21_{A_s}] +  \textstyle\sum_{k=s}^{l} \eta_{k+1}^2\lesssim \ks^\alpha \eta_s+ \textstyle\sum_{k=s}^{l} \eta_{k+1}^2\lesssim k_s^\alpha \textstyle\sum_{k=s}^{l} \eta_{k+1}^2.  
		\end{align}
		
		On the other hand, when $x_k \in B_\delta(x^\star)$, we have

		\begin{align*}
		D_{k+1}^4& \leq \norm{x_k-\eta_{k+1}v_k-P_\cM(x_k)}^4\\
		&= D_k^4-4\eta_{k+1}\inner{v_k, x_k-P_\cM(x_k)}D_k^2+\eta_{k+1}^4\norm{v_k}^4+2\eta_{k+1}^2D_k^2\norm{v_k}_2^2
		+4\eta_{k+1}^2\inner{v_k, x_k-P_\cM(x_k)}^2\\
		&\quad -4\eta_{k+1}^3\inner{ v_k, x_k-P_\cM(x_k)}\|v_k\|^2 \numberthis \label{eqn: equality}\\
		&\leq  D_k^4 - 4 \mu\eta_{k+1} D_k^3 + 4 \eta_{k+1} D_k^2 (1+ \|\nu_{k+1}\|^2) o(D_k) - 4\eta_{k+1} \dotp{\nu_{k+1}, x_k - P_\cM(x_k)}D_k^2 + \eta_{k+1}^4\|v_k\|^4\\
		&\quad + 6\eta_{k+1}^2 D_k^2 \|v_k\|^2 + 4 \eta_{k+1}^3 D_k \|v_k\|^3,\numberthis \label{eqn: inequality}
		\end{align*}
		where the equality~\eqref{eqn: equality} follows from expanding the fourth power directly and the estimate~\eqref{eqn: inequality} follows from Conditioning~\ref{assumption:aiming} of Assumption~\ref{assumption:Aproposed} and Cauchy-Schwarz inequality. Thus, there exists constant $C>0$ such that for each $i\ge 0$, we have 
		\begin{align}\label{eqn: fourthpoweronestep}
		\EE_{k}[D_{k+1}^41_{A_{k+1}}] &\le \EE_{k}[D_{k+1}^41_{A_{k}}] \notag\\
		&\le  D_{k}^41_{A_k} - 2\mu\stepsize_{k+1} D_{k}^31_{A_k}  + C\stepsize_{k+1}^2 D_{k}^21_{A_k}  + C\stepsize_{k+1}^3 D_{k}1_{A_k} +  C\stepsize_{k+1}^4 \notag\\
		&\le (1 - \mu\stepsize_{k+1} )D_{k}^41_{A_k} - \mu\stepsize_{k+1} D_{k}^31_{A_k}  + C\stepsize_{k+1}^2 D_{k}^21_{A_k} +  C\stepsize_{k+1}^3 D_{k}1_{A_k} + C\stepsize_{k+1}^4,
		\end{align}
		where the first inequality follows from $1_{A_{k+1}} \le 1_{A_k}$, the second inequality follows from the assumption that $\{\nu_{k+1}\}$ is a martingale difference sequence, our choice of $\delta$, and the bound on the fourth moment of $\nu_i$, and the third inequality follows from the assumption that $\delta <1$. By Lemma~\ref{lem: supportthirdorder}, for all $k \ge k_s$, we have 
		$$
		\frac{1 - \mu\stepsize_{k+1}}{\stepsize_{k+1}^2} \le \frac{1}{\stepsize_{k}^2}.
		$$
		Taking expectation and dividing both sides of~\eqref{eqn: fourthpoweronestep} by $\eta_{k+1}^2$, we have 
		\begin{align}\label{eqn: telescopecubic}
		\frac{\EE[D_{k+1}^41_{A_{k+1}}]}{\stepsize_{k+1}^2} \le \frac{\EE[D_k^41_{A_k}]}{\stepsize_{k}^2} - 2\frac{\mu}{\stepsize_{k+1}} \EE[D_{k}^31_{A_k}] + C \EE[D_{k}^21_{A_k}]+  C\stepsize_{k+1} \EE[D_{k}1_{A_k}] + C\stepsize_{k+1}^2
		\end{align}
		For any index $l \ge s \ge k_s$, summing~\eqref{eqn: telescopecubic} from $s$ to $l$, we have
		\begin{align}
		\textstyle\sum_{k=s}^{l} \frac{1}{\stepsize_{k+1}}\EE[D_{k}^31_{A_k}] &\le  \frac{\EE[D_s^41_{A_s}]}{\stepsize_{s}^2} + C\left(\textstyle\sum_{k=s}^{l} \EE[D_{k}^21_{A_k}] +  \textstyle\sum_{k=s}^{l}\stepsize_{k+1} \EE[D_{k}1_{A_k}]+ \textstyle\sum_{k=s}^{l}\stepsize_{k+1}^2\right) \notag\\
		&\lesssim \textstyle\sum_{k=s}^{l} \EE[D_{k}^21_{A_k}] + \ks^{3\alpha}\textstyle\sum_{k=s}^{l} \eta_{k+1}^2, \label{eqn: thridmomentumsum} 
		\end{align}
		
		where the second inequality follows from Lemma~\ref{lem: second_fourth_moment} and the estimate~\eqref{eqn: firstmomentsum}.  Combining~\eqref{eqn: firstmomentsum} and~\eqref{eqn: thridmomentumsum}, we have
		\begin{align*}
		\textstyle\sum_{k=s}^{l} \EE[D_{k}^21_{A_k}] & \le  \sqrt{ \textstyle\sum_{k=s}^{l}\stepsize_{k+1}\EE[D_k1_{A_k}] \cdot \textstyle\sum_{k= s}^{l} \frac{1}{\stepsize_{k+1}}\EE[D_{k}^31_{A_k}]} \\
		&\lesssim  \sqrt{ k_s^{\alpha}\textstyle\sum_{k=s}^{l} \stepsize_{k+1}^2 \cdot (\ks^{3\alpha}\textstyle\sum_{k=s}^{l} \stepsize_{k+1}^2 +\textstyle\sum_{k=s}^{l} \EE[D_{k}^21_{A_k}] )},
		\end{align*}
		where the first inequality follows from Holder's inequality.
		Simple calculation yields
		$$
		\textstyle\sum_{k=s}^{l} \EE[D_{k}^21_{A_k}] \lesssim \ks^{2\alpha} \textstyle\sum_{k=s}^{l} \stepsize_{k+1}^2,
		$$
		as desired.
	\end{proof}

	\begin{lemma}\label{lem: second_fourth_moment}
		Suppose that Assumption~\ref{assumption:localbound},~\ref{assumption:Aproposed}, and~\ref{assumption:zero} hold. Then for any sufficiently small $\delta>0$, there exists a constant $C>0$ such that for any $\ks \ge 1$, and any $k \ge \ks$, 
		\begin{align*}
		\EE[D_k^21_{\tau_{\ks,\delta} >k}] \le C k_s^\alpha \eta_k,\quad \EE[D_k^41_{\tau_{\ks,\delta} >k}] \le C k_s^{3\alpha} \eta_k^3.
		\end{align*}
	\end{lemma}
	\begin{proof}
		We require that $\delta$ is small enough so that $\delta \le 1$ and $B_{\delta}(x^\star)$ is contained in the neighborhood where Assumption~\ref{assumption:Aproposed} holds with probability 1.	Define $A_k := \{\tau_{k_s, \delta} > k\}$ for all $k \ge k_s$. Following the calculation in~\eqref{eqn: stayinball1}, we obtain~\eqref{eq:decrease}. Consequently,
		\begin{align*}
		\EE_{k}[D_{k+1}^2 1_{A_{k+1}}] &\le D_{k}^21_{A_k}  - \mu\stepsize_{k+1} D_{k} 1_{A_k} + C\stepsize_{k+1}^2\\
		&\le  (1-\mu \stepsize_{k+1})D_{k}^21_{A_k}  + 	C\stepsize_{k+1}^2,
		\end{align*}
		where the second inequality follows from $\delta \le 1$. Taking expectations, we have 
		\begin{align*}
		\EE[D_{k+1}^2 1_{A_{k+1}}] \le  (1-\mu \stepsize_{k+1}) \EE[D_{k}^21_{A_k}]  + 	C\stepsize_{k+1}^2.
		\end{align*}
		By Lemma~\ref{lem:seq_bound_lem}, there exists a constant $C$ such that for any $k \ge k_s$,
		\begin{align*}
		\EE[D_k^2 1_{A_k}] \le Ck_s^{\alpha} \eta_k.
		\end{align*} 
		On the other hand, by the same argument of the proof of Lemma~\ref{lm:sumDksqrbounmd}, we have~\eqref{eqn: fourthpoweronestep}, which reads
		\begin{align*}
		\EE_k[D_{k+1}^41_{A_{k+1}}] \le (1 - \mu\stepsize_{k+1} )D_{k}^41_{A_k} - \mu\stepsize_{k+1} D_{k}^31_{A_k}  + C\stepsize_{k+1}^2 D_{k}^21_{A_k} +  C\stepsize_{k+1}^3 D_{k} 1_{A_k} + C\stepsize_{k+1}^4.
		\end{align*}
		Note that there exists a constant $\tilde C$ depending only on $\mu$ and $C$ such that when $D_k \ge \tilde C \stepsize_{k+1}$, we have 
		$$- \mu\stepsize_{k+1} D_{k}^31_{A_k}  + C\stepsize_{k+1}^2 D_{k}^21_{A_k} +  C\stepsize_{k+1}^3 D_{k}1_{A_k} \le 0.$$ 
		When $D_k \le \tilde C\stepsize_{k+1}$, we have 
		$$
		- \mu\stepsize_{k+1} D_{k}^31_{A_k}  + C\stepsize_{k+1}^2 D_{k}^21_{A_k} +  C\stepsize_{k+1}^3 D_{k}1_{A_k} \le (C\tilde C^2 + C \tilde C + C) \eta_{k+1}^4.
		$$
		Therefore, by enlarging $C$ if necessary, we always have 
		\begin{align*}
		\EE_k[D_{k+1}^41_{A_{k+1}}] \le (1 - \mu\stepsize_{k+1} )D_{k}^41_{A_k} + C\stepsize_{k+1}^4.
		\end{align*}
		Taking expectations, we have 
		\begin{align*}
		\EE[D_{k+1}^41_{A_{k+1}}] \le (1 - \mu\stepsize_{k+1} )\EE[D_{k}^41_{A_k}] + C\stepsize_{k+1}^4.
		\end{align*}
		By Lemma~\ref{lem:seq_bound_lem}, there exists a constant $C$ such that for any $k \ge k_s$, 
		\begin{align*}
		\EE[D_k^4 1_{A_k}] \le C k_s^{3\alpha} \eta_k^3.
		\end{align*}
	\end{proof}

	We have the following lemma for the size of $E_k$.
	\begin{lemma}\label{lem:sizeofE_k}
		Suppose that  Assumption~\ref{ass:basic_assumpt} --~\ref{assum:bounded_seq} hold.  Let $\delta >0$ be small enough so that Assumption~\ref{ass:basic_assumpt} --~\ref{assumption:Aproposed} hold inside $B_\delta(x^\star)$.  For any $k \ge 0$, we have
		\begin{align*}
		\EE[\|E_k\|_2^2] \lesssim \left(\frac{\delta}{\eta_{k+1}}\right)^2 + \delta^2 + \cub
		\end{align*}
		Additionally, for any $\ks \ge 0$,  and $j \ge i \ge \ks$, we have
		\begin{align*}
		\textstyle\sum_{k=i}^{j}\EE[\|E_k\|_2^2\indic{k}] \lesssim  \textstyle\sum_{k = i}^{j} k_s^{2\alpha} \eta_k^2. \numberthis \label{eqn:sum_ek_square}
		\end{align*}
	\end{lemma}
	\begin{proof}
		By definition, we always have
		\begin{equation}\label{eqn:Ek_naive_bound}
		\begin{aligned}
		\EE[\|E_k\|_2^2] &= \EE\left[ \left\|\frac{y_{k+1}- y_k}{\eta_{k+1}} + F_{\cM}(y_k) + P_{T_\cM(y_k)}(\nu_k)\right\|_2^2\right]\\
		&\lesssim \left[ \left(\frac{\delta}{\eta_{k+1}}\right)^2 + \delta^2 + \cub\right],
		\end{aligned}
		\end{equation}
		where the last inequality follows from $y_k \in B_{4\delta}(x^\star)$, the smoothness of $F_\cM$, and Assumption~\ref{assumption:martinagle} and~\ref{assum:bounded_seq}. On the other hand, by~\cite[Proposition 6.3, item 2(a)]{davis2025active}, we have 
		\begin{align*}
		\|E_k\|_2 1_{\tau_{\ks,\delta} >k} \lesssim (1+\|\nu_k\|^2)(D_k + \eta_k). \numberthis\label{eq:ek2bound}
		\end{align*}
		The estimate~\ref{eqn:sum_ek_square} then follows from Assumption~\ref{assumption:zero} and Lemma~\ref{lm:sumDksqrbounmd}.
	\end{proof}
	
	\begin{lemma}\label{lm:yk4momentconv}
		Suppose that Assumption~\ref{ass:basic_assumpt} --~\ref{assum:bounded_seq} hold. Let $\delta >0$ be small enough so that Assumption~\ref{ass:basic_assumpt} --~\ref{assumption:Aproposed} hold inside $B_\delta(x^\star)$.  For any $\ks \ge 0$ and $k \ge \ks$, we have
		\begin{align*}
		\expec{\|y_{k}-x^\star\|_2^p\indic{k}}{}\lesssim  k_s^{p\alpha/2}\eta_{k}^{p/2}, \quad p = 1, 2, 4.
		\end{align*}
	\end{lemma}
	\begin{proof}\label{pf:yk4momentconv}
		Note that
		\begin{align*}
		&\quad \|y_{k+1}-x^\star\|_2^4\indic{k}\\
		&=\|y_k - \eta_{k+1} F_\cM(y_k)-\eta_{k+1} P_{T_\cM (y_k)}(\nu_k) + \eta_{k+1} E_k-x^\star\|_2^4\indic{k}\\
		&\le \|y_{k}-x^\star\|_2^4\indic{k}-4\eta_{k+1} \dotp{y_k-x^\star, F_\cM(y_k)+ P_{T_\cM (y_k)}(\nu_k)-E_k} \|y_k-x^\star\|_2^2\indic{k} \\
		&\quad + 6\eta_{k+1}^2\|y_{k}-x^\star\|_2^2 \|F_\cM(y_k)+ P_{T_\cM (y_k)}(\nu_k)-E_k \|_2^2\indic{k}\\
		&\quad + 4\eta_{k+1}^3\|y_k-x^\star\|_2 \|F_\cM(y_k)+ P_{T_\cM (y_k)}(\nu_k)-E_k\|_2^3 \indic{k}+\eta_{k+1}^4\|F_\cM(y_k)+ P_{T_\cM (y_k)}(\nu_k)-E_k\|_2^4\indic{k}\\
		&\leq \underbrace{\|y_{k}-x^\star\|_2^4\indic{k}-4\eta_{k+1} \dotp{y_k-x^\star, F_\cM(y_k)+ P_{T_\cM (y_k)}(\nu_k)-E_k} \|y_k-x^\star\|_2^2\indic{k}}_{(I)}\\
		&\quad + \underbrace{8\eta_{k+1}^2\|y_{k}-x^\star\|_2^2 \|F_\cM(y_k)+ P_{T_\cM (y_k)}(\nu_k)-E_k \|_2^2\indic{k}}_{(II)} +\underbrace{3\eta_{k+1}^4\|F_\cM(y_k)+ P_{T_\cM (y_k)}(\nu_k)-E_k\|_2^4\indic{k}}_{(III)},
		\end{align*}
		where the first inequality follows from direct expansion and Cauchy-Schwarz inequality, and the second inequality follows from Young's inequality.
		Taking expectations, we have
		\begin{align*}
		\EE[\|y_{k+1}-x^\star\|_2^4 \indic{k}] \le \EE[(I)] + \EE[(II)] + \EE[(III)].	
		\end{align*}
		We bound the terms separately. By Assumption~\ref{assumption:smoothcompatibility} and the inequality of arithmetic and geometric means, we have
		\begin{align*}
		\EE_k[(I)] &= \EE_k[\|y_{k}-x^\star\|_2^4\indic{k}  - 4\eta_{k+1} \dotp{y_k - x^\star, F_\cM(y_k)\indic{k} -\EE_k[E_k]} \|y_k - x^\star\|_2^2\indic{k}]\\
		&\le (1- 4\strong \eta_{k+1}) \|y_{k}-x^\star\|_2^4\indic{k}  + 4\eta_{k+1}  \EE_k[\|E_k\|_2 \|y_k - x^\star\|_2^3\indic{k}]\\
		&\le (1- 4\strong \eta_{k+1}) \|y_{k}-x^\star\|_2^4\indic{k} + \strong \eta_{k+1} \|y_k - x^\star\|_2^4\indic{k} + \frac{81}{\strong^3} \eta_{k+1}\EE_k[\|E_k\|_2^4\indic{k}]\\
		&\le (1-3\strong \eta_{k+1})\|y_{k}-x^\star\|_2^4 \indic{k} +\frac{81}{\strong^3}\eta_{k+1}\EE_k[\|E_k\|_2^4\indic{k}] ,	
		\end{align*}
		Taking expectation, using \eqref{eq:ek2bound}, and applying Lemma~\ref{lem: second_fourth_moment}, there exists constant $C$ such that 
		\begin{align*}
		\EE[(I)] &\le (1-3\strong \eta_{k+1})\EE[\|y_{k}-x^\star\|_2^4\indic{k}] +\frac{81}{\strong^3}\eta_{k+1}\EE[\|E_k\|_2^4 \indic{k}]\\
		&\le  (1-3\strong \eta_{k+1})\EE[\|y_{k}-x^\star\|_2^4\indic{k}] + C \ks^{3\alpha} \eta_{k+1}^4.
		\end{align*}
		Similarly, 
		\begin{align*}
		(II) &\le  24 \eta_{k+1}^2 \|y_k - x^\star\|_2^2 (\|F_\cM(y_k)\|_2^2 + \| P_{T_\cM (y_k)}(\nu_k)\|_2^2 + \|E_k\|_2^2)\indic{k}\\
		&\le  (24\lm^2 + 12)\eta_{k+1}^2 \|y_k - x^\star\|_2^4\indic{k} + 12\eta_{k+1}^2 \|E_k\|_2^4\indic{k}\\
		&\quad  + \gamma \eta_{k+1} \|y_k - x^\star\|_2^4\indic{k} + \frac{144}{\gamma} \eta_{k+1}^3 \| P_{T_\cM (y_k)}(\nu_k)\|_2^4\indic{k},
		\end{align*}
		where the first inequality follows from Cauchy-Schwarz inequality, and the second inequality follows from Lipschitz continuity of $F_\cM$ and Young's inequality. Taking expectation, and using \eqref{eq:ek2bound}, and Lemma~\ref{lem: second_fourth_moment} and Assumption~\ref{assumption:martinagle}, there exists a constant $C$ such that
		\begin{align*}
		\EE[(II)] &\le  (24\lm^2 + 12)\eta_{k+1}^2 \EE[\|y_k - x^\star\|_2^4\indic{k}] + C\ks^3 \eta_{k+1}^5  + \gamma \eta_{k+1} \EE[\|y_k - x^\star\|_2^4\indic{k}] + C \eta_{k+1}^3.
		\end{align*}
		Moreover, using Jensen's inequality, \eqref{eq:ek2bound}, and applying Lemma~\ref{lem: second_fourth_moment}, there exists a constant $C$ such that
		\begin{align*}
		\EE[(III)] &\le   81 \eta_{k+1}^4 \EE[\|F_\cM(y_k)\|_2^4 \indic{k}] + 81 \eta_{k+1}^4 \EE[\| P_{T_\cM (y_k)}(\nu_k)\|_2^4\indic{k}]  + 81 \eta_{k+1}^4 \EE[\|E_k\|_2^4 \indic{k}]\\
		&\le C \eta_{k+1}^4 + Ck_s^{3\alpha} \eta_{k+1}^7.
		\end{align*}
		Combining and using the fact that $k_s \eta_{k+1} \lesssim 1$, we have 
		\begin{align*}
		\EE[\|y_{k+1} - x^\star\|_2^4\indic{{k+1}}] \le (1-2\gamma \eta_{k+1} + (24\lm^2 + 12) \eta_{k+1}^2)\EE[\|y_{k}-x^\star\|_2^4\indic{k}] + C \ks^{2\alpha} \eta_{k+1}^3.   
		\end{align*}
		As a result, for any $k \ge \max\left\{k_s, \left(\frac{\eta(24\lm^2 + 12)}{\gamma}\right)^{1/\alpha}\right\}$, we have  
		\begin{align*}
		\EE[\|y_{k+1} - x^\star\|_2^4\indic{k}]&\le (1-\gamma \eta_{k+1})\EE[\|y_{k}-x^\star\|_2^4\indic{k}] + C \ks^{2\alpha} \eta_{k+1}^3.   
		\end{align*}
		By Lemma~\ref{lem:seq_bound_lem} and the fact that $\left(\frac{\eta(24\lm^2 + 12)}{\gamma}\right)^{1/\alpha}$ is a constant, there exists constant $C$ such that 
		\begin{align*}
		\EE[\|y_{k} - x^\star\|_2^4\indic{k}] \le C k_s^{2\alpha} \eta_{k}^2, \quad \forall k \ge \ks.
		\end{align*}
		This resolves the case when $p=4$. The other two cases follow from Holder's inequality.
	\end{proof}
	
	\section{Auxiliary lemmas}
	
	\begin{lemma}\label{lem: upperboundpij}
		Define 
		$$
		p_i^j = \begin{cases}
		\prod_{k=i}^{j} \left(1- \frac{\gamma\eta_i}{2}\right) & i \le j\\
		1 & i = j + 1.
		\end{cases} 
		$$ Then for any $j \ge i$, 
		$$
		p_{i}^j \le \exp \left(-\frac{\gamma\eta ((j+1)^{1-\alpha} - i^{1-\alpha})}{2(1-\alpha)}\right)
		$$
	\end{lemma}
	\begin{proof}
		Note that
		\begin{align*}
		\log(p_i^j) &= \textstyle\sum_{k=i}^{j} \log\left(1-\frac{\gamma\eta_i}{2}\right)\\
		&\le -\frac{\gamma\eta}{2} \textstyle\sum_{k=i}^{j} k^{-\alpha}\\
		&\le  -\frac{\gamma \eta}{2} \int_{i}^{j+1} x^{-\alpha} dx\\
		&= -\frac{\gamma\eta ((j+1)^{1-\alpha} - i^{1-\alpha})}{2(1-\alpha)}.
		\end{align*}
	\end{proof}
	\begin{lemma}\label{lem: infiniteseries}
		For any $\alpha \in \left(\frac{1}{2}, 1\right)$ and $1\le i \le j$, we have 
		$$
		\textstyle\sum_{k=i}^{j} k^{-2\alpha} \le 	\textstyle\sum_{k=i}^{\infty} k^{-2\alpha} \le  1+ \frac{1}{2\alpha - 1} i^{1-2\alpha}
		$$
	\end{lemma}
	\begin{proof}
		Note that 
		\begin{align*}
		\textstyle\sum_{k=i}^{\infty} k^{-2\alpha} &\le 1 + \textstyle\sum_{k=i+1}^{\infty} k^{-2\alpha}\\
		&\le 1+ \int_{i}^{\infty} x^{-2\alpha} dx\\
		&= 1+ \frac{1}{2\alpha - 1} i^{1-2\alpha} .
		\end{align*}
	\end{proof}
	\begin{lemma}\label{lem: supportthirdorder}
		If $\alpha \in (\frac{1}{2},1)$, then for all $k \ge \left(\frac{4\alpha}{\mu\eta}\right)^{1/(1-\alpha)}$, we have 
		$$
		\frac{1 - \mu\stepsize_{k+1}}{\stepsize_{k+1}^2} \le \frac{1}{\stepsize_{k}^2}.
		$$
		If $\alpha =1$ and $\eta \ge \frac{4}{\mu}$, then the same inequality holds for all $k\ge 1$.
	\end{lemma}
	\begin{proof}
		Note that $\eta_k = \eta k^{-\alpha}$ by Assumption~\ref{assumption:zero}, it suffices to show that 
		\begin{align*}
		\frac{1- \mu \eta (k+1)^{-\alpha}}{(k+1)^{-2\alpha}} \le \frac{1}{k^{-2\alpha}}.
		\end{align*}
		Equivalently, we show that
		\begin{align}\label{eqn:ineqthirdorder}
		(k+1)^{2\alpha} - \mu \eta(k+1)^{\alpha} \le k^{2\alpha}
		\end{align} 
		Note that 
		\begin{align*}
		\left(1+\frac{1}{k}\right)^{2\alpha} &\le 1 + \frac{4\alpha}{k}\\
		&\le 1 + \frac{\mu \eta (k+1)^\alpha}{k^{2\alpha}},
		\end{align*}
		where the first inequality follows from the fact that $(1+x)^{2\alpha} \le 1+ 4\alpha x$ for all $\alpha \in (\frac{1}{2}, 1]$ and $x \in (0,1]$, and the second inequality follows from our assumption on $k$ and $\eta$, for the cases $\alpha \in (\frac{1}{2}, 1)$ and $\alpha =1$, respectively. Rearranging it, we obtain~\eqref{eqn:ineqthirdorder}.
	\end{proof}	
	\begin{lemma}\label{lem: tailbound}
		Let $\alpha \in \left(\frac{1}{2}, 1\right)$ and $C >0$. Then for any $k \ge 0$, we have 
		$$
		\textstyle\sum_{i=k}^{\infty} \exp(-C (i+1)^{\alpha}) \le \frac{2\exp(-C\sqrt{k})}{C^2} + \frac{2\sqrt{k}\exp(-C\sqrt{k})}{C}
		$$ 
	\end{lemma}
	\begin{proof}
		Note that 
		\begin{align*}
		\textstyle\sum_{i=k}^{\infty} \exp(-C(i+1)^{\alpha}) &\le \textstyle\sum_{i=k}^{\infty} \exp(-C(i+1)^{1/2})\\
		&\le \int_{k}^{\infty} \exp(-Cx^{1/2})dx\\
		&\le \int_{\sqrt{k}}^{\infty} 2u\exp(-Cu) du\\
		& =  \frac{2\exp(-C\sqrt{k})}{C^2} + \frac{2\sqrt{k}\exp(-C\sqrt{k})}{C},
		\end{align*}
		where the equality follows from the standard calculus calculation using integration by parts. 
	\end{proof}
	\begin{lemma}\label{lem:seq_bound_lem}
		Let $\alpha \in (0,\frac{1}{2})$,  $\theta > \alpha$, $c_1 >0$, $c_2>0$, and $c_3 >0$ be constants. Let $\{s_k\}$ be a sequence such that $0\le s_k \le c_3$ for all $k\ge 0$. Suppose that there exists $k_0 \ge 0$ such that
		\begin{align}\label{eqn:recursive_bound}
		s_{k+1} \le (1 - c_1(k+1)^{-\alpha}) s_k + c_2 (k+1)^{-\theta},\quad  \forall k\ge k_0.
		\end{align}
		Let $C = \max\left\{c_3, c_3 \left(\frac{2(\theta-\alpha)}{c_1}\right)^{\frac{\theta-\alpha}{1-\alpha}}, \frac{2c_2}{c_1(k_0+1)^{\theta-\alpha}}\right\}.$ We have
		\begin{align*}
		s_k \le C (k_0+1)^{\theta - \alpha} (k+1)^{-(\theta - \alpha)}, \forall k \ge k_0.
		\end{align*}
	\end{lemma}
	\begin{proof}
		We first show that the desired bound holds for all the $k_0 \le k \le \max \left\{\left(\frac{2(\theta-\alpha)}{c_1 }\right)^{\frac{1}{1-\alpha}}, k_0\right\}$. Note that $s_k \le c_3$, it suffices to show 
		\begin{align*}
		C (k_0+1)^{\theta - \alpha} \left(\max \left\{\left(\frac{2(\theta-\alpha)}{c_1}\right)^{\frac{1}{1-\alpha}}, k_0\right\}+1\right)^{-(\theta - \alpha)} \ge c_3,
		\end{align*}
		which holds by our assumption on $C$. Next, we apply induction to prove the bound for all $k \ge \max \left\{\left(\frac{2(\theta-\alpha)}{c_1}\right)^{\frac{1}{1-\alpha}}, k_0\right\}.$ Suppose that the bound holds for some $k \ge \max \left\{\left(\frac{2(\theta-\alpha)}{c_1}\right)^{\frac{1}{1-\alpha}}, k_0\right\}$. By~\eqref{eqn:recursive_bound}, we have
		\begin{align*}
		s_{k+1} &\le   C (k_0+1)^{\theta-\alpha} (k+1)^{-(\theta-\alpha)} - c_1C (k_0+1)^{\theta-\alpha} (k+1)^{-\theta} + c_2 (k+1)^{-\theta}\\
		&\le C (k_0+1)^{\theta-\alpha} (k+1)^{-(\theta-\alpha)} - \frac{c_1C}{2} (k_0+1)^{\theta-\alpha} (k+1)^{-\theta}\\
		&= C(k_0+1)^{\theta-\alpha} (k+1)^{-(\theta-\alpha)}\left( 1- \frac{c_1}{2} (k+1)^{-\alpha}\right),
		\end{align*}
		where the second inequality follows from the lower bound on $C$. In addition, simple calculus shows that for any $x\in [0,1/2]$, $(1-x)^{\theta-\alpha} \ge 1-2(\theta -\alpha) x$. Therefore, $\left(1-\frac{1}{k+2}\right)^{\theta-\alpha} \ge 1- \frac{2(\theta -\alpha)}{k+2}$. By the lower bound on $k$, we have \begin{align*}
		1- \frac{2(\theta -\alpha)}{k+2} \ge 1 - \frac{c_1}{ 2}(k+1)^{-\alpha}.
		\end{align*}
		Combining, we have 
		\begin{align*}
		s_{k+1} \le C(k_0+1)^{\theta-\alpha} (k+2)^{-(\theta-\alpha)}.
		\end{align*}
		The result follows.
	\end{proof}
	\begin{lemma}[{\cite[Lemma A.2]{zhu2023online}}]\label{lm:Sijbound} For any $j >i$, we have
		\begin{align*}
		\norm{S_i^j}_2 \lesssim i^\alpha.
		\end{align*}
	\end{lemma}
	\begin{lemma}\label{lem:bounded_seq_lemma}
		Let $\{x_k\}_{k = 0}^{\infty}$ be a nonnegative sequence satisfying
		$$
		x_{k+1} \le (1 + C_1 (k+1)^{-2\alpha}) x_k  + C_2 (k+1)^{-2\alpha},
		$$
		where $C_1$ and $C_2$ are positive constants, and $\alpha \in (1/2, 1)$. Then, there exists a constant $C$ depending on $C_1, C_2, \alpha$ and $x_0$ such that $x_k \le C$ holds for any $k\ge 0$.
	\end{lemma}
	\begin{proof}
		We begin by unrolling the recurrence. For any $k\ge0$, iterating the inequality gives
		\begin{align*}
		x_{k+1} &\le \textstyle\prod_{j=0}^k \left(1 + C_1 (j+1)^{-2\alpha}\right) x_0
		+ \textstyle\sum_{i=0}^k \left(\textstyle\prod_{j=i+1}^k \left(1 + C_1 (j+1)^{-2\alpha}\right)\right) C_2 (i+1)^{-2\alpha}.
		\end{align*}
		
		To bound the products, we use the inequality $\log(1+u) \le u$ for all $u > -1$. Hence,
		\[
		\textstyle\prod_{j=i+1}^k \left(1 + C_1 (j+1)^{-2\alpha}\right)
		\le \exp\Biggl(C_1 \textstyle\sum_{j=i+1}^k (j+1)^{-2\alpha}\Biggr)
		\le \exp\Biggl(C_1 \textstyle\sum_{j=1}^{\infty} j^{-2\alpha}\Biggr) =: M.
		\]
		Since $\alpha>\frac12$, we have $2\alpha>1$, and the series $\textstyle\sum_{j=1}^{\infty} j^{-2\alpha}$ converges; hence, $M < \infty$.
		
		Using this bound, we deduce that
		\[
		x_{k+1} \le M x_0 + M C_2 \textstyle\sum_{i=0}^k (i+1)^{-2\alpha}.
		\]
		Setting	$C = M (x_0 + C_2 \textstyle\sum_{i=0}^k (i+1)^{-2\alpha})$ concludes the proof.
	\end{proof}
	
\end{document}